\documentclass{article}
\usepackage{iclr2027_conference,times}

\usepackage{hyperref}
\usepackage{url}
\usepackage{booktabs}
\usepackage{amsthm}
\usepackage{array}
\usepackage{algorithm,algorithmic}
\usepackage{graphicx}
\usepackage{subcaption}

\theoremstyle{plain}
\newtheorem{theorem}{Theorem}[section]

\newtheorem{lemma}[theorem]{Lemma}
\newtheorem{corollary}[theorem]{Corollary}
\theoremstyle{definition}
\newtheorem{definition}[theorem]{Definition}

\theoremstyle{remark}
\newtheorem{remark}[theorem]{Remark}

\iclrfinalcopy

\title{Learning the Robustness Mechanism with Bilevel Optimization}
\author{Yiyang Shen \\
Department of Informatics \\
University of Iowa\\
\texttt{yiyang-shen@uiowa.edu}
\And
Qihang Lin \\
Tippie College of Business \\
University of Iowa\\
\texttt{qihang-lin@uiowa.edu}
\And
Weiran Wang \\
Department of Computer Science \\
University of Iowa\\
\texttt{weiran-wang@uiowa.edu}
}

\usepackage{amsmath,amsfonts,bm}

\def\1{\bm{1}}

\DeclareMathAlphabet{\mathsfit}{\encodingdefault}{\sfdefault}{m}{sl}
\SetMathAlphabet{\mathsfit}{bold}{\encodingdefault}{\sfdefault}{bx}{n}

\newcommand{\Ls}{\mathcal{L}}

\newcommand{\std}[1]{{\tiny $\pm$#1}}

\newcommand{\weiran}[1]{\textcolor{red}{#1 (Weiran)}}

\begin{document}
\maketitle
\begin{abstract}
We propose a distributionally robust learning framework where parameters defining the robustness mechanism are learned from held-out data instead of extensively tuned. Using bilevel optimization with both upper and lower level minimax problems, we create two instances of our framework to tackle setups with and without group labels in the training set. Theoretically, we provide sample complexity analysis for our robustness mechanism learning paradigm, showing that it achieves generalization guarantees comparable to exhaustive grid search while being more computationally efficient. Empirically, we evaluate our framework under a challenging setup when both intra-group and inter-group test distribution shifts occur at the same time, thereby demonstrating the efficacy and scalability of our method.

\end{abstract}

\section{Introduction}
Many machine learning methods require optimizing model parameters to minimize the empirical risk or average sample loss. The empirical risk minimization (ERM) paradigm assumes that unseen data are sampled from the same distribution as seen training data are. 
Since ERM weighs all samples equally, it is particularly vulnerable to \textit{subpopulation shift}, where the training samples consist of several groups divided by spurious attributes whose proportions are different from those of the test data \citep{sagawa2020distributionally,shen2021oodsurvey,cai2021theory,yang2023change,yu2024oodsurvey}.
In real-world applications such as healthcare \citep{zech2018variable,badgeley2019deep}, fairness \citep{buolamwini2018gender,mehta2024fairnessimage,lei2024fairdro}, robotics \citep{ryu2024riskawarecontrol}, and autonomous driving \citep{zhang2017curriculum,azizi2025autonomousrobust}, the classifier parameters may inadvertently depend on spurious attributes, causing failures when the testing environment is different. 

Distributionally robust optimization (DRO, \citet{duchi2021statistics}) along with many out-of-distribution generalization algorithms \citep{arjovsky2019invariant,sohoni2020no,krueger2021out} are developed to address this issue. A simplified setup for DRO is Group DRO (GDRO, \citet{sagawa2020distributionally}), which assumes grouping among samples and robustifies the model by minimizing the training loss of the group with the highest loss or worst training accuracy. Take the widely used CelebrityAttributes (CelebA) benchmark as an example, an ERM-trained model tend to correlate golden hair color (class labels) with female (attribute) celebrities, leading to severe performance degradation on minority groups, male celebrities with blond hair. Thus, GDRO explicitly minimizes the loss of the group incurring high loss. This has inspired a growing line of research which aims at developing out-of-distribution (OOD) generalization algorithms for the GDRO setup\citep{ahmed2020systematic,creager2021environment,piratlafocus,izmailov2022feature,nam2022spread,seo2022unsupervised,asgari2022masktune,zhang2022cnc,ghosal2023pgdro,paranjapeagro,wu2023discover,deng2023pde,han2024gic,jainimproving,labonte2024towards,pezeshki2024xrm,jeong2025medro,jo2026hdro}. Most methods fall under two general categories of setup. One stream focuses on the more ideal setup where all attribute labels are known, so the algorithms know the ground truth group membership of samples during training. The other focuses on a weakly group-supervised or group-unsupervised setup which is more realistic since attribute labels that define groups are often unavailable. Typically, a worst-group identification model is trained (e.g., worst loss samples in ERM, attribute prediction) during the first stage with or without a small amount of group-labeled data; in the second stage, robust training is done using those pseudo-labeled groups, e.g., GDRO.

GDRO is suited to address shifts in group distributions, i.e., inter-group subpopulation shift, which leads to failures concentrated in minority groups. However, treating groups as fixed distributions overlooks another type of uncertainty: the conditional distribution \textit{within} a group, i.e., \textbf{intra-group subpopulation shift} \citep{ben2002robust,devroye2013probabilistic,duchi2021learning}. For example, a minority group at training time may contain a small variation of environment, while the test distribution gives underrepresented variants of the same group. DRO with hierarchical ambiguity set (HDRO, \citet{jo2026hdro}) addresses this issue by introducing an adversarial perturbation within each group, which controls how much intra-group variation the model protects against. However, the appropriate amount of robustness signals per group is generally unknown, may not be uniform across groups, and requires extensive tuning. Similarly, when \textbf{group membership is unknown at training time}, worst group membership predictor typically is crucial to robustness and requires tuning as well.

As such, while effective at tackling various adversarial conditions, advanced robustness mechanisms require users to manually specify the type of uncertainty, and \textit{the level of uncertainty} the trained model should be robust against under different evaluation distributions. This introduced many more potentially sensitive models and hyper-parameters, making exhaustive tuning expensive and naturally raises the following question:
\begin{quote}
    \it Can we treat the robustness mechanism itself as an active, integral component to be learned during active training?
\end{quote}

Our affirmative answer contributes to DRO research in three aspects:
\begin{enumerate}
    \item We propose a bilevel adaptive tuning framework that \textit{directly learns a given robustness mechanism} based on the validation set, with a minimax problem for tuning robustness parameters on the upper level and a minimax problem for robust training on the lower level. With two instantiations, we show such bilevel problem can be tractably solved by a first-order method proposed by \citet{shen2026bilevel}.
    \item While most existing theory is concerned with the optimization complexity of solving such empirical problem to certain optimality conditions, such as saddle point and $\epsilon$-KKT point \citep{lu2024firstorder}, we derive generalization guarantees for our bilevel framework, providing the sample complexity for learning robust model which is new to the best of our knowledge.
    The guarantee shows advantage of continuous optimization of the hyperparameters avoids discretization error compared to grid search.
    \item We validate and enhance the evaluation setting proposed by \citet{jo2026hdro} where both inter-group and intra-group minority group test distribution shifts are manifest, showing significant improvement on worst-group prediction using our methods both with and without group labels at training time.
\end{enumerate}
\section{A Bilevel Adaptive Framework for GDRO}

\subsection{Group Distributionally Robust Optimization}
\label{sec:gdro_framework}
Suppose there are $G$ groups in the data distribution. Let $(x,y,a)$ be a data point, where $x$ is the feature vector, $y$ is the target variable, and $a\in \mathcal{A}$ is an attribute label that defines the group the data point belongs to together with known $y$. We consider a predictive task where the goal is to predict $y$ based on $x$ through a model $f_{W,\theta}(x):=Wh_\theta(x)$, where $h_\theta(x)$ is a mapping parameterized by $\theta$ that produces a representation of $x$ while $W$ is a matrix that defines a linear model that produces the prediction $Wh_\theta(x)$ for $y$. 

Let $D_{\mathrm{tr}}^{(g)}=\{(x_{g,i}^{\mathrm{tr}},y_{g,i}^{\mathrm{tr}})\}_{i=1}^{n_g^{\mathrm{tr}}}$ be a set of $n_g^{\mathrm{tr}}$ training samples from group $g$ for $g\in\{1,\dots,G\}$, and $\ell(f_{W,\theta}(x),y)$ be a loss function that measures the discrepancy between the prediction $f_{W,\theta}(x)$ and the target $y$. The average training loss on group $g$ is
$
L^{\mathrm{tr}}_g(W;\theta)
:=
\frac{1}{n_g^{\mathrm{tr}}}
\sum_{i=1}^{n_g^{\mathrm{tr}}}
\ell\left(f_{W,\theta}(x_{g,i}^{\mathrm{tr}}),y_{g,i}^{\mathrm{tr}}\right).
$
The GDRO model for learning $W$ and $\theta$ can be formulated as 
\begin{equation}
\label{eq:gdro}
(W^\star,\theta^\star)
\in
\arg\min_{W,\theta}\max_{q\in\Delta_G}
\left\{
\sum_{g=1}^{G}q_g L^{\mathrm{tr}}_g(W;\theta)
-
\frac{\eta}{2}
\left\|q-\frac{1}{G}\mathbf{1}\right\|_2^2
+
\frac{\lambda}{2}\|W\|_F^2
+
\frac{\lambda}{2}\|\theta\|_2^2
\right\},
\end{equation}
where $\|\cdot\|_F$ denotes the Frobenius norm, $\Delta_G:=\{q=(q_1,\dots,q_G)|q_g\geq0,1\leq g\leq G,\mathbf{1}^\top q=1\}$, $\eta\geq0$ is the \emph{robustness parameter} \citep{huang2021vrsmda,zhang2022sapd+}, and $\lambda\geq0$ is the regularization parameter. Here, the goal is to achieve a robust performance of $f_{W,\theta}$ by minimizing a weighted loss over groups with more weight put on the groups of larger average losses. Note that $\eta$ controls how far the group weight $q$ may deviate from uniform weighting, i.e, $\mathbf{1}/G$.
Naturally, it is critical to select $\eta$ in \eqref{eq:gdro} to achieve the best out-of-sample performance. Typically, a validation set is used, denoted by $D_{\mathrm{val}}^{(g)}=\{(x_{g,i}^{\mathrm{val}},y_{g,i}^{\mathrm{val}})\}_{i=1}^{n_g^{\mathrm{val}}}$ for group $g\in\{1,\dots,G\}$, to evaluate the robustness of the performance of $f_{W,\theta}$, for example, in its largest validation loss among the groups, i.e., \footnote{Other robustness metrics can be used here as well, such as a truncated simplex that replaces $\Delta_{G}$ in \eqref{eq:gdro_val}. }
\begin{equation}
\label{eq:gdro_val}
\max_{p\in\Delta_{G}}
\sum_{g=1}^{G}
p_g L^{\mathrm{val}}_g(W^\star;\theta^\star),
\quad \mathrm{where} \quad L^{\mathrm{val}}_g(W;\theta)
:=
\frac{1}{n_g^{\mathrm{val}}}
\sum_{i=1}^{n_g^{\mathrm{val}}}
\ell\left(f_{W,\theta}(x_{g,i}^{\mathrm{val}}),y_{g,i}^{\mathrm{val}}\right).
\end{equation}
Then a value of $\eta$ is selected from a grid to minimize \eqref{eq:gdro_val}. While widely used, this approach requires training a model for each candidate of $\eta$ and does not directly extend to the setting where the attribute label $a$ is missing from most of the training data.

\subsection{Warm-up: Bilevel Group DRO (Bi-GDRO)}

An adaptive bilevel group DRO method can be used to address the challenges caused by tuning. Instead of training the full model as in \eqref{eq:gdro} and selecting $\eta$ based on \eqref{eq:gdro_val}, we integrate the training and the parameter tuning into a bilevel optimization model as follows
\begin{align}
\label{eq:dro_upper}
\min_{W,\theta,\eta\geq0}
&\max_{p\in\Delta_{G}}
\sum_{g=1}^{G}
p_g L^{\mathrm{val}}_g(W;\theta)
+
\frac{\lambda}{2}\|\theta\|_2^2\\
\text{s.t.}&
\label{eq:dro_lower}
~~W
\in
\arg\min_{W'}\max_{q\in\Delta_G}
\left\{
\sum_{g=1}^{G}q_g L^{\mathrm{tr}}_g(W';\theta)
-
\frac{\eta}{2}
\left\|q-\frac{1}{G}\mathbf{1}\right\|_2^2
+
\frac{\lambda}{2}\|W'\|_F^2
\right\}.
\end{align}
Different from \eqref{eq:gdro} and \eqref{eq:gdro_val}, $\eta$ becomes a continuous upper-level decision variable in \eqref{eq:dro_upper} without being limited in a finite grid. Parameter $\theta$ is another upper-level decision variable while the linear model $W$ is optimized in the lower-level problem \eqref{eq:dro_lower}. This way, $W$ is learned from the training data for any given $\theta$ and $\eta$, while $\theta$ and $\eta$ are learned by optimizing the performance of $f_{W,\theta}$ on the validation set. Note that, with this design, the lower-level problem \eqref{eq:dro_lower} becomes convex in $W'$ and concave in $q$, which is required by most algorithms for bilevel optimization\footnote{As noted by \citet{kang2019decoupling,kirichenko2023dfr}, tuning of $W$ alone is sufficient for robustness. Furthermore, while closed form of $q$ is available, it is costly to compute as we show in Appendix \ref{apdx:lower_adversary}.}, although the upper-level objective function in \eqref{eq:dro_upper} can be nonconvex jointly in $W$ and $\theta$.

This bilevel minimax hyperparameter optimization model has been studied by \citet{shen2026bilevel}. However, as we show below, the framework like \eqref{eq:dro_upper} can be extended to learn additional elements of a robustness mechanism that is more general than \eqref{eq:gdro}, such as the radius of an ambiguity set and the latent group structure itself when group attribute labels are unavailable during training.

\subsection{Bilevel-Hierarchical DRO (Bi-HDRO)}
\label{sec:bi-hdro}
GDRO can be extended into hierarchical DRO (HDRO) whose ambiguity set has a hierarchical structure~\citep{jo2026hdro}. Let $P_g$ be the empirical distribution on $D_{\mathrm{tr}}^{(g)}$. The HDRO in our notation can be formulated as
\begin{equation}
\label{eq:hdro}
\min_{W,\theta}
\max_{Q\in\mathcal{Q}}\,
\mathbb{E}_{(X,Y)\sim Q}[
\ell\left(f_{W,\theta}(X),Y\right)],
\end{equation}
where $(X,Y)$ denotes a random data point,  $\mathbb{E}_{(X,Y)\sim Q}$ denotes the expectation taken over $(X,Y)$ when $(X,Y)$ follows distribution $Q$, and 
\[
\mathcal Q=\left\{\sum_{g=1}^G q_gQ_g:\;q\in\Delta_G,\;W_\infty(Q_g,P_g)\le \epsilon_g,~\text{ for }g=1,\dots,G\right\}
\]
is the hierachical ambiguity set, where
$Q_g$ is a distribution of $(X,Y)$, $W_\infty(Q_g,P_g)$ is the $\infty$-Wasserstein distance between $Q_g$ and $P_g$, and $\epsilon_g$ is a radius. As in \eqref{eq:gdro}, weights $q$ model the changes in group proportions, i.e., inter-group shift, while $Q_g$ is used to further to accommodate shifts within each group, i.e., intra-group shift. Note that \eqref{eq:hdro} is reduced to \eqref{eq:gdro} when $\epsilon_g=0$ since $Q_g=P_g$. 

Direct optimization over the distributions of $Q_g$ is generally intractable. Therefore, \citet{jo2026hdro} (Theorem 4.1) proposed solving the following upper approximation of  \eqref{eq:hdro}
\begin{align}
\label{eq:hdro_approx}
(W^\star,\theta^\star)
&\in
\arg\min_{W,\theta}
\max_{q\in\Delta_{G}}
\sum_{g=1}^{G} q_gL^{\mathrm{tr},\epsilon_g}_g(W;\theta)\\
&\mathrm{where}\quad
L^{\mathrm{tr},\epsilon_g}_g(W;\theta):=\frac{1}{n_g^{\mathrm{tr}}}\sum_{i=1}^{n_g^{\mathrm{tr}}}\,
\left[
\max_{\substack{z : \|z-h_\theta(x_{g,i}^{\mathrm{tr}})\|\le \epsilon_g}}
\ell\!\left(Wz,y_{g,i}^{\mathrm{tr}}\right)
\right].
\end{align}
Note that $L^{\mathrm{tr}}_g(W;\theta)\leq L^{\mathrm{tr},\epsilon_g}_g(W;\theta)$. In \eqref{eq:hdro_approx}, we minimize the largest loss over groups when the latent representation of each sample can be adversarially perturbed within a ball of radius $\epsilon_g$. The perturbation makes the model more robust to test-distribution intra-group shift. 


\paragraph{Problems with HDRO} \textbf{(1)} Solving the inner maximization over $z$ for each data point to evaluate  $L^{\mathrm{tr},\epsilon_g}_g(W;\theta)$ is computationally challenging when $n_g^{\mathrm{tr}}$ is large. \citet{jo2026hdro} proposed a heuristic method that performs one step of gradient ascent over $z$ from $h_\theta(x_{g,i}^{\mathrm{tr}})$, which only solves the inner maximization suboptimally and thus approximates $L^{\mathrm{tr},\epsilon_g}_g(W;\theta)$ and its gradient poorly. \textbf{(2)} $\epsilon_g$ requires additional tuning for each $g$. Although \citet{jo2026hdro} (Appendix D.3) proposed tuning a scalar $\epsilon$ with $\epsilon_g=\epsilon/\sqrt{n_g^{\mathrm{tr}}}$, this remains heuristic and still requires grid search over $\epsilon$ based on a performance metric on the validation set such as \eqref{eq:hdro_approx}.

\paragraph{Our Solution} \textbf{(1)} For each $g$, we propose a modification of $L^{\mathrm{tr},\epsilon_g}_g(W;\theta)$, denoted by $\tilde{L}^{\mathrm{tr},\epsilon_g}_g(W;\theta;u)$ with an additional variable $u$. The specific form of $\tilde{L}^{\mathrm{tr},\epsilon_g}_g$ depends on the prediction task and the loss function $\ell(\cdot,\cdot)$. We show that, for a binary classification problem where $\ell$ is either the hinge loss or the logistic loss, $\tilde{L}^{\mathrm{tr},\epsilon_g}_g(W;\theta;u)$ is jointly convex in $W$ and $u$, and \eqref{eq:hdro_approx} equals
\begin{align}
\label{eq:hdro_approx_tilde}
\min_{W,\theta,u}
\max_{q\in\Delta_{G}}
\sum_{g=1}^{G} q_g\tilde{L}^{\mathrm{tr},\epsilon_g}_g(W;\theta;u)~\text{ s.t. }~r(W,u)\leq 0,
\end{align}
where $r(W,u)$ is a jointly convex function of $W$ and $u$. For a multiclass classification problem where $\ell$ is the cross-entropy loss, we show that the corresponding $\tilde{L}^{\mathrm{tr},\epsilon_g}_g(W;\theta;u_g)$ and $r(W,u)$ are still jointly convex in $W$ and $u$ but \eqref{eq:hdro_approx_tilde} is only an upper bound of \eqref{eq:hdro_approx}. Therefore, we propose solving \eqref{eq:hdro_approx_tilde} as the gradient of $\tilde{L}^{\mathrm{tr},\epsilon_g}_g$ can be evaluated exactly without solving the inner maximization problems. Furthermore, in all the aforementioned cases, we can show that projection to the constraint set defined by the inequality $r(W,u)\leq 0$ has a closed form, meaning that \eqref{eq:hdro_approx_tilde} is not computationally more difficult than \eqref{eq:hdro_approx}. We present the details in Appendix \ref{apdx:perturb}.
\textbf{(2)} Similar to \eqref{eq:dro_upper}, we can tune $\epsilon_g$ in a bilevel optimization model based on the performance on the validation set after adding $\{\epsilon_g\}_{g=1}^G$ as upper-level decision variables just like $\eta$: 
\begin{align}
\label{eq:hdro_upper}
\min_{W,\theta,u,\eta\geq0,\{\epsilon_g\}_{g=1}^G}
&\max_{p\in\Delta_{G}}
\sum_{g=1}^{G}
p_g L^{\mathrm{val}}_g(W;\theta)
+
\frac{\lambda}{2}\|\theta\|_2^2\\\nonumber
\text{s.t.}&
\label{eq:hdro_lower}
~~W
\in
\arg\min_{W',u'}\max_{q\in\Delta_G}
\left\{
\sum_{g=1}^{G}q_g \tilde{L}^{\mathrm{tr},\epsilon_g}_g(W';\theta;u')
-
\frac{\eta}{2}
\left\|q-\frac{1}{G}\mathbf{1}\right\|_2^2
+
\frac{\lambda}{2}\|W'\|_F^2
\right\},\\\nonumber
&\qquad\qquad\quad\text{ s.t. }r(W',u')\leq0,
\end{align}
where we've replaced the loss $L^{\mathrm{tr},\epsilon_g}$ in \eqref{eq:dro_lower} to be $\tilde{L}^{\mathrm{tr},\epsilon_g}$. 

\subsection{Bilevel-Probabilistic Group DRO (Bi-PG-DRO)}
\label{sec:bi-pgdro}
In real-world scenarios, it is possible that only a very small portion of data has attribute label $a$ so we are not able to formulate $L^{\mathrm{tr}}_g(W;\theta)$ using all data points due to the lack of group information. To address this issue, \citet{ghosal2023pgdro} proposed PG-DRO, a robustness mechanism that uses a small amount of attribute-labeled training data to train a soft group predictor to generate pseudo-membership labels before using a robust model such as GDRO for training \citep{sagawa2020distributionally}.

Formally, let $D^{\mathrm{ul}}=\{(x_i^{\mathrm{ul}},y_i^{\mathrm{ul}})\}_{i=1}^{n^{\mathrm{ul}}}$ be a separate subset without group labels. We assume $n_g^{\mathrm{tr}}\ll n^{\mathrm{ul}}$ for any $g=(a,y)$, where attribute label $a$ and class label $y$ jointly determines the group. 
PG-DRO introduces another classification model $\tilde{f}_\phi(x)$ parameterized by $\phi$ and train $\tilde{f}_\phi(x)$ on $D_{\mathrm{tr}}^{(g)}$ to predict the attribute label $a\in \mathcal{A}$ based on $x$. For each data sample $(x_i^{\mathrm{ul}},y_i^{\mathrm{ul}})$ from $D^{\mathrm{ul}}$, we assume $\tilde{f}_\phi(x_i^{\mathrm{ul}})=(\gamma_{i1}(\phi),\gamma_{i2}(\phi),\dots,\gamma_{iG}(\phi))^\top$ where $\gamma_{ig}(\phi)$ is the predicted probability of $(x_i^{\mathrm{ul}},y_i^{\mathrm{ul}})$ being in group $g$ for each $g$, since class label is known. Using this conditional probability as soft group labels, we can assign a fraction of $(x_i^{\mathrm{ul}},y_i^{\mathrm{ul}})$ to each group, yielding the probabilistic loss
\[
L_g^{\rm PG}(W;\theta,\phi)
=
\frac{
\sum_{i=1}^{n^{\mathrm{ul}}}
\gamma_{ig}(\phi)
\ell\left(f_{W,\theta}(x_i^{\mathrm{ul}}),y_i^{\mathrm{ul}}\right)
}{
\sum_{i=1}^{n^{\mathrm{ul}}}
\gamma_{ig}(\phi)+\epsilon
},
\]
where $\epsilon$ is a smoothing parameter to avoid a zero denominator. 
Then PG-DRO solves
\begin{align}
\label{eq:pgdro}
(W^\star,\theta^\star)
\in
\arg\min_{W,\theta}
\max_{q\in\Delta_{G}}
\sum_{g=1}^{G} q_gL_g^{\rm PG}(W;\theta,\phi).
\end{align}
However, PG-DRO requires additional training for $\tilde{f}_\phi(x)$. For a more efficient training approach, we propose to integrate the training of $f_{W,\theta}(x)$ and $\tilde{f}_\phi(x)$ as well as the tuning of the robustness parameter into a bilevel optimization model below
\begin{align}
\label{eq:pgdro_upper}
\min_{W,\theta,\phi,\eta\geq0}
&\max_{p\in\Delta_{G}}
\sum_{g=1}^{G}
p_g L^{\mathrm{val}}_g(W;\theta)
+
\frac{\lambda}{2}\|\theta\|_2^2+
\beta\,
{\rm KL}(\pi_{\rm tr}||\pi_\phi)\\
\text{s.t.}&
\label{eq:pgdro_lower}
~~W
\in
\arg\min_{W'}\max_{q\in\Delta_G}
\left\{
\sum_{g=1}^{G}q_g L_g^{\rm PG}(W';\theta,\phi)
-
\frac{\eta}{2}
\left\|q-\frac{1}{G}\mathbf{1}\right\|_2^2
+
\frac{\lambda}{2}\|W'\|_F^2
\right\}.
\end{align}
Here, $\pi_\phi=(\sum_{i=1}^{n^{\mathrm{ul}}}\gamma_{ig}(\phi)/n^{\mathrm{ul}})_{g=1}^G$ is the proportion of data points predicted to be in group $g$ by model $\tilde{f}_{\phi}$, $\pi_{\rm tr}=(n_{g}^{\mathrm{tr}}/(\sum_{g'=1}^{G}n_{g'}^{\mathrm{tr}}))_{g=1}^G$ is an estimation of the prior distribution of the group labels, and  ${\rm KL}(\pi_{\rm tr}||\pi_\phi)$ is the Kullback-Leibler (KL) divergence between $\pi_\phi$ and $\pi_{\rm tr}$. Different from PG-DRO, $\tilde{f}_{\phi}$ in \eqref{eq:pgdro_upper} is not trained separately on a binary cross-entropy (BCE) loss to predict $a$. Instead, $\phi$ optimized in the upper-level in \eqref{eq:pgdro_upper} such that the produced $\gamma_{ig}$ helps ensure a good performance of the resulting $f_{W,\theta}(x)$ on the validation set, and a high prediction accuracy of $\tilde{f}_{\phi}$ is not necessary. One may replace ${\rm KL}(\pi_{\rm tr}||\pi_\phi)$  in  \eqref{eq:pgdro_upper} to the training loss of $\tilde{f}_{\phi}$ in predicting the group label $g$. Empirical findings (Appendix \ref{apdx:infernce_precision}) show that this has little impact on the numerical performance but using ${\rm KL}(\pi_{\rm tr}||\pi_\phi)$ as the regularizer makes the optimization more lightweight.

\subsection{Bilevel Minimax Algorithm}
Despite their different formulations, Problems \eqref{eq:dro_upper}, \eqref{eq:hdro_upper}, and \eqref{eq:pgdro_upper} can be solved using the first‑order method proposed by \citet{shen2026bilevel}. We provide the algorithm's pseudocode in Appendix \ref{apdx:algorithm} and summarize it here.
Firstly, these problems are instances of the following bilevel minimax problem
\begin{equation}
\label{eq:bilevel}
\min_{\alpha,W,q}
\left\{
    \max_{p}
    F(\alpha,p,W,q)
    \;\middle|\;
    (W,q)
    \in
    \arg\min_{\widetilde W }\max_{ \widetilde q}
    \widetilde F(\alpha,\widetilde w,\widetilde q)
\right\},
\end{equation}
where $\alpha$ denotes the collection of all primal upper-level decision variables, including $\theta$, $\eta$, $\epsilon_g$ and $\phi$ in the three bilevel models above, $p$ is the upper-level group weight, $q$ is the lower-level group weight, and  
$W$ is the parameter of the linear classifier within model $f_{W,\theta}$. Here, $F$ and $\widetilde F$ are different objectives.
The primal and dual value functions of the lower-level minimax problem of \eqref{eq:bilevel} are denoted by 
\[
V_{\mathrm P}(\alpha,W)
:=
\max_{\widetilde q}
\widetilde F(\alpha,W, \widetilde q)
\quad \mathrm{and} \quad
V_{\mathrm D}(\alpha,q)
:=
\min_{\widetilde W}
\widetilde F(\alpha,\widetilde W,q),
\]
respectively.
Therefore, Problem \eqref{eq:bilevel} can be equivalently written as the single-level constrained problem with a primal-dual gap constraint
\[
\min_{\alpha,W,q}
\left\{
    \max_{p} F(\alpha,p,W,q)
    \;\middle|\;
    V_{\mathrm P}(\alpha,W)
    -
    V_{\mathrm D}(\alpha,q)
    \leq 0
\right\}.
\]
Lastly, introducing a penalty parameter $\rho>0$ yields the penalized minimax problem
\begin{equation}
\label{eq:penalized_minimax}
\min_{\alpha,W,q}
\max_p
\left\{
F(\alpha,p,W,q)
+
\rho
\left[
V_{\mathrm P}(\alpha,W)
-
V_{\mathrm D}(\alpha,q)
\right]
\right\}
=
\min_{\alpha,W,q}
\max_{p,\widetilde W,\widetilde q}
P_\rho
\left(
\alpha,p,W,q,\widetilde W,\widetilde q
\right),
\end{equation}
where $P_\rho$ denotes the corresponding penalized objective.
To compute a $\epsilon$-primal-dual stationary point of $P_\rho$ which is nonconvex-concave, we apply the inexact proximal-point method as in  \citet{shen2026bilevel}. The original nonconvex-concave problem is thereby reduced to a sequence of approximately
solved strongly-convex-strongly-concave minimax subproblems, each of which
solved using the stochastic accelerated primal-dual (SAPD) algorithm by treating minimizing and maximizing variables as separate blocks
\citep{zhang2022sapd+}.

\subsection{Generalization Theory}

We establish a generalization theory for our continuous bilevel hyperparameter tuning framework. 
For clarity of exposition, we present the results in this section without the encoder $\theta$; the full extensions and analysis details are provided in Appendix~\ref{apdx:gen}.
Formally, we analyze the problem of finding the optimal multidimensional hyperparameters $\hat{\boldsymbol{\psi}} \in \Psi$ (e.g., $\boldsymbol{\psi} = (\lambda, \eta)$ for Group DRO, or $\boldsymbol{\psi} = (\lambda, \eta, \boldsymbol{\epsilon})$ for Bi-HDRO) that minimize the worst-group validation loss:
\begin{equation}
\label{eq:gen_upper}
\hat{\boldsymbol{\psi}} = \arg\min_{\boldsymbol{\psi} \in \Psi} \max_{p\in\Delta_{G}}
\sum_{g=1}^{G}
p_g L^{\mathrm{val}}_g(\widehat{W}_{\boldsymbol{\psi}}),
\end{equation}
where the robust model parameters $\widehat{W}_{\boldsymbol{\psi}}$ are trained via the lower-level robust objective:
\begin{equation}
\label{eq:gen_lower}
\widehat{W}_{\boldsymbol{\psi}} = \arg\min_{W} \max_{q \in \Delta_G} \left[ \sum_{g=1}^G q_g L^{\mathrm{tr}}_g(W) - \frac{\eta}{2}\left\|q - \frac{1}{G}\mathbf{1}\right\|^2 + \frac{\lambda}{2}\|W\|^2 \right].
\end{equation}
Under standard assumptions on the learning problem (e.g., Lipschitz and bounded convex loss, bounded continuous hyperparameter space, and strongly convex lower-level regularization), we establish that the lower-level optimization algorithm induces a bounded, Lipschitz-continuous hypothesis space with respect to the continuous hyperparameters. 

Informally, we prove a \emph{Continuous Oracle Inequality} for our robust tuning frameworks, demonstrating that tuning continuous multidimensional hyperparameters (such as the robustness penalty $\eta$, the $L_2$ regularization coefficient $\lambda$, and group-specific perturbation radii $\boldsymbol{\epsilon}$) on a validation set allows our algorithm to achieve an optimal bias-variance trade-off without suffering from discretization error or grid-search penalties.
Our main technical tools 
rely on the uniform stability of the lower-level predictor (due to strong convexity) and the Rademacher complexity of the algorithmic hypothesis class (due to algorithmic Lipschitzness) \citep{shalev2014understanding}. 

\begin{theorem}[Informal Continuous Oracle Inequality for Bilevel DRO]
\label{thm:informal_dro_oracle}
Let $\hat{\boldsymbol{\psi}}$ be the multidimensional continuous hyperparameters tuned via the upper-level continuous validation process. 
Let $\widehat{W}_{\hat{\boldsymbol{\psi}}}$ be the corresponding robust model trained in the lower level. With high probability over the training and validation sets, the true worst-group risk $L_{\mathcal{D}}^{\text{worst}}(\widehat{W}_{\hat{\boldsymbol{\psi}}}) := \max_{g \in [G]} L_{\mathcal{D},g}(\widehat{W}_{\hat{\boldsymbol{\psi}}})$ is bounded by:
\begin{align*}
    L_{\mathcal{D}}^{\text{worst}}(\widehat{W}_{\hat{\boldsymbol{\psi}}}) 
    &\le \underbrace{L_{\mathcal{D}}^{\text{worst}}(W^*)}_{\text{Reference Risk}} 
    + \underbrace{\mathcal{O}\left( \sqrt{\frac{\log G}{\min_g n_g^{\mathrm{tr}}}} + \sqrt{\frac{\log G}{\min_g n_g^{\mathrm{val}}}} \right)}_{\text{Statistical Gaps}} \\
    &\hspace{-1em} + \min_{\boldsymbol{\psi}} \Bigg( \underbrace{\mathcal{O}(\lambda + \eta)}_{\text{Approximation Bias}(W^*; \boldsymbol{\psi})} + \underbrace{\mathcal{O}\left( \frac{1}{\lambda \eta \min_g n_g^{\mathrm{tr}}} \right)}_{\text{Stability Gap}(\boldsymbol{\psi})} \Bigg) 
    + \underbrace{\mathcal{O}\left( \sqrt{\frac{k}{\min_g n_g^{\mathrm{val}}} \log(\rho_{\mathcal{A}, \boldsymbol{\psi}})} \right)}_{\text{Continuous Tuning Penalty}}
\end{align*}
where $\min_g n_g^{\mathrm{tr}}$ and $\min_g n_g^{\mathrm{val}}$ are the sizes of the smallest groups in the training and validation sets respectively, $G$ is the number of groups, $k$ is the dimensionality of the hyperparameter space (e.g., $k=2$ for standard GDRO, $k=G+2$ for HDRO), and $\rho_{\mathcal{A}, \boldsymbol{\psi}}$ is the algorithmic Lipschitz constant of the lower-level optimization. The Approximation Bias scales with the algorithmic regularization penalties $\lambda$ and $\eta$, while the Stability Gap (derived via uniform stability) shrinks as regularization increases, fundamentally capturing the bias-variance trade-off parameterized by $\boldsymbol{\psi}$.
\end{theorem}

\begin{remark}[Algorithmic Lipschitz Constant]
    For standard Group DRO where $\boldsymbol{\psi} = (\lambda, \eta)$, the Lipschitz constant is bounded by $\rho_{\mathcal{A}} = \mathcal{O}\big( \frac{1}{\lambda_{\min}^2} + \frac{1}{\lambda_{\min}\eta_{\min}^2} \big)$ (see Corollary~\ref{cor:dro_joint_tuning}), where $\lambda_{\min}>0$ and $\eta_{\min}>0$ are lower bounds of search spaces. When extending to Bi-HDRO with tunable perturbation radii $\boldsymbol{\psi} = (\lambda, \eta, \boldsymbol{\epsilon})$, the mapping remains Lipschitz continuous with the constant expanding by $\mathcal{O}\big( \frac{1}{\lambda_{\min}} + \frac{1}{\lambda_{\min}\eta_{\min}} \big)$ due to the 
    norm-bounded inner perturbations (see Lemma~\ref{lem:hdro_lipschitz}). Furthermore, as detailed in the appendix, this framework extends to the scenario where the weights of a deep neural network encoder (upper level decision variables) are also treated as tuning parameters. Similar bounds on uniform stability and algorithmic Lipschitz continuity hold in this high-dimensional regime (see Corollaries~\ref{cor:dro_joint_tuning_encoder} and~\ref{cor:hdro_joint_tuning_encoder}).
\end{remark}

This result confirms that continuous bilevel tuning discovers the theoretically optimal configuration for any unknown reference predictor $W^*$. Crucially, the statistical penalty for tuning over a continuous space scales logarithmically with the algorithmic Lipschitz constant. This constant dictates an ``effective grid size''---the finite number of distinguishable models within the search space---allowing us to bypass discretization error while paying a statistical penalty no worse than a discrete grid search. Furthermore, unlike exhaustive grid search which suffers from exponential computational complexity in high dimensions, our scheme enables efficient continuous optimization over multidimensional hyperparameter spaces (full assumptions and proofs are provided in Appendix~\ref{sec:dro_theory}).

\section{Related Works}

\paragraph{GDRO} GDRO \citep{sagawa2020distributionally} aims at minimizing the training loss on the group with least training signals due to the spurious attribute. Empirically, minimizing the worst group training loss per se does not translate to robustness over test data, necessitating a tuned weight decay term \citep{sagawa2020distributionally}. 
DFR \citep{kirichenko2023dfr} and AFR \citep{qiu2023afr} improve on GDRO by retraining the convex classifier (last layer of a deep neural network) using a group-balanced set \citep{ren2018learning}, which they show to improve the worst-group robustness even with a ERM-trained model. When both inter-group and intra-group uncertainty exist, HDRO \citep{jo2026hdro} perturbs the latent representation with group-dependent radii before performing classification in the last layer. The perturbation radii is fixed and sensitive, so extensive tuning is necessary.

Other methods focus on using a small amount of group-labeled data to achieve similar worst-case oracle performance. 
SSA \citep{nam2022spread} first trains a hard group predictor before running GDRO on pseudolabeled data, and PG-DRO \citep{ghosal2023pgdro} instead use soft group prediction and robust training on soft labels.
Notably, AGRO \citep{paranjapeagro} jointly trains an adversarial soft group prediction model by changing group assignments to increase robust classifier’s group-wise prediction loss. CnC \citep{zhang2022cnc} aligns samples with the same class but different attributes in a two-stage contrastive learning framework. DISC \citep{wu2023discover} partitions data with a ``concept bank'' which consists of candidate spurious attributes. GIC \citep{han2024gic} has three stages and identifies spurious features by comparing the training set with a carefully selected reference dataset. D3M \citep{jainimproving} removes examples that disproportionately degrade worst-group accuracy. Remarkably, XRM \citep{pezeshki2024xrm} requires no auxiliary datasets whatsoever and instead identifies spurious attributes by training twin classifiers with mutually exclusive training splits and discover attributes using worst-performing samples before running GDRO.

\paragraph{Bilevel Optimization}
Bilevel optimization methods has been widely applied to hyperparameter tuning~\citep{bennett2008bilevel,franceschi2018bilevel}, meta-learning~\citep{franceschi2018bilevel,bertinetto2018meta,rajeswaran2019meta}, reinforcement learning~\citep{hong2023two,yang2024bilevel,li2024learning,li2024bialign}, and neural architecture search~\citep{liu2018darts}. 

Among hyperparameter tuning applications, validation set performance is optimized in the upper level to improve model generalizability over the training set \citep{domke2012generic, maclaurin2015gradient, franceschi2017forward, franceschi2018bilevel, shaban2018truncated,feurer2019hyperparameter,lorraine2020optimizing}.
In our work, we treat the robustness mechanism, which may contain non-convex neural networks, as hyperparameters to be optimized. 
Inspired by \citet{lu2024firstorder} and \citet{lu2025solving} which solved bilevel optimization problem via deterministic minimax optimization, \citet{shen2026bilevel} addressed a more general bilevel optimization problem when both upper and lower level are minimax problems and extended it to a stochastic case.
The fact that GDRO itself is a minimax problem naturally makes such bilevel-minimax algorithm good solver candidates.

\section{Experiments}
In this section, we describe the modified datasets under both inter-group and intra-group subpopulation shift before showing the performance on two instantiations of our bilevel robust mechanism learning framework for DRO, namely \textbf{Bi-HDRO} and \textbf{Bi-PG-DRO}.
\subsection{Datasets with Test Distribution Shift}
We use datasets where both intra-group distribution and inter-group distribution shifts exist. Since baseline methods have reached similar performance under inter-group subpopulation shift, and further tuning offers no improvements, we do not investigate it here. Dataset details in Appendix \ref{apdx:data-meta}.
\paragraph{Shifted CMNIST}
Class label is digit (0-4 or 5-9), and the spurious attribute is color. We rotate minority group samples (red, lebel 1 (digit 5-9)) by 90$^\circ$ in the validation and test sets. 
\paragraph{Shifted CelebA} Class label is hair color, and spurious attribute is gender. We include only no-glasses images in training and validation, and only with-glasses images at test time for minority group (male with blond hair). 
\paragraph{Shifted CivilComments}
Class label is toxicity and spurious attribute is black/white.
Since ``black'' and ``white'' attributes may be co-mentioned, for the minority toxic/black group, we include no ``white'' attribute in the training set, whereas in the test set all entries have the ``white'' attribute.

\begin{table}[htbp]
\caption{Accuracy over 3 runs under shifted distributions for Bi-HDRO and its baselines}
\label{tab:hdro_results}
\centering
\begin{tabular}{l cc cc cc}
\toprule
& \multicolumn{2}{c}{\textbf{CMNIST}}
& \multicolumn{2}{c}{\textbf{CelebA}}
& \multicolumn{2}{c}{\textbf{CivilComments}} \\
Method & Worst & Avg & Worst & Avg & Worst & Avg \\
\midrule
\textbf{GDRO}&71.3\std{1.3}&72.7\std{1.6}&59.2\std{0.1}&92.7\std{0.1}&34.8\std{6.0}&87.8\std{1.4}\\
\textbf{DFR$^{\rm Tr}$} &62.9\std{8.3}&68.9\std{4.3}&65.5\std{4.8}&89.4\std{0.3}&40.8\std{2.7}&87.9\std{0.9}\\
\textbf{PDE} &62.8\std{7.8}&69.1\std{3.8}&35.9\std{3.4}&92.0\std{0.6}&39.0\std{3.9}&81.8\std{0.8}\\
\textbf{HDRO} & 72.7\std{0.2} & 76.3\std{3.1}& 72.4\std{3.0} & 91.4\std{0.2}& 40.8\std{3.1} & 88.0\std{1.8} \\
\midrule
\textbf{Fixed $\epsilon$}\\

\quad -- Bi-GDRO ($\epsilon=0$)
& 73.9\std{0.9} & 74.8\std{0.8}
& 77.0\std{4.3} & 90.1\std{0.5}
& 57.9\std{5.1} & 79.1\std{1.3} \\

\quad -- Grid search $\epsilon$
& 72.3\std{2.1} & 73.9\std{1.8}
& 82.5\std{2.0} & 90.7\std{1.3}
& 54.7\std{2.3} & 80.0\std{1.2} \\

\quad -- Grid search $\eta$ and $\epsilon$
& 73.7\std{0.7} & 75.0\std{1.2}
& 82.1\std{2.5} & 90.4\std{0.8}
& 57.9\std{6.5} & 78.0\std{3.4} \\

\textbf{Bi-HDRO (learnable $\epsilon$)}
& 74.0\std{0.3} & 76.4\std{0.4}
& 81.8\std{3.5} & 89.2\std{1.0}
& \bf 59.5\std{3.5} & 78.5\std{2.6} \\

\textbf{Bi-HDRO (learnable $\epsilon_g$)}
& \bf 74.2\std{0.3} & 76.9\std{1.6}
& \bf 82.8\std{3.1} & 90.1\std{1.2}
& 57.4\std{2.0} & 78.5\std{0.6} \\
\bottomrule
\end{tabular}
\end{table}
\subsection{Results on Bi-HDRO}
In this setup, group membership is known at all times, so we use \textbf{GDRO} \citep{sagawa2020distributionally}, \textbf{DFR}$^{\rm Tr}$ \citep{kirichenko2023dfr} and \textbf{PDE} \citep{deng2023pde} as baselines since they need access to group information at training time. See details in Appendix \ref{apdx:bi-hdro-details}.

\paragraph{Validity of our enhanced setup} 
As seen in the top half of Table \ref{tab:hdro_results}, intra-group shift on top of existing inter-group shift indeed decimates the performance of strong baselines, making the doubly-shifted datasets valuable benchmarks. We further validated the improvement of HDRO compared to existing baselines and show that employing ambiguity sets in such setup provides tangible benefits across all datasets. For example, HDRO improves from 59.2 of GDRO to 72.4 on CelebA. However, we used 4 $\epsilon$ and 5 $C$ values (part of the scaling term $C/n_g^{\rm tr}$ in \citealp[Eq (5)]{sagawa2020distributionally}) for HDRO tuning, yielding 16 combinations in grid search.
\paragraph{Benefit of parameter tuning} 
In the bottom half of Table \ref{tab:hdro_results}, we show that (1) Bi-HDRO performs better than HDRO (lower level) alone; and (2) ablating individual $\epsilon_g$ and $\eta$ as manually tuned and fixed components shows the superior performance of Bi-HDRO. 


\subsection{Results on Bi-PG-DRO}
In this setting, group labels are not available at training time, so we uniformly sample a small fraction of group-labeled validation data (5\% from CMNIST, 15\% from CelebA, and 3\% from CivilComments) and create two splits, where one is used to tune the attribute prediction model and the other is used to tune the validation set. For comparison fairness among baseline methods, we made sure attribute prediction training sees the same split, and the other split is for manual model tuning. 
We use \textbf{ERM}, \textbf{GIC}$^{C_y}$ \citep{han2024gic}, \textbf{XRM} \citep{pezeshki2024xrm}, \textbf{AGRO} \citep{paranjapeagro}, and \textbf{SSA} \citep{nam2022spread} as our baselines since they need minimal or no group-labeled data at training time.
See details in Appendix \ref{apdx:bi-pgdro-details}.
\paragraph{Group-labeled validation data boost our performance} With validation tuning, fixed $\eta$, and no attribute prediction regularization, our method already performs better than all baseline methods.
\paragraph{Implicitly tuning auxiliary model yields superior performance} Compared to explicitly using a cross-entropy loss to train the attribute predictor, implicitly tuning the attribute predictor by minimizing KL divergence works just as well. Finally, we show that the complete bilevel method (Bi-PG-DRO) tuning $\eta$ works the best when combined with KL, outperforming all baseline methods.

\begin{table}[htbp]
    \caption{Accuracy over 3 runs under shifted subpopulations for Bi-PG-DRO and its baselines}
    \label{tab:ssa_results}
    \centering
    \begin{tabular}{l cc cc cc}
    \toprule
    &\multicolumn{2}{c}{\textbf{CMNIST}} &\multicolumn{2}{c}{\textbf{CelebA}} &\multicolumn{2}{c}{\textbf{CivilComments}} \\
    \bf Method&Worst&Avg&Worst&Avg&Worst&Avg\\
    \midrule
     \textbf{ERM} &1.6\std{1.7}&15.7\std{6.1}&25.0\std{2.6}&95.4\std{0.1}&36.8\std{4.6}&91.3\std{0.8}\\
     \textbf{GIC$^{C_y}$}&25.7\std{8.8}&48.0\std{13.4}&47.1\std{9.2}&90.8\std{0.8}&54.2\std{4.9}&86.8\std{1.4} \\
     \textbf{XRM}&68.8\std{3.8}&72.2\std{3.2}&51.4\std{2.8}&89.7\std{0.2}&25.9\std{6.4}&89.6\std{1.5}\\
     \textbf{AGRO}    &22.8\std{4.9}&32.4\std{2.2}&22.1\std{1.8}&95.4\std{0.2}&24.5\std{3.5}&91.4\std{0.4}   \\
     \textbf{SSA}    &70.0\std{2.9}&72.3\std{1.7}&58.6\std{2.3}&89.7\std{0.1}&27.4\std{8.6}&87.1\std{1.9}   \\
     \textbf{PG-DRO}    &64.3\std{2.0}&70.7\std{0.7}&66.7\std{1.0}&91.6\std{0.4}&16.1\std{3.1}&82.3\std{3.0}   \\
     \midrule
     \textbf{Fixed $\eta=1.0$}&67.8\std{1.5}&74.1\std{1.6}&69.8\std{6.9}&92.6\std{0.2}&66.2\std{1.1}&84.3\std{1.1} \\
     \quad + BCE &69.5\std{1.4}&75.3\std{4.0}&73.3\std{5.4}&91.8\std{2.2}&65.0\std{1.3}&79.3\std{1.4}\\
     \quad + KL &71.0\std{0.8}&80.1\std{0.9}&73.3\std{4.8}&92.5\std{0.3}&68.0\std{0.3}&84.7\std{1.1} \\
     \textbf{Bi-PG-DRO}    &\bf 71.7\std{0.8}&79.5\std{1.5}&\bf 76.7\std{3.4}&92.2\std{0.3}&\bf 68.2\std{0.2}&83.6\std{1.9} \\
     \bottomrule
    \end{tabular}
\end{table}
\section{Conclusion}
We proposed a DRO framework to learn the robustness mechanism that requires minimal tuning and provided two instantiantions. There are two future directions from this work. First, our bilevel framework is still somewhat restrictive in that the algorithm proposed by \citet{shen2026bilevel} requires lower level training to be convex-concave. By assuming Kurdyka-Łojasiewicz (KL) condition so that lower level becomes a non-convex-concave problem, a more general class of applications ensues, but no such algorithm exists yet. Second, this class of algorithms may be further extended to machine unlearning and language model alignment tasks \citep{fan2025unlearning,wu2025drdpo,asif2026ofmu}.

\newpage
\subsection*{AI use statement}
In this work, we used generative AI tools to implement methods, clean and reformat dataset, and assist in the writing of proofs.

We have not used generative AI tools to help develop theoretical models or conceptual frameworks, formulate mathematical claims, provide critical ingredients for proving mathematical claims, propose or refine hypotheses, design or provide feedback on research methodology or experiments, assist with translation, support qualitative and thematic data analysis, or interpret results.

Generating synthetic datasets is not applicable to this work.

We have reviewed all AI-assisted work. LLM-generated code was verified and tested for correctness. 
LLM-generated proof steps are judiciously reviewed and revised by all authors.
We take responsibility for the final content of this work,
including text, claims or artifacts produced with the aid of generative AI.
\newpage
\bibliographystyle{iclr2027_conference}
\bibliography{opt}




\newpage
\appendix
\section{A First-Order Method for Bilevel Minimax Problems}
\label{apdx:algorithm}

\begin{algorithm}[htbp]
\caption{A First-Order Method for \eqref{eq:bilevel}}
\label{alg:bilevel-dro}
\begin{algorithmic}[1]
\REQUIRE Initial iterates
$(\alpha^0,W^0,q^0)$ and
$(p^0,\widetilde W^0,\widetilde q^0)$ and number of proximal iterations $K$.
\STATE Set $\rho,\rho_1,\rho_2$ and relevant parameters according to
\citealp[Algorithm 1]{shen2026bilevel}.
\FOR{$k=0,\ldots,K-1$}
    \STATE Construct the proximal penalized objective
    {\small\[
    \bar{\mathcal P}_k
    =
    P_\rho
    (\alpha,p,W,q,\widetilde W,\widetilde q)
    +\frac{\rho_1}{2}
      \|(\alpha,W,q)-(\alpha^k,W^k,q^k)\|^2
    -\frac{\rho_2}{2}
      \|(p,\widetilde W,\widetilde q)
       -(p^k,\widetilde W^k,\widetilde q^k)\|^2 
    \]}
    \STATE Solve the resulting
    strongly-convex-strongly-concave problem with SAPD \citep{zhang2022sapd+}:
    {\small\[
    ((\alpha^{k+1},W^{k+1},q^{k+1}),
      (p^{k+1},\widetilde W^{k+1},\widetilde q^{k+1}))
      \leftarrow
      \mathrm{SAPD}(\bar{\mathcal P}_k)
    \]}
\ENDFOR
\RETURN $(\alpha^{k'},W^{k'},q^{k'})$ with $k'$ sampled uniformly from
$\{1,\ldots,K\}$.
\end{algorithmic}
\end{algorithm}
While the original work has many hyperparameters on the algorithmic level, we stress that those are tuned once and can then be applied to all datasets used in this paper and both of our proposed methods (See Appendix \ref{apdx:hyper} for details). We further note that validation accuracy is used for model selection rather than objective convergence, which is prohibitively expensive in deep learning experiments.

\section{Closed Form of Lower Level Adversary}
\label{apdx:lower_adversary}
We can use the definition of training loss in Section \ref{sec:gdro_framework},
\[L^{\rm tr}(W;\theta)
=
\left(
L_1^{\rm tr}(W;\theta),
\ldots,
L_G^{\rm tr}(W;\theta)
\right)^\top,\]

to obtain the adversarial subproblem,
\[
q^\star
\in
\arg\max_{q\in\Delta_G}
\left\{
q^\top L^{\rm tr}(W;\theta)
-
\frac{\eta}{2}
\left\|
q-\frac1G\mathbf 1
\right\|_2^2
\right\}.
\]
For \(\eta>0\), we can complete the square:
\[
q^\top L^{\rm tr}
-
\frac{\eta}{2}
\left\|q-\frac1G\mathbf1\right\|_2^2
=
-\frac{\eta}{2}
\left\|
q-
\left(
\frac1G\mathbf1+
\frac{L^{\rm tr}}{\eta}
\right)
\right\|_2^2
+C,
\]where \(C\) does not depend on \(q\). Therefore,
\[
q^\star(W,\theta,\eta)
=
\Pi_{\Delta_G}
\left(
\frac1G\mathbf1+
\frac{L^{\rm tr}(W;\theta)}{\eta}
\right)
,
\] where component-wise
$
q_g^\star
=
\left[
\frac1G+
\frac{L_g^{\rm tr}(W;\theta)}{\eta}
-\nu
\right]_+,
$ and \(\nu\) is chosen such that
$\sum_{g=1}^Gq_g^\star=1.$

While there is a closed form solution, computing it requires the losses of all groups and hence a full pass of the entire dataset, which is prohibitively expensive during stochastic training.

\section{Perturbation Formulations}
\label{apdx:perturb}
\subsection{Closed-form Perturbation for Binary Classification}
We show below that closed form can be derived for latent representation perturbation when the task is binary classification. We use the binary cross entropy (BCE) loss for all experiments but show the formulation for hinge loss as well.
\paragraph{Case 1: the last layer is linear and $\Ls$ is hinge loss.}

Let the final layer be binary linear classification:
$f_\theta^L(z)=w^\top z+b$.
The hinge loss is
\[
L(f_\theta^L(z),y)
=
\max(0,1-y(w^\top z+b)).
\]

The inner maximization in equation becomes\[
\sup_{\|z'-z\|\le \epsilon}
\max\!\left(0,1-y(w^\top z'+b)\right).
\]

Write $z'=z+\delta, \|\delta\|\le \epsilon,$ then
\[
\sup_{\|\delta\|\le\epsilon}
\max\!\left(
0,
1-y(w^\top z+b)-y w^\top\delta
\right).
\]
Since $x\mapsto \max(0,x)$ is monotonically increasing,\[
=
\max\!\left(
0,
1-y(w^\top z+b)
+
\sup_{\|\delta\|\le\epsilon}(-y w^\top\delta)
\right).
\]

Because of $y\in\{\pm1\}$ and the support function of the norm ball,\[
\sup_{\|\delta\|\le\epsilon}(-y w^\top\delta)
=
\sup_{\|\delta\|\le\epsilon} w^\top\delta
=
w^\top\delta^*
=
w^\top (\epsilon\frac{w}{\|w\|^2})
=
\epsilon \|w\|_*,
\]
where $\|\cdot\|_*$ is the dual norm.
Therefore, the inner robust hinge loss has closed form:
\[
\sup_{\|z'-z\|\le\epsilon}
L(f_\theta^L(z'),y)
=
\max\!\left(
0,
1-y(w^\top z+b)+\epsilon\|w\|_*
\right)
\]

\paragraph{Case 2: the last layer is linear and $\Ls$ is logistic/binary cross entropy (BCE) loss.}
Using the same notations as above, the logistic loss is
\[
L(f_\theta^L(z),y)
=
\log\!\left(1+\exp(-y(w^\top z+b))\right).
\]

Due to the monotonicity of the exponent of the maximized perturbation, we have \[
\sup_{\|\delta\|\le \epsilon}
\log\!\left(
1+\exp\left(
-y(w^\top(z+\delta)+b)
\right)
\right)
\Leftrightarrow
\inf_{\|\delta\|\le \epsilon}
y(w^\top(z+\delta)+b)
=
y(w^\top z+b)
+
\inf_{\|\delta\|\le\epsilon} y w^\top\delta.
\]

As derived in Case 1, $\inf_{\|\delta\|\le\epsilon} y w^\top\delta=-\epsilon\|w\|_*$, so we can substitute this back to the maximized perturbation, forming
\[
\sup_{\|\delta\|\le\epsilon}
\log\!\left(
1+\exp(-y(w^\top(z+\delta)+b))
\right)
=
\log\!\left(
1+\exp\big(
-(y(w^\top z+b)-\epsilon\|w\|_*)
\big)
\right).
\]

$\|w\|$ makes the both losses non-smooth, so we smoothen it using a substitution scaler $u$. In the logistic example, it becomes

\begin{equation}
\label{eq:smoothen}
J(w,u)
=
\sum_{i=1}^n
\log\left(
1+\exp(-y_i(w^\top x_i+b)+\epsilon u)
\right)
\quad\text{s.t.}\quad
\|w\|_* \le u,
\end{equation}
which is jointly convex in $(w,u)$.

As a result, we perform projected descent when optimizing Eq \eqref{eq:smoothen}. 
Define the second-order cone $\mathcal{K}=\{(w,u):\|w\|_2\leq u\}$, the Euclidean projection is then \begin{equation}
\label{eq:perturbation_closed_form}
\min_{(w,u)\in\mathcal{K}} \frac12 \|w-v\|_2^2+\frac12 (u-t)^2,
\end{equation}
the closed form of which falls into three categories, where $(v,t)$ falls inside, outside and ``behind'', and outside but near the boundary of the cone, i.e.,
\begin{align*}
\Pi_{\mathcal K}(v,t)
=
\begin{cases}
(v,t),
&
\|v\|_2 \le t
\\[1ex]
(0,0),
&
\|v\|_2 \le -t
\\[1ex]
\left(
\frac{\|v\|_2+t}{2\|v\|_2}v,
\frac{\|v\|_2+t}{2}
\right),
&
\text{otherwise}.
\end{cases}
\end{align*}

\subsection{Perturbation for Multi-class Classification}
Next, we show that latent representation perturbation can still be done, although in a relaxed form, when the task is multi-class.

Let the final layer be linear with $K$ classes:
\[
w_k(z)=\mathbf{w}_k^\top z+b_k,\qquad k=1,\ldots,K,
\]
and let the true label be $y\in\{1,\ldots,K\}$.

\paragraph{Case 1: multi-class hinge loss.}
We use the multi-class hinge loss \citep{crammer2001svm}:
\[
L(w(z),y)=\max\left\{0,\max_{j\neq y}\left[1+(w_j(z)+b_j)-(w_y(z)+b_y)\right]\right\},
\]
where $w_y(z)$ is the score of the correct class, and the inner max aims to find the largest violation over all incorrect classes. For a perturbation $z'=z+\delta$, $\|\delta\|\le \epsilon$, we have
\[
w_j(z+\delta)-w_y(z+\delta)
=
(w_j-w_y)^\top z+(b_j-b_y)+(w_j-w_y)^\top\delta.
\]
Therefore,
\[
\sup_{\|\delta\|\le \epsilon} L(w(z+\delta),y)
=
\sup_{\|\delta\|\le \epsilon}
\max\left\{0,\max_{j\neq y}
\left[1+(w_j-w_y)^\top z+b_j-b_y+(w_j-w_y)^\top\delta\right]\right\}.
\]
Since the maximum is over finitely many ($K$) affine functions, the supremum can be exchanged with the finite maximum:
\[
=
\max\left\{0,\max_{j\neq y}
\left[
1+(w_j-w_y)^\top z+b_j-b_y
+
\sup_{\|\delta\|\le \epsilon}(w_j-w_y)^\top\delta
\right]\right\}.
\]
Using the support function of the norm ball,
\[
\sup_{\|\delta\|\le \epsilon}(w_j-w_y)^\top\delta
=
\epsilon\|w_j-w_y\|_*.
\]
Thus the robust multi-class hinge loss has the closed form
\[
\sup_{\|z'-z\|\le \epsilon} L(w(z'),y)
=
\max\left\{0,\max_{j\neq y}
\left[
1+w_j(z)-w_y(z)+\epsilon\|\mathbf{w}_j-\mathbf{w}_y\|_*
\right]\right\}.
\]
Equivalently, one may introduce auxiliary variables $u_j$ satisfying
\[
\|w_j-w_y\|_*\le u_j,\qquad j\neq y,
\]
and write the robust loss as
\[
\max\left\{0,\max_{j\neq y}
\left[
1+w_j(z)-w_y(z)+\epsilon u_j
\right]\right\}.
\]
This form is convex in the final-layer parameters for fixed features $z$.

\paragraph{Case 2: multi-class logistic / softmax cross-entropy loss.}
For true class $y$, the softmax cross-entropy loss is
\[
L(w(z),y)
=
-\log\frac{\exp(w_y(z))}{\sum_{k=1}^K \exp(w_k(z))}
=
\log\left(\sum_{k=1}^K \exp(w_k(z)-w_y(z))\right).
\]
Define
\[
a_k=w_k(z)-w_y(z),\qquad 
v_k=\mathbf{w}_k-\mathbf{w}_y.
\]
Then
\[
w_k(z+\delta)-w_y(z+\delta)
=
a_k+v_k^\top\delta,
\]
with $a_y=0$ and $v_y=0$. The robust loss is
\[
\sup_{\|\delta\|\le \epsilon}
\log\left(\sum_{k=1}^K \exp(a_k+v_k^\top\delta)\right).
\]


\noindent A useful variational representation is obtained from the Fenchel form of log-sum-exp:
\[
\log\sum_{k=1}^K \exp(r_k)
=
\sup_{p\in\Delta_K}
\left\{
p^\top r+H(p)
\right\},
\]
where
\[
H(p)=-\sum_{k=1}^K p_k\log p_k.
\]
Applying this with $r_k=a_k+v_k^\top\delta$ gives
\begin{gather} \label{eqn:robust-multi-class-ce}
\sup_{\|\delta\|\le \epsilon}
\log\sum_{k=1}^K \exp(a_k+v_k^\top\delta)
=
\sup_{p\in\Delta_K}
\left\{
\sum_{k=1}^K p_k a_k+H(p)
+
\epsilon\left\|\sum_{k=1}^K p_k v_k\right\|_*
\right\}.
\end{gather}

However, this is not a closed-form perturbation of the same type as the binary case.
Continuing from~\eqref{eqn:robust-multi-class-ce}, we derive an efficient upper bound relaxation as follows
\begin{align*}
&\sup_{\|\delta\| \le \epsilon} \log \sum_{k=1}^K \exp(a_k + v_k^\top \delta) \\
&\le \sup_{p \in \Delta_K} \left\{ \sum_{k=1}^K p_k a_k + H(p) + \epsilon \sum_{k=1}^K p_k \|v_k\|_* \right\} \quad \text{(Jensen's inequality on } \|\cdot\|_*) \\
&= \sup_{p \in \Delta_K} \left\{ \sum_{k=1}^K p_k \left( a_k + \epsilon \|v_k\|_* \right) + H(p) \right\} \quad \text{(Grouping linear terms)} \\
&= \log \sum_{k=1}^K \exp\left( a_k + \epsilon \|v_k\|_* \right) \quad \text{(Reverse Fenchel conjugate)}
\end{align*}


\section{Experimental Details}
\subsection{Shifted Datasets}
\label{apdx:data-meta}
\paragraph{Shifted CMNIST} Dataset statistics are presented in Table \ref{tab:cmnist-meta}. In the original HDRO experiment \citep{jo2026hdro}, rotation is only applied to test set and its distribution significantly differs from the validation distribution, causing large variance across different random seeds (8\%-10\%) and tuning on such validation set would make little sense. In our setup, rotation is applied to both validation and test split of the minority group (red, label 1). This is done to reduce the extreme large variance observed during model tuning and stabilize prediction performance. Note that only rotations are applied, and samples are not moved.
\begin{table}[htbp]
\caption{CMNIST Statistics.}
\label{tab:cmnist-meta}
\centering
\begin{tabular}{lrrr}
\toprule
\textbf{Group} & \textbf{Train} & \textbf{Val} & \textbf{Test} \\
\midrule
$y=0$, green & 2,998  & 2,591 & 8,966 \\
$y=1$, green & 11,781 & 2,513 & 1,013 \\
$y=0$, red   & 12,130 & 2,465 & 1,068 \\
$y=1$, red   & 3,091  & 2,431 & 8,953 \\
\midrule
\textbf{Total} & 30,000 & 10,000 & 20,000 \\
\bottomrule
\end{tabular}
\end{table}

\paragraph{Shifted CelebA} Dataset statistics are presented in Table \ref{tab:celeba-groups}. Our shifting procedure is the same as that of \citet{jo2026hdro}. For minority group (blond male), 164 without glasses are moved from test to train, 90 with glasses are moved from train to test, and 10 with glasses are moved from validation to test.

\begin{table}[htbp]
\caption{CelebA Statistics}
\label{tab:celeba-groups}
\centering
\begin{tabular}{llrrr}
\toprule
\textbf{Group} & \textbf{Train} & \textbf{Val} & \textbf{Test} \\
\midrule
non-blond female & 71,629 & 8,535 & 9,767 \\
non-blond male   & 66,874 & 8,276 & 7,535 \\
blond female     & 22,880 & 2,874 & 2,480 \\
blond male (before shift)       & 1,387 & 182 & 180 \\
blond male (after shift)       & 1,461  & 172   & 116 \\
\midrule
\multicolumn{2}{l}{\textbf{Total (before shift)}}162,770 & 19,867 & 19,962 \\
\multicolumn{2}{l}{\textbf{Total (after shift)}}162,844 & 19,857 & 19,898 \\
\bottomrule
\end{tabular}
\end{table}

\paragraph{Shifted CivilComments} Dataset statistic presented in Table \ref{tab:civilcomments-groups}. Similar to the test distribution shift of CMNIST and CelebA, we created a shifted version of the CivilComments dataset \citep{koh2021wilds}, which is originally used for toxic language classification.  Intra-group shift on this dataset has not been studied previously to our knowledge. 
For the shifted version, 1,278 minority group (toxic, black) samples with white=1 are moved from train to test, and 905 minority group samples with white=0 are moved from test to train. After shifting, train minority group contains only non-white-annotated samples, and test minority contains only white annotated samples. The fact that attributes black=1 and white=1 are not mutually exclusive allows such shift to happen.

\begin{table}[htbp]
\caption{CivilComments Statistics}
\label{tab:civilcomments-groups}
\centering
\begin{tabular}{lrrr}
\toprule
\textbf{Group} & \textbf{Train} & \textbf{Val} & \textbf{Test} \\
\midrule
non-toxic, non-black & 231,738 & 39,006 & 115,223 \\
non-toxic, black & 6,785   & 1,119  & 3,335 \\
toxic, non-black & 27,404  & 4,522  & 13,687 \\
toxic, black (before shift) & 3,111   & 533    & 1,537 \\
toxic, black (after shift) & 2,738   & 533    & 1,910 \\
\midrule
\textbf{Total (before shift)} & 269,038 & 45,180 & 133,782 \\
\textbf{Total (after shift)} & 268,665 & 45,180 & 134,155 \\
\bottomrule
\end{tabular}
\end{table}

\paragraph{Shifted Waterbirds}
Waterbirds \citep{sagawa2020distributionally} is one of the most commonly used benchmark in this line of DRO research, but we do not use it here because (1) the dataset itself has known mislabeled attributes \citep{asgari2022masktune}, making performance metrics less informative; (2) the above issue is compounded with the small size of the training set (4,795 total, with 56 in the minority group of waterbird with land background); and (3) the performance has saturated with or without intra-group shifts on existing methods, rendering comparisons banal.

\subsection{Settings}
\label{apdx:hyper}
In Table \ref{tab:bilevel-hyperparameters}, we provide the important settings (hyperparameters and architectures) used in this work. We emphasize that those hyperparameters are relatively insensitive to the dataset, so we only tune them on CMNIST for one method and use the same values for the other two datasets. $\epsilon_g$ is cliped [0,1] to follow the HDRO preset range.
\begin{table}[htbp]
\caption{Settings and Hyperparameters for our methods.}
\label{tab:bilevel-hyperparameters}
\centering
\small
\setlength{\tabcolsep}{2pt}
\renewcommand{\arraystretch}{1.12}
\begin{tabular}{@{}p{0.39\textwidth}*{3}{>{\centering\arraybackslash}p{0.18\textwidth}}@{}}
\toprule
\bf Setting & \bf CMNIST & \bf CelebA & \bf CivilComments \\
\midrule
Model
    & ResNet-50 & ResNet-50 & DistilBERT \\
Weight decay $\lambda$
    & $10^{-3}$ & $10^{-1}$ & $10^{-3}$ \\
Lower Train/Upper Validation batch size
    & 256 & 128 & 64 \\
$L_{\nabla h}$ in \citet{shen2026bilevel}
    & \multicolumn{3}{c}{10000} \\
Iterations of SAPD per outer iteration $T$
    & \multicolumn{3}{c}{1} \\
$\eta_{\mathrm{init}}$
    & \multicolumn{3}{c}{1.0} \\
Gradient multiplier of $\eta$
    & \multicolumn{3}{c}{10}\\
Gradient multiplier of validation simplex $u$
    & \multicolumn{3}{c}{100}\\
$\epsilon_{\rm{init}}$ for Bi-HDRO
    & \multicolumn{3}{c}{$96/255$} \\
KL coefficient $\beta$ for Bi-PG-DRO
    & \multicolumn{3}{c}{10} \\
\bottomrule
\end{tabular}
\end{table}
\subsection{Bi-HDRO}
\label{apdx:bi-hdro-details}
\subsubsection{Ablation Procedure}
To show the advantage of automatically tuning both $\eta$ and $\epsilon$ (or $\epsilon_g$) in Bi-HDRO, (1) we fixed $\epsilon$ to be 0 and leave $\eta$ learnable, which recovers Bi-GDRO; (2) we performed grid search with fixed $\epsilon$ over \{60,72,84,96,108\}/255 when $\eta=0$, and (3) we performed grid search $\epsilon$ with the above schedule and $\eta$ over \{0.01, 0.1, 1.0, 10.0\}. 
\subsubsection{Baselines}
\paragraph{DFR$^{\rm Tr}$ \citep{kirichenko2023dfr}} A two-stage method where the training set is partitioned 80-20 for different purposes. (1) The larger chunk of data is used to train ERM with uniform sampling. (2) The smaller group-balanced subset to retrain last layer. For DFR$^{\rm Val}$, they include all minority group data in validation and sample the same amount from other groups. For fairness of group information access, we compare our method to DFR$^{\rm Tr}$ only. Default hyperparameters from their implementation are used.
\paragraph{PDE \citep{deng2023pde}} A two-stage method. (1) Warmup the entire model using group-balanced subset where all groups have same size as the minority group. (2) Progressively add more training data (progressive data expansion) for training the entire model as well as using existing warmup subset. We tune the warm-up epochs in \{10,20\} and added samples in \{50,100,500\}.

\paragraph{HDRO \citep{jo2026hdro}}
A single-level min-max-max method described in Section \ref{sec:bi-hdro}. We use the implementation from the authors and tune $\epsilon$ in the range \{60,72,84,96\}/255 and $C\in\{0,1,2,3\}$ across all three of our datasets.
\subsection{Bi-PG-DRO}
\label{apdx:bi-pgdro-details}
\subsubsection{Ablation Procedure}
We ablated two types of regularization: BCE loss for attribute predictor, KL prior regularization for attribute predictor. We performed a grid search \{1,10,30\} for these two terms.
\subsubsection{Baselines}
\paragraph{GIC$^{C_y}$ \citep{han2024gic}} A three-stage method. (1) Feature extraction. (2) Train attribute predictor maximizing spurious attribute label KL between predicted label and true label (Eq (11) in their paper, hence the superscript). (3) Train GDRO. We tune the $\gamma$ (Eq (9) in their paper) in \{2,5,10\} for our shifted datasets.
\paragraph{XRM \citep{pezeshki2024xrm}} A two-stage method. (1) Train two auxilary models with mutually exclusive held-in and held-out split of the \textit{training set}. Notably, it does not rely on any attribute labels from, for example, the validation set. Instead, it flips training set labels so that minority samples are identified, since they are confidently misclassified in the held-out set. (2) Train GDRO with pseudo-group-labels. We sample three hyperparameter combination candidates, choose one with the highest flip rates, and train GDRO using 3 different seeds.
\paragraph{AGRO \citep{paranjapeagro}} A greedy unified method that jointly trains attribute predictor and robust model. It learns soft group memberships that makes robust training (GDRO) difficult. We use the hyperparameters described in their Table 7 (and 4 slices for CMNIST), but tune their $\alpha$ in \{0.2,0.3,0.4\} due to our different setup (intra-group shift).
\paragraph{SSA \citep{nam2022spread}} A two-stage method. (1) Train a attribute label prediction model and performs \textit{hard prediction} to generate pseudo-group-label and (2) Train GDRO. Default hyperparameters are used. We tune the model with $C\in\{0,1,2,3\}$ across all three of our datasets.
\paragraph{PG-DRO \citep{ghosal2023pgdro}} A two-stage method described in Section \ref{sec:bi-pgdro}. Essentially SSA but with probabilistic group labels. We tune the model with $C\in\{0,1,2,3\}$ across all three of our datasets.
\section{Additional Results}
\subsection{Efficiency}
\paragraph{Validation performance.} In Figure \ref{fig:hdro-valcurve}, we show that Bi-HDRO reaches optimal validation accuracy quicker than HDRO does. For CelebA, HDRO took about 6,000 optimization steps to reach level similar to Bi-HDRO. Additionally, HDRO requires extensive grid search to obtain such results, while Bi-HDRO requires only one run.
\begin{figure}[t]
\centering
\begin{subfigure}{0.32\linewidth}
    \centering
    \includegraphics[width=\linewidth]{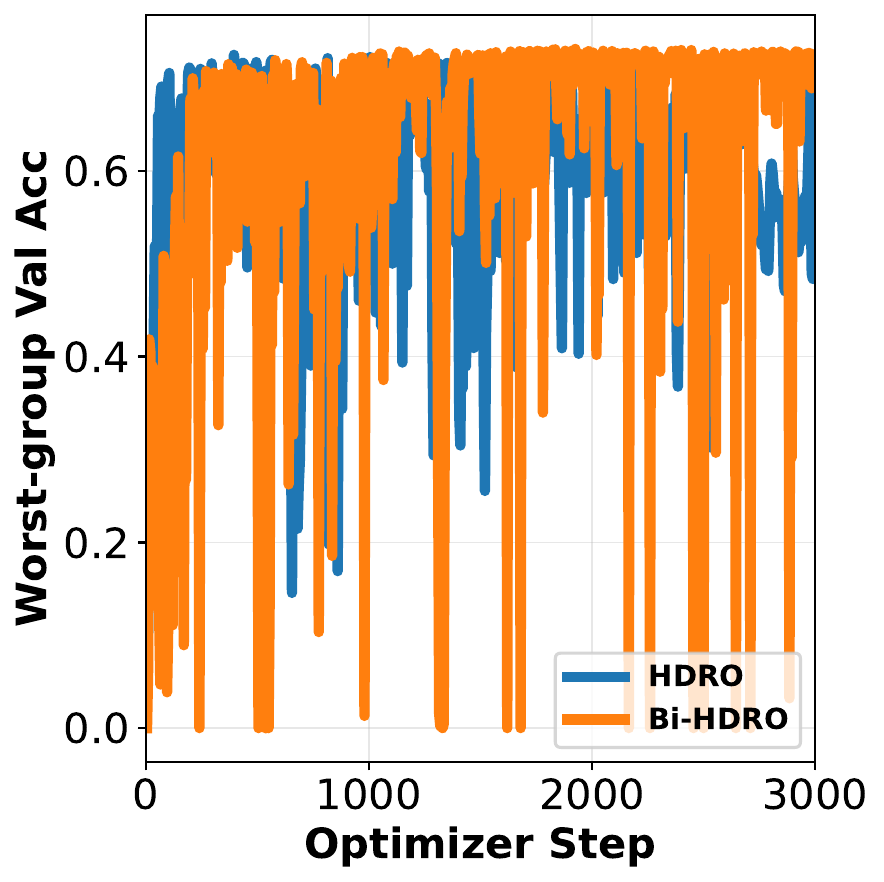}
    \caption{CMNIST}
\end{subfigure}
\begin{subfigure}{0.33\linewidth}
    \centering
    \includegraphics[width=\linewidth]{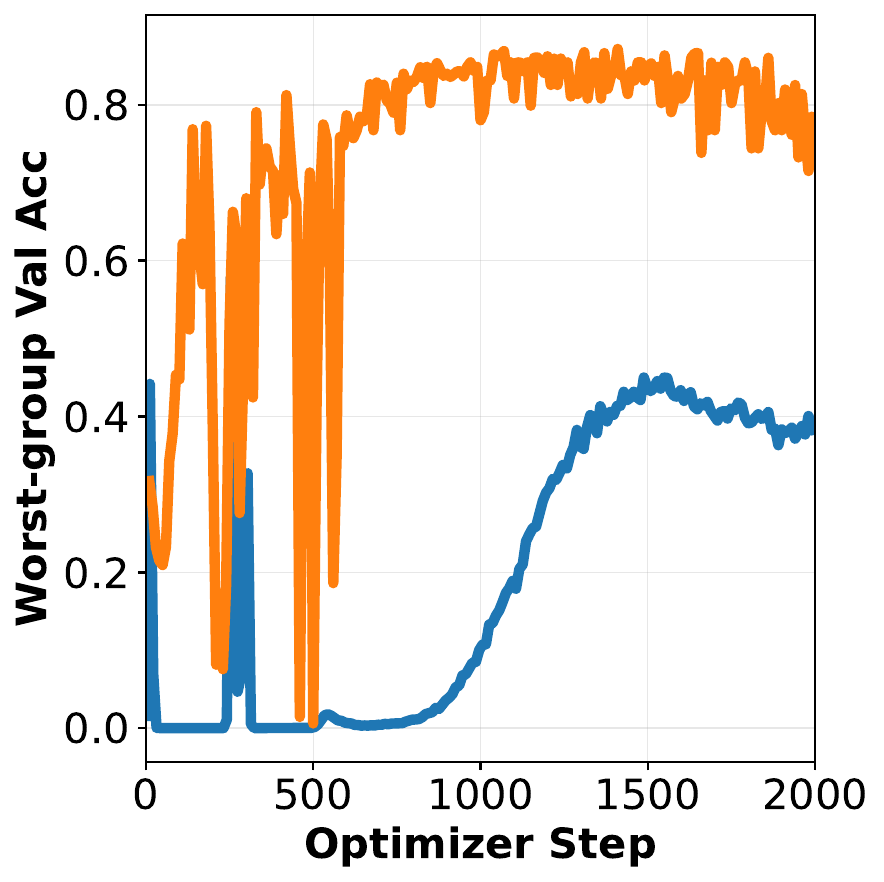}
    \caption{CelebA}
\end{subfigure}
\begin{subfigure}{0.32\linewidth}
    \centering
    \includegraphics[width=\linewidth]{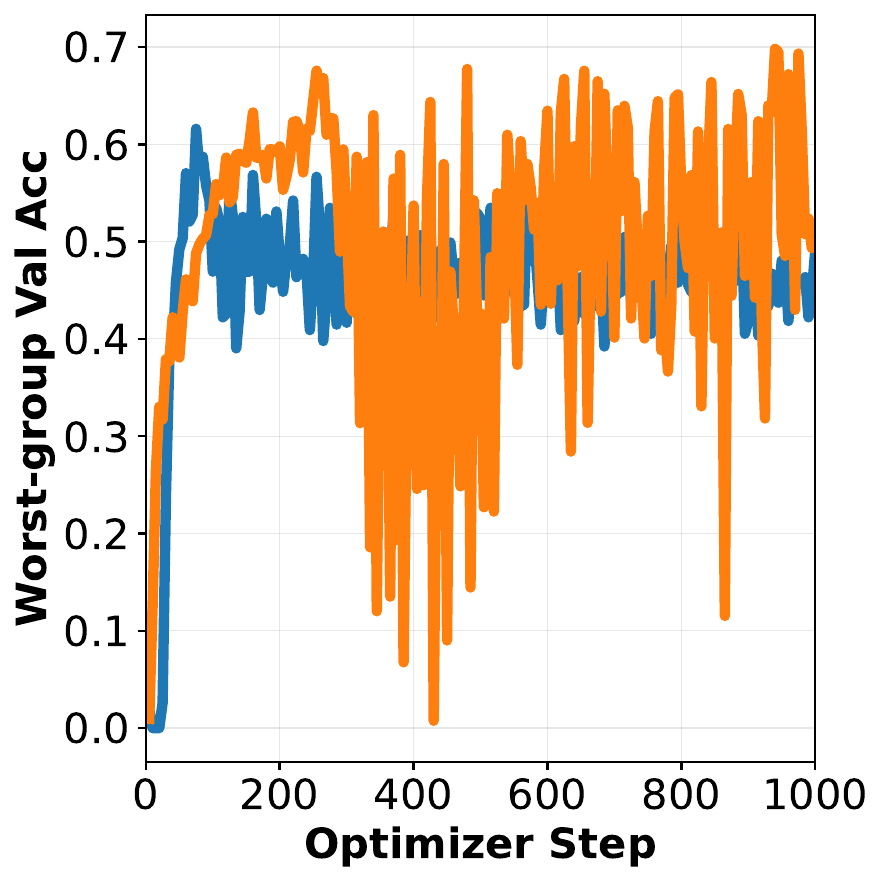}
    \caption{CivilComments}
\end{subfigure}
\caption{Validation Curve of HDRO and Bi-HDRO.}
\label{fig:hdro-valcurve}
\end{figure}
\paragraph{Runtime.} We show below in Table \ref{tab:wallclock} that Bi-HDRO takes shorter time to run. HDRO would require much longer runtime than stated if grid search is taken into account. CMNIST in general takes longer to achieve reasonable performance.
\begin{table}[htbp]
    \caption{Wall-clock time for one run in seconds.}
    \label{tab:wallclock}
    \centering
    \begin{tabular}{l ccc}
    \toprule
        &\bf CMNIST&\bf CelebA&\bf CivilComments\\
        \midrule
         \bf HDRO&405&3145&625\\
         \bf Bi-HDRO ($\epsilon_g$)&6674&129&3912\\
         \bottomrule
    \end{tabular}
\end{table}

\subsection{Does Perturbation Help Soft Group Assignments (Bi-PG-DRO)?}
\label{apdx:perturb_pgdro}
We decide not to perturb the latent space in Bi-PG-DRO because empirically, fixed perturbation alone degrades model performance, and $\epsilon_g$ converges to 0 when automatically tuned. On CMNIST, we found that perturbation $\epsilon_g$ all converged to 0 regardless of initial values (Figure \ref{fig:pg-dro-eps}). We attempted 3 setups: initial value is 0, initial value is 96/255, initial value is 0 but unbounded from above. Validation performance worsens in Bi-PG-DRO with active perturbation. We observed similar performance degradation on CelebA and CivilComments. Based on this result, we decide not to include perturbation with probabilistic group membership.
\begin{figure}[htbp]
    \centering
    \includegraphics[width=\linewidth]{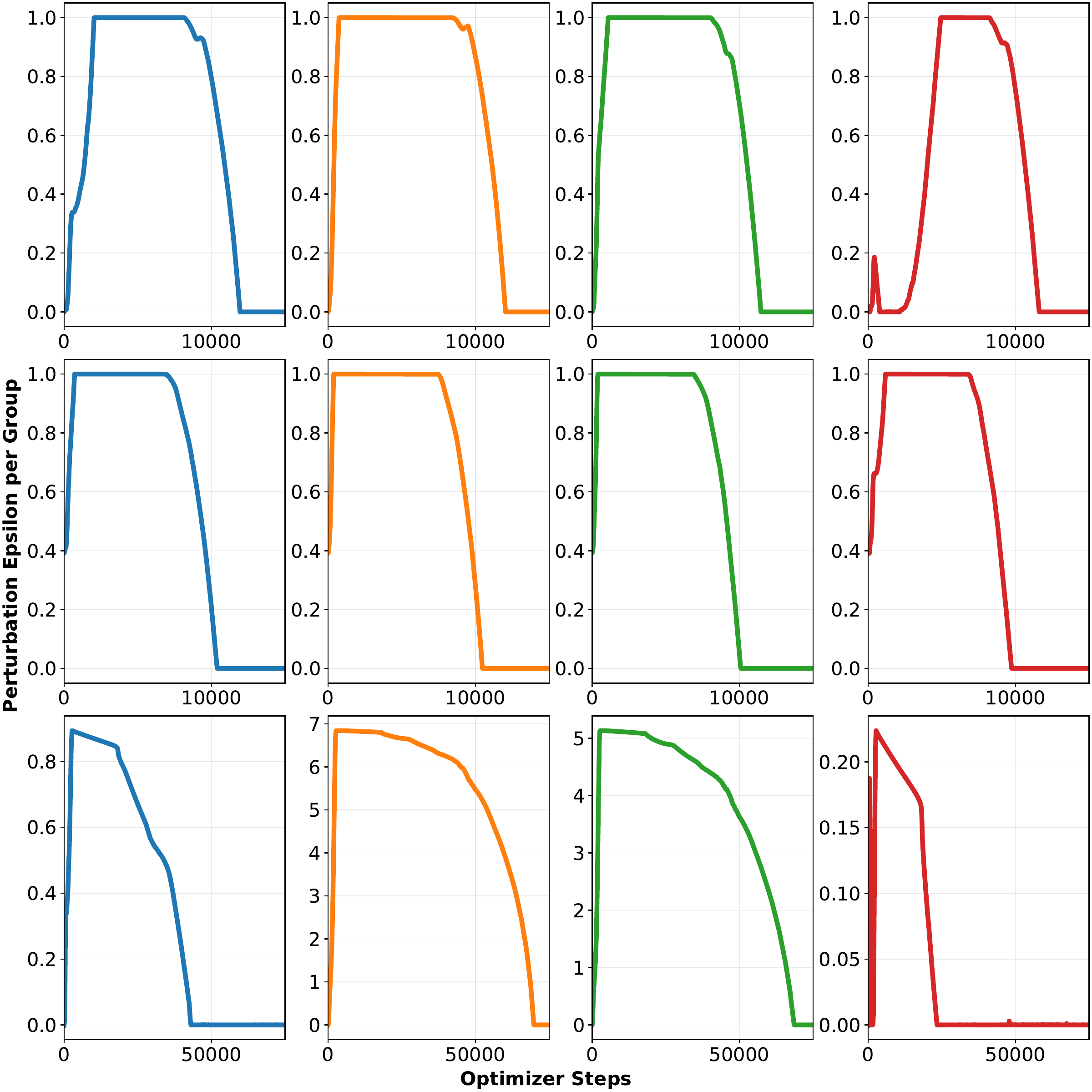}
    \caption{$\epsilon_g$ converges to 0 for Bi-PG-DRO. First row: $\epsilon_{\rm{init}}=0$. Second row: $\epsilon_{\rm{init}}=96/255$. Third row: Unclipped $\epsilon_{\rm{init}}=0$. Last column is the shifted minority group.}
    \label{fig:pg-dro-eps}
\end{figure}

\subsection{Should Membership Inference be Precise for Worst-Group Performance?}
\label{apdx:infernce_precision}
As long as worst-group validation loss is focused on during bilevel optimization, predicted group labels need not be precise. In Figure \ref{fig:attr_val_corres}, we show that post-hoc attribute accuracy over group-unlabeled training set without upper level regularization is poor, but the validation performance does not degrade as much. Moreover, we show that KL regularization instead of explicit BCE training provides closer attribute prediction accuracy results and gives a small boost in validation worst-group accuracy.

\begin{figure}[htbp]
    \centering

    \begin{subfigure}[b]{0.32\textwidth}
        \centering
        \includegraphics[width=\linewidth]{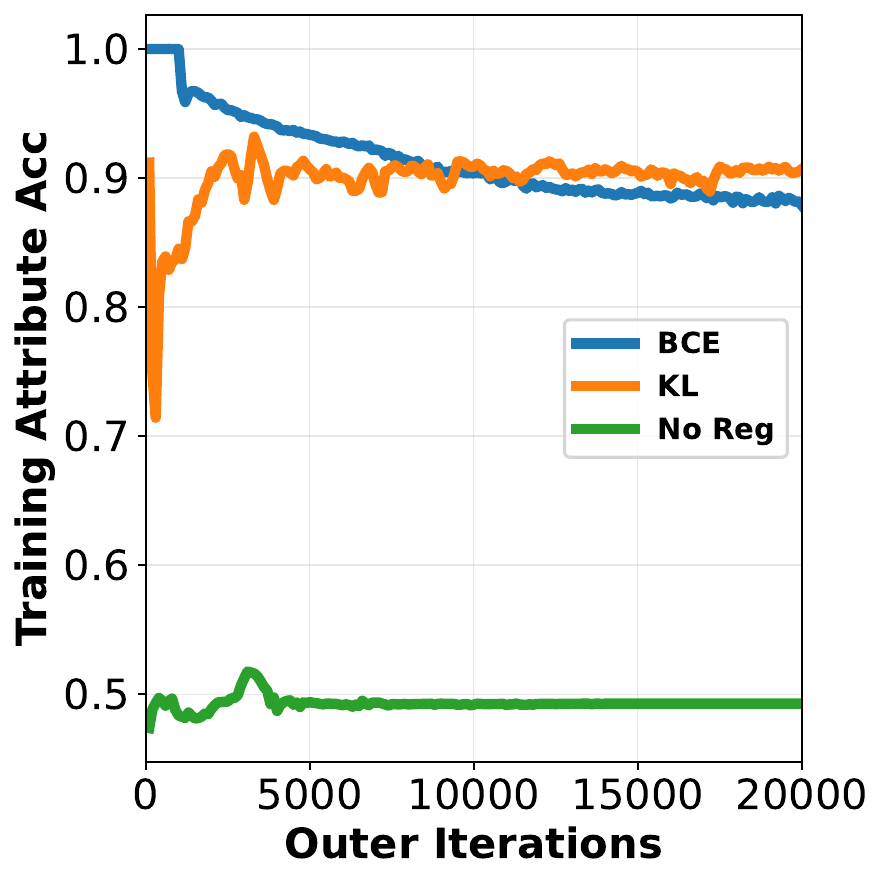}

        \par\medskip

        \includegraphics[width=\linewidth]{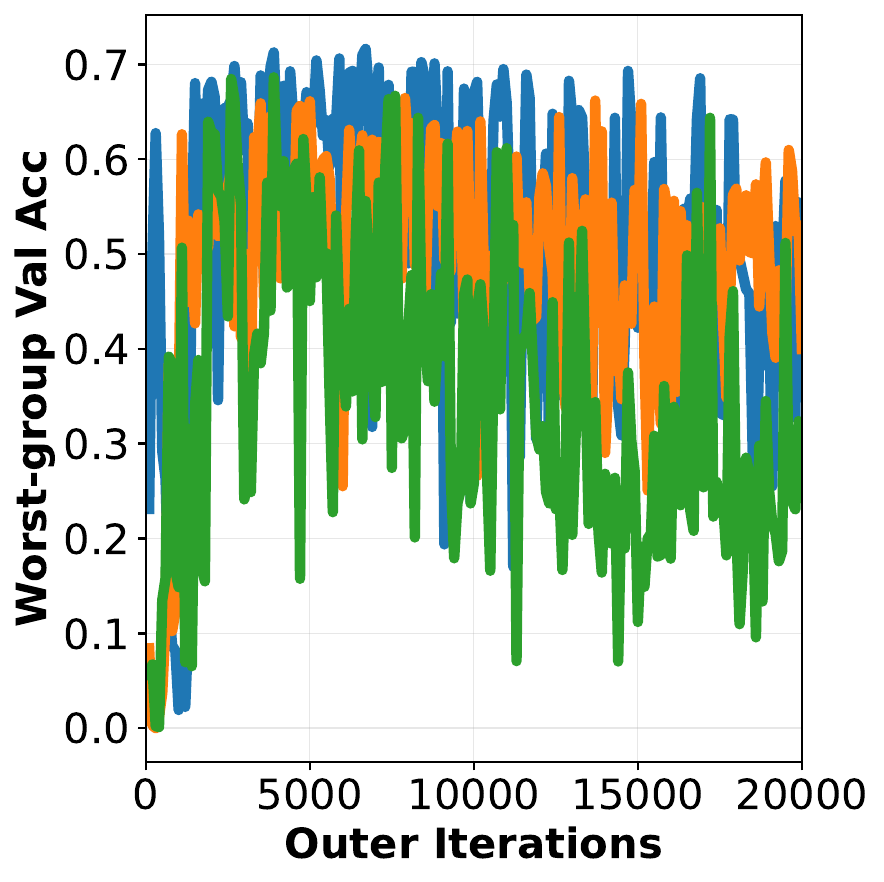}
        \caption{CMNIST}
        \label{fig:cmnist}
    \end{subfigure}
    \hfill
    \begin{subfigure}[b]{0.32\textwidth}
        \centering
        \includegraphics[width=\linewidth]{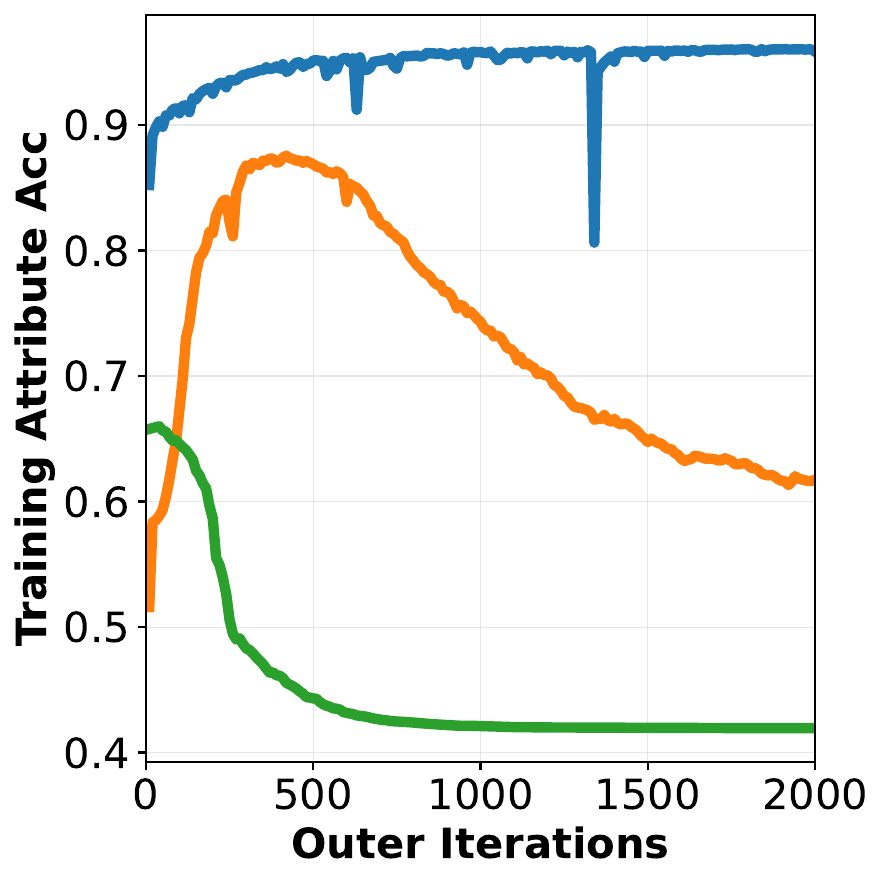}

        \par\medskip

        \includegraphics[width=\linewidth]{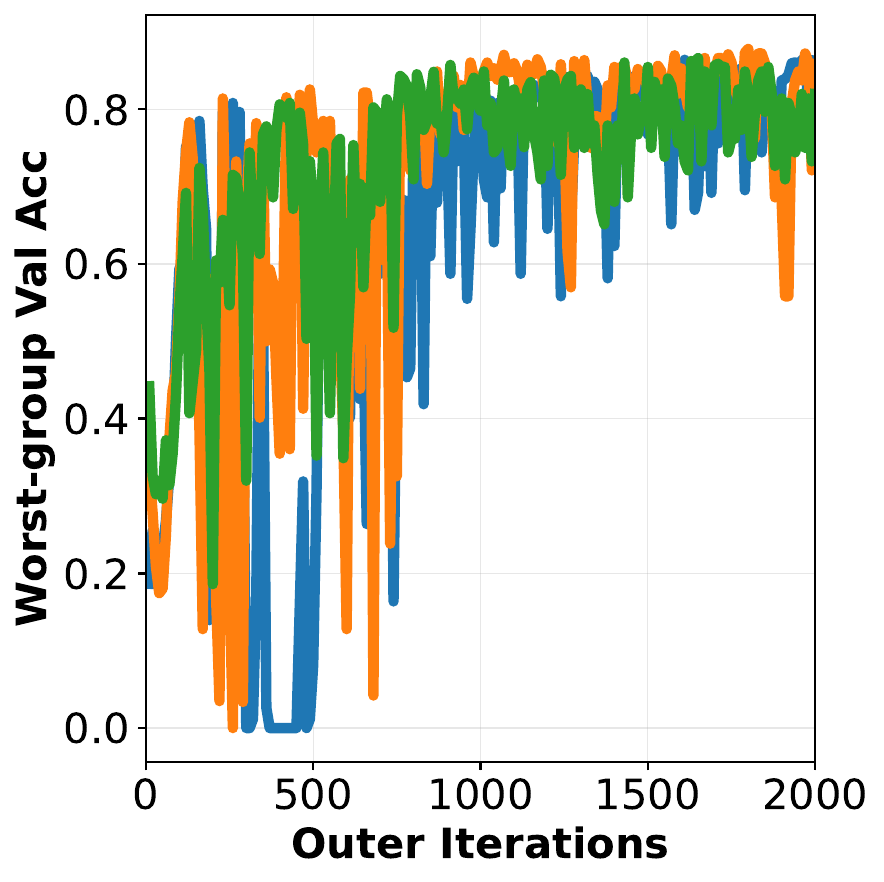}
        \caption{CelebA}
        \label{fig:celeba}
    \end{subfigure}
    \hfill
    \begin{subfigure}[b]{0.32\textwidth}
        \centering
        \includegraphics[width=\linewidth]{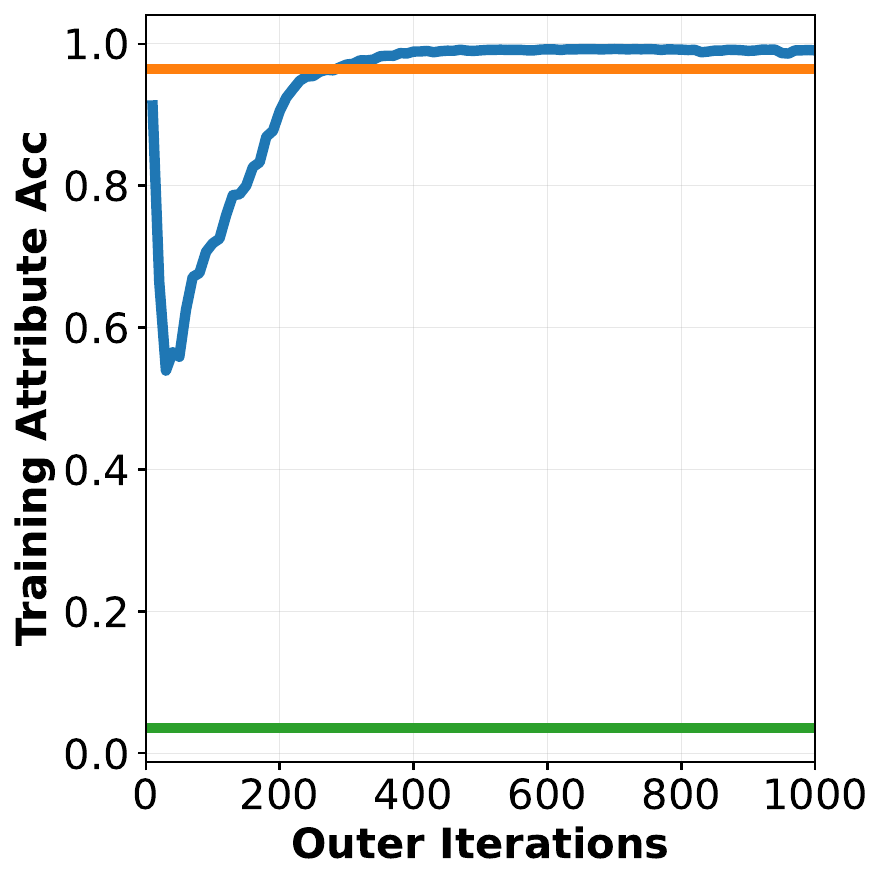}

        \par\medskip

        \includegraphics[width=\linewidth]{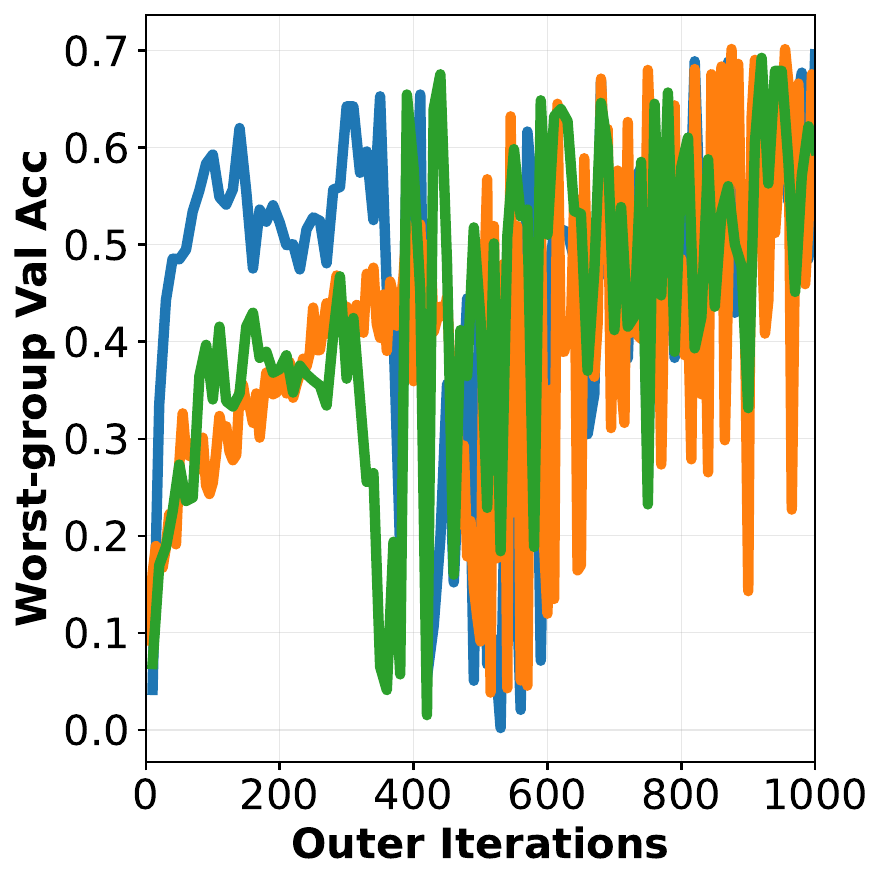}
        \caption{CivilComments}
        \label{fig:civilcomments}
    \end{subfigure}

    \caption{Attribute prediction accuracy on group-unlabeled training data (top) and worst-group validation accuracy (bottom) across three datasets organized by column.}
    \label{fig:attr_val_corres}
\end{figure}

\section{Generalization Theory for Continuous Bilevel Hyperparameter Tuning}
\label{apdx:gen}
In this section, we analyze the generalization guarantees of continuous bilevel hyperparameter tuning. Our primary theoretical goal is to establish a \emph{Continuous Oracle Inequality} for Group Distributionally Robust Optimization (DRO). This inequality characterizes how optimizing hyperparameters on a validation set allows the algorithm to seamlessly navigate the fundamental trade-off between structural complexity and worst-case group robustness.

\textbf{Roadmap.} Because establishing continuous generalization bounds requires multiple statistical learning tools, our analysis proceeds in four stages. 
\begin{enumerate}
    \item We formalize the bilevel setup and explicitly prove that the lower-level optimization algorithm induces a bounded, Lipschitz-continuous hypothesis space with respect to the hyperparameters (Appendix~\ref{sec:setup}). 
    \item We establish the mathematical machinery required to bound the upper-level validation error across the entire continuous hyperparameter path using Rademacher complexity (Appendix~\ref{sec:uniform_convergence}).
    \item As a theoretical warm-up, we apply this machinery to standard Regularized Empirical Risk Minimization, which isolates how the upper-level continuous tuning error cleanly decouples from the lower-level uniform stability (Appendix~\ref{sec:erm_theory}).
    \item We derive our main result: the Group DRO Oracle Inequality (Appendix~\ref{sec:dro_theory}).
    \item Finally, we extend this continuous generalization theory to Hierarchical DRO (Bi-HDRO), demonstrating that tuning multi-dimensional perturbation radii preserves the algorithmic Lipschitz continuity and resulting Oracle Inequalities (Appendix~\ref{sec:hdro_extension}).
\end{enumerate}

\subsection{Setup and Assumptions} \label{sec:setup}
Throughout this section, we make the following formal assumptions about the learning problem and the bilevel setup:
\begin{enumerate}
    \item \label{asm:loss} \textbf{Loss Function:} The instantaneous loss function $\ell(W, z)$ is convex and $L$-Lipschitz with respect to $W$. Furthermore, its absolute value is bounded by $M$ (i.e., $|\ell(W, z)| \le M$) over a bounded optimization domain of radius $B$ (i.e., $\|W\| \le B$).
    \item \label{asm:hyper_space} \textbf{Continuous Hyperparameter Space:} The generic hyperparameter $\boldsymbol{\phi}$ is tuned over a continuous $k$-dimensional bounded domain $\Phi \subset \mathbb{R}^k$. We assume the maximum $\ell_2$ distance between any two hyperparameters in $\Phi$ is bounded by a diameter $R$.
    \item \label{asm:strong_convexity} \textbf{Strong Convexity of Regularization:} The lower-level objective includes a regularization term $\Omega_\lambda(W)$ that is $\lambda$-strongly convex with respect to $W$.
    \item \label{asm:algo_mapping} \textbf{Algorithmic Mapping and Lipschitz Continuity:} The lower-level optimization acts as a mapping algorithm $\mathcal{A}: \Phi \to \mathbb{R}^d$ on a training set $S_{\text{train}}$ of size $n^{\mathrm{tr}}$, outputting a parameter $\widehat{W}_{\boldsymbol{\phi}} = \mathcal{A}(\boldsymbol{\phi})$. We require this mapping $\mathcal{A}$ to be $\rho_{\mathcal{A}}$-Lipschitz continuous with respect to $\boldsymbol{\phi}$ in the $\ell_2$ norm: $\|\mathcal{A}(\boldsymbol{\phi}_1) - \mathcal{A}(\boldsymbol{\phi}_2)\| \le \rho_{\mathcal{A}} \|\boldsymbol{\phi}_1 - \boldsymbol{\phi}_2\|$. While formalized here as a high-level requirement for our general theorems, we explicitly demonstrate later that this continuity is a consequence of strong convexity (Assumption~\ref{asm:strong_convexity}) for our specific learning objectives.
\end{enumerate}

Based on this mapping, the upper-level problem evaluates the induced predictors on an independent validation set $S_{\text{val}}$ of size $n^{\mathrm{val}}$. We define the effective \textbf{algorithmic hypothesis class} explored by the validation process as the $k$-dimensional continuous manifold induced by $\mathcal{A}$:
\begin{equation} \label{eq:hypothesis_class}
    \mathcal{H}_\Phi = \{ \widehat{W}_{\boldsymbol{\phi}} : \boldsymbol{\phi} \in \Phi \}
\end{equation}

To establish our generalization guarantees, we rely on the classical framework of \emph{uniform stability}. For completeness, we defer the standard formal definitions and stability derivations to Appendix~\ref{sec:uniform_stability_background}.

\subsection{Algorithmic Lipschitz Continuity Examples}

To provide concrete intuition for Assumption~\ref{asm:algo_mapping} (Algorithmic Lipschitz Continuity), we walk through two representative cases: a warm-up for standard Regularized Loss Minimization, and a formal proof for our primary application of Group DRO.

\textbf{Example 1: (Warm-up) Standard Regularized Loss Minimization.} 
Consider a standard bilevel formulation for Regularized Empirical Risk Minimization (ERM), where the goal is to tune the continuous regularization penalty. Here, the one-dimensional continuous hyperparameter is the regularization weight $\lambda \in \Lambda = [\lambda_{\min}, \lambda_{\max}]$ where $\lambda_{\min} > 0$. The lower-level objective minimizes the regularized empirical risk over the training set:
\begin{equation}
    F_\lambda(W) = L^{\mathrm{tr}}(W) + \Omega_\lambda(W)
\end{equation}
yielding the optimal predictor $\widehat{W}_\lambda = \arg\min_{W} F_\lambda(W)$. The upper-level problem evaluates these predictors to minimize the unregularized risk on a validation set: $\min_{\lambda \in \Lambda} L^{\mathrm{val}}(\widehat{W}_\lambda)$.

Because the lower-level objective $F_\lambda(W)$ is $\lambda$-strongly convex (Assumption~\ref{asm:strong_convexity}), the Lipschitz continuity of the mapping $\lambda \mapsto \widehat{W}_\lambda$ is naturally guaranteed without requiring differentiability of the loss. By the property of strong convexity at the optimum $\widehat{W}_\lambda$, for any $W$ we have $F_\lambda(W) \ge F_\lambda(\widehat{W}_\lambda) + \frac{\lambda}{2}\|W - \widehat{W}_\lambda\|^2$. 
Evaluating this for $\lambda_1$ at $W = \widehat{W}_{\lambda_2}$ and for $\lambda_2$ at $W = \widehat{W}_{\lambda_1}$ gives two inequalities:
\begin{align*}
    F_{\lambda_1}(\widehat{W}_{\lambda_2}) &\ge F_{\lambda_1}(\widehat{W}_{\lambda_1}) + \frac{\lambda_1}{2} \|\widehat{W}_{\lambda_2} - \widehat{W}_{\lambda_1}\|^2 \\
    F_{\lambda_2}(\widehat{W}_{\lambda_1}) &\ge F_{\lambda_2}(\widehat{W}_{\lambda_2}) + \frac{\lambda_2}{2} \|\widehat{W}_{\lambda_1} - \widehat{W}_{\lambda_2}\|^2
\end{align*}
Expanding $F_\lambda(W) = L^{\mathrm{tr}}(W) + \Omega_\lambda(W)$ and summing them, the empirical loss terms $L^{\mathrm{tr}}(\widehat{W}_{\lambda_1})$ and $L^{\mathrm{tr}}(\widehat{W}_{\lambda_2})$ exactly cancel out on both sides. Rearranging the remaining regularization terms leaves:
\begin{equation}
    \frac{\lambda_1 + \lambda_2}{2} \|\widehat{W}_{\lambda_1} - \widehat{W}_{\lambda_2}\|^2 \le \left( \Omega_{\lambda_1}(\widehat{W}_{\lambda_2}) - \Omega_{\lambda_1}(\widehat{W}_{\lambda_1}) \right) + \left( \Omega_{\lambda_2}(\widehat{W}_{\lambda_1}) - \Omega_{\lambda_2}(\widehat{W}_{\lambda_2}) \right)
\end{equation}
Consider the general proximal regularization case where $\Omega_\lambda(W) = \frac{\lambda}{2}\|W - \mathbf{a}\|^2$ for some reference vector $\mathbf{a}$ (standard Ridge is recovered when $\mathbf{a} = \mathbf{0}$). The right side simplifies to $\frac{\lambda_1 - \lambda_2}{2} (\|\widehat{W}_{\lambda_2} - \mathbf{a}\|^2 - \|\widehat{W}_{\lambda_1} - \mathbf{a}\|^2)$. Using the algebraic identity $\|x\|^2 - \|y\|^2 = \langle x - y, x + y \rangle$, we can rewrite this difference as:
\begin{equation*}
    \frac{\lambda_1 - \lambda_2}{2} \langle \widehat{W}_{\lambda_2} - \widehat{W}_{\lambda_1}, \widehat{W}_{\lambda_2} + \widehat{W}_{\lambda_1} - 2\mathbf{a} \rangle
\end{equation*}
Applying the Cauchy-Schwarz inequality, and noting that the average $\frac{\lambda_1 + \lambda_2}{2}$ is strictly lower-bounded by $\lambda_{\min}$, yields:
\begin{equation}
    \lambda_{\min} \|\widehat{W}_{\lambda_1} - \widehat{W}_{\lambda_2}\|^2 \le \frac{|\lambda_1 - \lambda_2|}{2} \|\widehat{W}_{\lambda_1} - \widehat{W}_{\lambda_2}\| \|\widehat{W}_{\lambda_1} + \widehat{W}_{\lambda_2} - 2\mathbf{a}\|
\end{equation}
By the triangle inequality, and assuming the optimization domain is bounded by a constant radius $B$ (Assumption~\ref{asm:loss}), the second term is bounded by $\|\widehat{W}_{\lambda_1}\| + \|\widehat{W}_{\lambda_2}\| + 2\|\mathbf{a}\| \le 2(B + \|\mathbf{a}\|)$. Dividing by $\|\widehat{W}_{\lambda_1} - \widehat{W}_{\lambda_2}\|$ proves that the mapping $\mathcal{A}$ is strictly bounded by a constant $\rho_{\mathcal{A}} = (B + \|\mathbf{a}\|) / \lambda_{\min}$ and is therefore $\rho_{\mathcal{A}}$-Lipschitz:
\begin{equation}
    \|\widehat{W}_{\lambda_1} - \widehat{W}_{\lambda_2}\| \le \rho_{\mathcal{A}} |\lambda_1 - \lambda_2|
\end{equation}
Alternatively, if the loss and regularizer are strictly twice continuously differentiable, one can recover this same Lipschitz constant $\rho_{\mathcal{A}} \le B / \lambda_{\min}$ by directly bounding the spectral norm of the Implicit Function Theorem Jacobian: $\frac{d\widehat{W}_\lambda}{d\lambda} = - [ \nabla_{W}^2 F_\lambda(\widehat{W}_\lambda) ]^{-1} \nabla_{\lambda W}^2 F_\lambda(\widehat{W}_\lambda)$.

\textbf{Example 2: Group Distributionally Robust Optimization (DRO).} As a concrete application, consider a linearized Group DRO setting where we optimize a linear classifier $W$ directly on fixed inputs. Given training data partitioned into $G$ groups with empirical group losses $L^{\mathrm{tr}}(W) \in \mathbb{R}^G$, the lower level optimizes the classifier against a worst-case group distribution $q \in \Delta_G$ (the probability simplex weighting the groups). The deviation of this worst-case distribution from the uniform distribution is penalized by a robustness hyperparameter $\eta \in H = [\eta_{\min}, \eta_{\max}]$ where $\eta_{\min} > 0$. The lower-level problem is:
\begin{equation}
    \widehat{W}_\eta = \arg\min_{W} \max_{q \in \Delta_G} \left[ q^T L^{\mathrm{tr}}(W) - \frac{\eta}{2} \left\| q - \frac{1}{G}\mathbf{1} \right\|^2 + \frac{\lambda}{2} \|W\|^2 \right]
\end{equation}
where $\lambda$ is a fixed $L_2$ regularization coefficient. In the upper level, we tune the robustness hyperparameter $\eta \in H$ using an independent validation set $D_{\text{val}}$ (partitioned into $G$ groups) to minimize the worst-group validation loss. Here, $\eta$ naturally assumes the role of the generic continuous tuning hyperparameter $\boldsymbol{\phi}$ (with $k=1$) from Section~\ref{sec:setup}, while $\lambda$ serves strictly as the fixed strong convexity constant required for stability:
\begin{equation}
    \min_{\eta \in H} \max_{g \in [G]} L^{\mathrm{val}}_g(\widehat{W}_\eta)
\end{equation}
By continuously tuning $\eta$, the bilevel formulation dynamically discovers the optimal trade-off between average-case empirical risk and worst-group robustness on unseen data.

Remarkably, this formulation strictly satisfies the algorithmic Lipschitz condition (Assumption~\ref{asm:algo_mapping}). We formally state this property below.

\begin{lemma}[Algorithmic Lipschitz Continuity of Group DRO]
\label{lem:dro_lipschitz}
Suppose the instantaneous group losses $L^{\mathrm{tr}}_g(W)$ are $L$-Lipschitz and bounded by $M$ over a bounded optimization domain of radius $B$ (Assumption~\ref{asm:loss}). For any fixed strong convexity regularizer $\lambda > 0$ (Assumption~\ref{asm:strong_convexity}), the max-marginalized lower-level objective $J_\eta(W)$ induces an algorithmic mapping $\widehat{W}_\eta$ that is $\rho_{\mathcal{A}}$-Lipschitz continuous with respect to the robustness hyperparameter $\eta \ge \eta_{\min} > 0$, with explicit constant:
\begin{equation}
    \rho_{\mathcal{A}} \le \frac{G L M}{\lambda \eta_{\min}^2}
\end{equation}
\end{lemma}

\begin{proof}
Let $J_\eta(W) = \max_{q \in \Delta_G} F(W, q, \eta)$ be the max-marginalized lower-level objective, where $F(W, q, \eta) = q^T L^{\mathrm{tr}}(W) - \frac{\eta}{2}\|q - \frac{1}{G}\mathbf{1}\|^2 + \frac{\lambda}{2}\|W\|^2$. By Danskin's theorem \citep{danskin1967theory}, the gradient of the max-marginalized function is simply the gradient of the objective evaluated at the optimal inner variable. Thus, its gradient is $\nabla J_\eta(W) = \nabla_{W} F(W, q^*_\eta(W), \eta) = J(W) q^*_\eta(W) + \lambda W$, where $J(W)$ is the $d \times G$ Jacobian matrix of group losses. By Assumption~\ref{asm:loss}, each group loss is $L$-Lipschitz, so the operator norm of $J(W)$ is bounded by its Frobenius norm $\sqrt{\sum_{g=1}^G \|\nabla L^{\mathrm{tr}}_g(W)\|^2} \le \sqrt{G}L$.

Next, we bound how much this gradient shifts with respect to $\eta$. The optimal inner distribution is exactly the simplex projection $q^*_\eta(W) = \Pi_{\Delta_G}\left( \mathbf{u}_\eta(W) \right)$, where $\mathbf{u}_\eta(W) = \frac{1}{\eta}L^{\mathrm{tr}}(W) + \frac{1}{G}\mathbf{1}$. Because the projection $\Pi_{\Delta_G}$ is non-expansive (1-Lipschitz), the shift in the optimal inner distribution is bounded by the shift in the unprojected vector: $\|q^*_{\eta_1}(W) - q^*_{\eta_2}(W)\| \le \|\mathbf{u}_{\eta_1}(W) - \mathbf{u}_{\eta_2}(W)\| = \big\| (\frac{1}{\eta_1} - \frac{1}{\eta_2})L^{\mathrm{tr}}(W) \big\|$. By Assumption~\ref{asm:loss}, each group loss is bounded by $M$, so $\|L^{\mathrm{tr}}(W)\| \le \sqrt{G}M$. Requiring $\eta \ge \eta_{\min} > 0$, we have $|1/\eta_1 - 1/\eta_2| \le |\eta_1 - \eta_2|/\eta_{\min}^2$. This directly provides the shift bound: $\|q^*_{\eta_1}(W) - q^*_{\eta_2}(W)\| \le \frac{\sqrt{G}M}{\eta_{\min}^2} |\eta_1 - \eta_2|$, which in turn limits the overall gradient shift to $\| \nabla J_{\eta_1}(W) - \nabla J_{\eta_2}(W) \| \le \frac{G L M}{\eta_{\min}^2} |\eta_1 - \eta_2|$.

Finally, let $\widehat{W}_{\eta_1}$ and $\widehat{W}_{\eta_2}$ be the optimal lower-level classifiers for hyperparameters $\eta_1$ and $\eta_2$. Because $J_\eta$ is $\lambda$-strongly convex, its gradient is $\lambda$-strongly monotone:
\begin{equation*}
    \langle \nabla J_{\eta_1}(\widehat{W}_{\eta_1}) - \nabla J_{\eta_1}(\widehat{W}_{\eta_2}), \widehat{W}_{\eta_1} - \widehat{W}_{\eta_2} \rangle \ge \lambda \|\widehat{W}_{\eta_1} - \widehat{W}_{\eta_2}\|^2
\end{equation*}
For any closed convex domain (in particular $\|W\|\le B$), the first-order optimality condition dictates:
\begin{align*}
    \langle - \nabla J_{\eta_1}(\widehat{W}_{\eta_1}), \widehat{W}_{\eta_1} - \widehat{W}_{\eta_2} \rangle &\ge 0 \\
    \langle \nabla J_{\eta_2}(\widehat{W}_{\eta_2}), \widehat{W}_{\eta_1} - \widehat{W}_{\eta_2} \rangle &\ge 0
\end{align*}
Summing the above three inequalities and applying the Cauchy-Schwarz inequality yields:
\begin{align}
    \lambda \|\widehat{W}_{\eta_1} - \widehat{W}_{\eta_2}\|^2 &\le \langle \nabla J_{\eta_2}(\widehat{W}_{\eta_2}) - \nabla J_{\eta_1}(\widehat{W}_{\eta_2}), \widehat{W}_{\eta_1} - \widehat{W}_{\eta_2} \rangle \nonumber \\
    &\le \| \nabla J_{\eta_2}(\widehat{W}_{\eta_2}) - \nabla J_{\eta_1}(\widehat{W}_{\eta_2}) \| \|\widehat{W}_{\eta_1} - \widehat{W}_{\eta_2}\|
\end{align}
Dividing by $\lambda \|\widehat{W}_{\eta_1} - \widehat{W}_{\eta_2}\|$ and plugging in our bound for the gradient shift establishes the explicit Lipschitz constant $\rho_{\mathcal{A}}$.
\end{proof}

Thus, tuning the DRO robustness parameter over a continuous space is theoretically well-behaved and Lipschitz-stable as long as the search space is bounded away from zero ($\eta \ge \eta_{\min} > 0$).


\subsection{Uniform Convergence over Multi-dimensional Continuous Hypothesis Spaces} \label{sec:uniform_convergence}

Before specializing to specific learning algorithms like ERM or Group DRO, we first establish a general uniform convergence bound for the $k$-dimensional continuous hypothesis class $\mathcal{H}_\Phi$. By leveraging the Lipschitz continuity of the algorithmic mapping, we can bound the Rademacher complexity of this manifold as a function of the hyperparameter space dimensionality $k$.

\begin{lemma}[Rademacher Complexity of $k$-dimensional Algorithmic Hypothesis Class]
\label{lem:general_rademacher}
Suppose the loss function $\ell$ is $L$-Lipschitz and bounded by $M$ (Assumption~\ref{asm:loss}). Let $\Phi \subset \mathbb{R}^k$ be a bounded $k$-dimensional hyperparameter space with $\ell_2$ diameter $R$ (Assumption~\ref{asm:hyper_space}). If the lower-level mapping $\mathcal{A}(\boldsymbol{\phi}) = \widehat{W}_{\boldsymbol{\phi}}$ is $\rho_{\mathcal{A}}$-Lipschitz (Assumption~\ref{asm:algo_mapping}), the empirical Rademacher complexity of the validation loss class over $\mathcal{H}_\Phi$ on a set $S_{\text{val}}$ of size $n^{\mathrm{val}}$ is bounded by:
\begin{equation}
    \hat{\mathcal{R}}_{S_{\text{val}}}(\mathcal{H}_\Phi) \le \frac{6M}{\sqrt{n^{\mathrm{val}}}} \sqrt{k} \left( \sqrt{\log \left( \frac{3\rho_{\mathcal{A}} R L}{M} \right)} + 2\sqrt{\log 2} \right)
\end{equation}
\end{lemma}
\begin{proof}
The proof relies on bounding the continuous covering number and applying discrete chaining. For a $k$-dimensional space $\Phi$ with $\ell_2$ diameter $R$, its $r$-covering number in $\ell_2$ norm is bounded by $N(r, \Phi) \le (3R/r)^k$ for $r \le R$. This follows from a standard volumetric argument: a maximal $r$-separated set of size $N$ in $\Phi$ induces $N$ disjoint balls of radius $r/2$. Since $\Phi$ has diameter $R$, these disjoint balls are entirely contained within a larger ball of radius $R + r/2$. Comparing their volumes yields $N(r/2)^k \le (R + r/2)^k$, which implies $N \le (1 + 2R/r)^k \le (3R/r)^k$ when $r \le R$.
Due to the $\rho_{\mathcal{A}}$-Lipschitz property of the mapping, an $r$-cover of $\Phi$ projects to an $(\rho_{\mathcal{A}} \cdot r)$-cover of the hypothesis class $\mathcal{H}_\Phi$. Thus, the covering number of the effective hypothesis class is bounded by $N(r, \mathcal{H}_\Phi) \le (3\rho_{\mathcal{A}} R / r)^k$.

Let $A \subset \mathbb{R}^{n^{\mathrm{val}}}$ be the set of loss evaluations on $S_{\text{val}} = \{z_1, \dots, z_{n^{\mathrm{val}}}\}$ for all predictors in $\mathcal{H}_\Phi$. The maximum $\ell_2$ norm of any vector in $A$ is $C = M\sqrt{n^{\mathrm{val}}}$. Since the loss is $L$-Lipschitz, the $\ell_2$ distance between two loss vectors $\mathbf{a}_{\boldsymbol{\phi}_1}$ and $\mathbf{a}_{\boldsymbol{\phi}_2}$ generated by predictors $\widehat{W}_{\boldsymbol{\phi}_1}$ and $\widehat{W}_{\boldsymbol{\phi}_2}$ is bounded by $\|\mathbf{a}_{\boldsymbol{\phi}_1} - \mathbf{a}_{\boldsymbol{\phi}_2}\|_2 \le L \sqrt{n^{\mathrm{val}}} \|\widehat{W}_{\boldsymbol{\phi}_1} - \widehat{W}_{\boldsymbol{\phi}_2}\|_2$. Therefore, guaranteeing an $\epsilon$-cover of $A$ requires at most an $r$-cover of $\mathcal{H}_\Phi$ with $r = \frac{\epsilon}{L \sqrt{n^{\mathrm{val}}}}$. 
Thus, the covering number of $A$ is bounded by $N(\epsilon, A) \le N(r, \mathcal{H}_\Phi) \le \left( \frac{3\rho_{\mathcal{A}} R L \sqrt{n^{\mathrm{val}}}}{\epsilon} \right)^k$.

By the discrete chaining lemma~\citep[Lemma 27.4]{shalev2014understanding}, we evaluate the complexity over discrete scales $\epsilon_i = C 2^{-i}$. For any $i \ge 1$:
\begin{equation}
    \sqrt{\log N(C 2^{-i}, A)} \le \sqrt{k \log \left( \frac{3\rho_{\mathcal{A}} R L \sqrt{n^{\mathrm{val}}}}{M\sqrt{n^{\mathrm{val}}}} 2^i \right)} \le \sqrt{k} (\alpha + \beta i)
\end{equation}
where $\alpha = \sqrt{\log\left( \frac{3\rho_{\mathcal{A}} R L}{M} \right)}$ and $\beta = \sqrt{\log 2}$. Evaluating the full infinite chaining sum gives the explicit Rademacher complexity ~\citep[Lemma 27.5]{shalev2014understanding}:
\begin{equation}
    \hat{\mathcal{R}}_{S_{\text{val}}}(\mathcal{H}_\Phi) \le \frac{6C}{n^{\mathrm{val}}} \sqrt{k} (\alpha + 2\beta) = \frac{6M}{\sqrt{n^{\mathrm{val}}}} \sqrt{k} \left( \sqrt{\log \left( \frac{3\rho_{\mathcal{A}} R L}{M} \right)} + 2\sqrt{\log 2} \right)
\end{equation}
\end{proof}

This lemma provides a deterministic uniform convergence bound. By applying standard Rademacher concentration bounds \citep[Theorem 26.5]{shalev2014understanding} combined with McDiarmid's inequality, the uniform deviation $\sup_{\boldsymbol{\phi} \in \Phi} |L_{\mathcal{D}}(\widehat{W}_{\boldsymbol{\phi}}) - L^{\mathrm{val}}(\widehat{W}_{\boldsymbol{\phi}})| \le \epsilon_{\text{val}}$ holds with probability $1-\delta$, where:
\begin{equation}
\label{eq:general_uniform_deviation}
    \epsilon_{\text{val}} = \frac{12M}{\sqrt{n^{\mathrm{val}}}} \sqrt{k} \left( \sqrt{\log \left( \frac{3\rho_{\mathcal{A}} R L}{M} \right)} + 2\sqrt{\log 2} \right) + M \sqrt{\frac{2\log(2/\delta)}{n^{\mathrm{val}}}}
\end{equation}

\subsection{Warm-up: Continuous Oracle Inequality for Regularized ERM} \label{sec:erm_theory}
In this subsection, we formalize the generalization guarantees for Regularized Empirical Risk Minimization (ERM), where the generic continuous tuning parameter is instantiated as the $L_2$ regularization coefficient (i.e., we set $\boldsymbol{\phi} = \lambda$ with $k=1$). In this setting, the lower level trains a regularized model $\widehat{W}_\lambda$ on a training set $S_{\text{train}} = \{z_1^{\text{train}}, \dots, z_{n^{\mathrm{tr}}}^{\text{train}}\}$ drawn i.i.d. from a single data distribution $\mathcal{D}$:
\begin{equation}
    \widehat{W}_\lambda = \arg\min_{W} \left( \frac{1}{n^{\mathrm{tr}}} \sum_{i=1}^{n^{\mathrm{tr}}} \ell(W, z_i^{\text{train}}) + \Omega_\lambda(W) \right)
\end{equation}
The upper level evaluates this continuous hypothesis class $\mathcal{H}_\Lambda = \{ \widehat{W}_\lambda : \lambda \in \Lambda \}$ on an independent validation set $S_{\text{val}} = \{z_1^{\text{val}}, \dots, z_{n^{\mathrm{val}}}^{\text{val}}\}$ also drawn from $\mathcal{D}$:
\begin{equation}
    \hat{\lambda} = \arg\min_{\lambda \in \Lambda} \frac{1}{n^{\mathrm{val}}} \sum_{j=1}^{n^{\mathrm{val}}} \ell(\widehat{W}_\lambda, z_j^{\text{val}})
\end{equation}
This serves as the foundational continuous oracle inequality, clearly distinct from the worst-case Group DRO formulation analyzed subsequently.

The following theorem combines the lower-level high-probability generalization bound for a fixed $\lambda$ with the uniform convergence bound over the 1-dimensional continuous hypothesis class $\mathcal{H}_\Lambda$.

\begin{theorem}[Continuous Oracle Inequality for ERM]
\label{thm:erm_oracle}
Suppose Assumptions~\ref{asm:loss}--\ref{asm:algo_mapping} hold: the loss $\ell$ is $L$-Lipschitz and bounded by $M$, the regularization $\Omega_\lambda$ is $\lambda$-strongly convex, the continuous tuning interval $\Lambda$ has length $R$, and the lower-level mapping $\mathcal{A}(\lambda) = \widehat{W}_\lambda$ is $\rho_{\mathcal{A}}$-Lipschitz. Let $W^*$ be an arbitrary reference predictor (e.g., the population risk minimizer). Let $\hat{\lambda} \in \Lambda$ be the hyperparameter chosen by minimizing the validation risk over $\mathcal{H}_\Lambda$. With probability at least $1 - \delta$ over the random draw of both $S_{\text{train}}$ and $S_{\text{val}}$, the true risk $L_{\mathcal{D}}(\widehat{W}_{\hat{\lambda}}) = \mathbb{E}_{z \sim \mathcal{D}}[\ell(\widehat{W}_{\hat{\lambda}}, z)]$ satisfies:
\begin{align*}
    L_{\mathcal{D}}(\widehat{W}_{\hat{\lambda}}) \le L_{\mathcal{D}}(W^*) &+ \min_{\lambda \in \Lambda} \Bigg( \Omega_{\lambda}(W^*) + \frac{2L^2}{\lambda n^{\mathrm{tr}}} + \left( \frac{4L^2}{\lambda n^{\mathrm{tr}}} + \frac{4M}{n^{\mathrm{tr}}} \right) \sqrt{\frac{n^{\mathrm{tr}} \log(4/\delta)}{2}} \Bigg) \\
    &+ \frac{24M}{\sqrt{n^{\mathrm{val}}}} \left( \sqrt{\log \left( \frac{3\rho_{\mathcal{A}} R L}{M} \right)} + 2\sqrt{\log 2} \right) + 2M \sqrt{\frac{2\log(4/\delta)}{n^{\mathrm{val}}}}
\end{align*}
\end{theorem}

\textbf{Interpretation of the Continuous Oracle Inequality.}
In standard Regularized Loss Minimization (as noted in \citet[Corollary 13.8]{shalev2014understanding}), finding the optimal hyperparameter $\lambda$ requires prior knowledge of the optimal predictor's norm $\|W^*\|$. If $\|W^*\|$ is known, one can analytically set $\lambda$ to perfectly balance the bias (the $\Omega_\lambda(W^*)$ term) and the variance (the stability gap, which scales as $O(1/\lambda m)$), thereby achieving an optimal generalization rate of $O(1/\sqrt{m})$.

However, in practice, the true optimal predictor $W^*$ and its norm are unknown. The standard approach to circumvent this is Structural Risk Minimization (SRM), where one trains models on a finite, discrete grid of $\lambda$ values and selects the best one using a validation set. While SRM guarantees learning, it restricts the solution to the predefined grid, introducing discretization error. Furthermore, as established by \citet[Theorem 11.2]{shalev2014understanding}, to theoretically guarantee that the validation set does not overfit to any model in the finite grid, one applies the union bound. Plugging the union bound over all $|Grid|$ discrete options into Hoeffding's inequality incurs a uniform convergence penalty scaling with $O\left(\sqrt{\log(|Grid|)/n^{\mathrm{val}}}\right)$. Consequently, attempting to reduce discretization error by making the grid denser degrades the theoretical generalization guarantee.

Our Continuous Oracle Inequality demonstrates that continuous bilevel tuning effectively acts as an ``oracle,'' automatically discovering the theoretically optimal bias-variance tradeoff $\min_{\lambda \in \Lambda} ( \Omega_\lambda(W^*) + \epsilon_{\text{train}}(\lambda) )$ for the \textit{unknown} reference predictor $W^*$. Crucially, it achieves this without discretization error, paying only a logarithmic uniform convergence penalty $O\big( \sqrt{\log(\rho_{\mathcal{A}} R L / M)/n^{\mathrm{val}}} \big)$. Conceptually, the inner term $\rho_{\mathcal{A}} R L / M$ acts as an ``effective grid size''---representing the finite number of distinguishable models within the continuous interval---allowing us to bypass discretization error while paying a statistical penalty no worse than a dense discrete grid. Furthermore, this continuous formulation allows us to directly traverse the hyperparameter space using efficient continuous optimization techniques, such avoiding the prohibitive computational cost of repeatedly training independent models associated with standard grid search.

\begin{proof}
Given the Lipschitz continuity of the algorithmic mapping established in the assumptions, the proof decomposes the true risk using the uniform convergence of the validation loss (via our general Rademacher bound) and the stability of the lower-level algorithm.

\textbf{Step 1: Upper-Level Uniform Convergence.}
By substituting $k=1$ into the general uniform deviation bound derived in Equation~\eqref{eq:general_uniform_deviation}, the two-sided uniform deviation $\sup_{\lambda \in \Lambda} |L_{\mathcal{D}}(\widehat{W}_\lambda) - L^{\mathrm{val}}(\widehat{W}_\lambda)| \le \epsilon_{\text{val}}$ holds with probability $1-\delta/2$, where:
\begin{equation}
    \epsilon_{\text{val}} = \frac{12M}{\sqrt{n^{\mathrm{val}}}} \left( \sqrt{\log \left( \frac{3\rho_{\mathcal{A}} R L}{M} \right)} + 2\sqrt{\log 2} \right) + M \sqrt{\frac{2\log(4/\delta)}{n^{\mathrm{val}}}}
\end{equation}

\textbf{Step 2: Final Continuous Oracle Inequality.}
Let $W^*$ be any fixed reference predictor (such as the population risk minimizer). We define the optimal regularization parameter for this predictor as $\lambda^* = \arg\min_{\lambda \in \Lambda} \left( \Omega_\lambda(W^*) + \epsilon_{\text{train}}(\lambda) \right)$, where $\epsilon_{\text{train}}(\lambda)$ will be defined shortly. Since $\hat{\lambda}$ minimizes the empirical validation risk, we have $L^{\mathrm{val}}(\widehat{W}_{\hat{\lambda}}) \le L^{\mathrm{val}}(\widehat{W}_{\lambda^*})$. Applying the uniform deviation bound (which holds over all $\mathcal{H}_\Lambda$ with probability $1-\delta/2$) to both $\hat{\lambda}$ and $\lambda^*$ yields:
\begin{align*}
    L_{\mathcal{D}}(\widehat{W}_{\hat{\lambda}}) &\le L^{\mathrm{val}}(\widehat{W}_{\hat{\lambda}}) + \epsilon_{\text{val}} \\
    &\le L^{\mathrm{val}}(\widehat{W}_{\lambda^*}) + \epsilon_{\text{val}} \\
    &\le L_{\mathcal{D}}(\widehat{W}_{\lambda^*}) + 2\epsilon_{\text{val}}
\end{align*}
To bound $L_{\mathcal{D}}(\widehat{W}_{\lambda^*})$, we compare it against the fixed $W^*$. By the empirical optimality of $\widehat{W}_{\lambda^*}$, we have $L^{\mathrm{tr}}(\widehat{W}_{\lambda^*}) + \Omega_{\lambda^*}(\widehat{W}_{\lambda^*}) \le L^{\mathrm{tr}}(W^*) + \Omega_{\lambda^*}(W^*)$. 
Because $\Omega_\lambda \ge 0$, we can decompose the true risk as:
\begin{align*}
    L_{\mathcal{D}}(\widehat{W}_{\lambda^*}) &= L^{\mathrm{tr}}(\widehat{W}_{\lambda^*}) + \left( L_{\mathcal{D}}(\widehat{W}_{\lambda^*}) - L^{\mathrm{tr}}(\widehat{W}_{\lambda^*}) \right) \\
    &\le L^{\mathrm{tr}}(W^*) + \Omega_{\lambda^*}(W^*) + \left( L_{\mathcal{D}}(\widehat{W}_{\lambda^*}) - L^{\mathrm{tr}}(\widehat{W}_{\lambda^*}) \right) \\
    &= L_{\mathcal{D}}(W^*) + \Omega_{\lambda^*}(W^*) + \underbrace{\left( L_{\mathcal{D}}(\widehat{W}_{\lambda^*}) - L^{\mathrm{tr}}(\widehat{W}_{\lambda^*}) \right)}_{\text{Stability gap}} + \underbrace{\left( L^{\mathrm{tr}}(W^*) - L_{\mathcal{D}}(W^*) \right)}_{\text{Hoeffding gap}}
\end{align*}
By Lemma~\ref{lem:stability_high_prob} with probability $1-\delta/4$, the stability gap is bounded by:
\begin{equation*}
    \frac{2L^2}{\lambda^* n^{\mathrm{tr}}} + \left( \frac{4L^2}{\lambda^* n^{\mathrm{tr}}} + \frac{2M}{n^{\mathrm{tr}}} \right) \sqrt{\frac{n^{\mathrm{tr}} \log(4/\delta)}{2}}
\end{equation*}
Simultaneously, since $W^*$ is fixed independent of $S_{\text{train}}$, we apply Hoeffding's inequality \citep[Lemma B.6]{shalev2014understanding}. Because the absolute loss is bounded by $M$, the loss variables fall in a range of $2M$. For a one-sided bound with local confidence $\delta_{\text{local}} = \delta/4$, Hoeffding's inequality exactly bounds the second gap by $M \sqrt{\frac{2\log(4/\delta)}{n^{\mathrm{tr}}}} = \frac{2M}{n^{\mathrm{tr}}} \sqrt{\frac{n^{\mathrm{tr}} \log(4/\delta)}{2}}$ with probability $1-\delta/4$. Summing these bounds gives:
\begin{equation}
    L_{\mathcal{D}}(\widehat{W}_{\lambda^*}) \le L_{\mathcal{D}}(W^*) + \Omega_{\lambda^*}(W^*) + \underbrace{\frac{2L^2}{\lambda^* n^{\mathrm{tr}}} + \left( \frac{4L^2}{\lambda^* n^{\mathrm{tr}}} + \frac{4M}{n^{\mathrm{tr}}} \right) \sqrt{\frac{n^{\mathrm{tr}} \log(4/\delta)}{2}}}_{:= \epsilon_{\text{train}}(\lambda^*)}
\end{equation}
By the definition of $\lambda^*$, this is exactly $L_{\mathcal{D}}(W^*) + \min_{\lambda \in \Lambda} \left( \Omega_\lambda(W^*) + \epsilon_{\text{train}}(\lambda) \right)$. Substituting this bound directly into $L_{\mathcal{D}}(\widehat{W}_{\hat{\lambda}}) \le L_{\mathcal{D}}(\widehat{W}_{\lambda^*}) + 2\epsilon_{\text{val}}$, and applying the union bound over all three events (total probability $1-\delta$), we obtain the Continuous Oracle Inequality.
Expanding $\epsilon_{\text{train}}(\lambda)$ and $2\epsilon_{\text{val}}$ precisely matches the theorem statement, completing the proof.
\end{proof}

\subsection{Continuous Oracle Inequality for Group DRO} \label{sec:dro_theory}

While Theorem~\ref{thm:erm_oracle} outlines the oracle inequality for standard ERM over a unified dataset, applying this continuous generalization bound to the Group DRO formulation requires substituting the sample complexities. In this setting, the generic continuous tuning parameter is instantiated as the robustness penalty (i.e., we set $\boldsymbol{\phi} = \eta$ with $k=1$), while the $L_2$ regularization coefficient $\lambda$ is held fixed strictly to satisfy the required strong convexity for algorithmic stability. Because DRO is evaluated on a worst-case basis at both levels, uniform stability in the lower level and uniform convergence in the upper level are strictly bottlenecked by the most scarcely represented groups.

\begin{lemma}[Modified Uniform Stability of Group DRO] \label{lem:dro_stability}
Let $J_\eta(W; S) = \max_{q \in \Delta_G} [\sum_{g=1}^G q_g L^{\mathrm{tr}}_g(W) - \frac{\eta}{2}\|q - \frac{1}{G}\mathbf{1}\|^2 + \frac{\lambda}{2}\|W\|^2]$ be the lower-level Group DRO objective. Under the assumptions of bounded and Lipschitz continuous group losses (Assumption~\ref{asm:loss}), the lower-level Group DRO algorithm $\mathcal{A}(S) = \arg\min_{W} J_\eta(W; S)$ is uniformly stable with modified constant:
\begin{equation}
    \beta_{\text{DRO}} = \frac{2L^2}{\lambda \min_g n_g^{\mathrm{tr}}} \left( 1 + \frac{M\sqrt{G}}{\eta_{\min}} \right)
\end{equation}
where $\min_g n_g^{\mathrm{tr}}$ is the size of the smallest group in the training set $S$. \end{lemma}

\begin{proof}
This proof follows the same standard gradient-based strong convexity argument as the standard stability result (Lemma~\ref{lem:uniform_stability}), but must account for the coupled min-max optimization over all groups. When the training set $S$ is perturbed by changing exactly one example $z$ into $z'$, this perturbation occurs in exactly one group, say group $k$. Because each instantaneous loss is bounded by $M$, the empirical loss of group $k$ changes by at most $\frac{1}{n_k^{\mathrm{tr}}} |\ell(W, z) - \ell(W, z')| \le \frac{2M}{\min_g n_g^{\mathrm{tr}}}$.

To bound the uniform stability, we evaluate how much the gradient of $J_\eta$ shifts due to this perturbation. By Danskin's theorem, the gradient is exactly $\nabla J_\eta(W; S) = \sum_{g=1}^G q_g \nabla L^{\mathrm{tr}}_g(W) + \lambda W$, where $q = \Pi_{\Delta_G}\left( \frac{1}{\eta}L^{\mathrm{tr}}(W) + \frac{1}{G}\mathbf{1} \right)$ are the optimal simplex weights. Let $S'$ be the perturbed dataset with corresponding optimal weights $q'$. The shift in the gradient of the loss term is bounded by:
\begin{align*}
    &\left\| \sum_{g=1}^G q_g \nabla L^{\mathrm{tr}}_g(S) - \sum_{g=1}^G q_g' \nabla L^{\mathrm{tr}}_g(S') \right\| \\&\le \underbrace{\left\| \sum_{g=1}^G q_g \left( \nabla L^{\mathrm{tr}}_g(S) - \nabla L^{\mathrm{tr}}_g(S') \right) \right\|}_{\text{Direct gradient shift}} + \underbrace{\left\| \sum_{g=1}^G (q_g - q_g') \nabla L^{\mathrm{tr}}_g(S') \right\|}_{\text{Weight perturbation shift}}
\end{align*}
For the first term, only group $k$'s empirical gradient changes (by at most $\frac{2L}{\min_g n_g^{\mathrm{tr}}}$). Since $q_k \le 1$, this direct shift is bounded by $\frac{2L}{\min_g n_g^{\mathrm{tr}}}$.
For the second term, the simplex projection is 1-Lipschitz, so the shift in optimal weights is bounded by the scaled shift in the group losses: $\|q - q'\| \le \frac{1}{\eta} \|L^{\mathrm{tr}}(S) - L^{\mathrm{tr}}(S')\| \le \frac{1}{\eta_{\min}} \frac{2M}{\min_g n_g^{\mathrm{tr}}}$. Multiplying by the Jacobian norm of the group losses ($\sqrt{G}L$), the second term is bounded by $\frac{2ML\sqrt{G}}{\eta_{\min} \min_g n_g^{\mathrm{tr}}}$.

Summing these, the entire gradient shifts by at most $\frac{2L}{\min_g n_g^{\mathrm{tr}}} \left( 1 + \frac{M\sqrt{G}}{\eta_{\min}} \right) := \Delta$. 
Because $J_\eta$ is $\lambda$-strongly convex, applying the exact same strong monotonicity displacement argument from Equation \eqref{eq:gradient_shift_stability} in Lemma~\ref{lem:uniform_stability} guarantees $\|\widehat{W}_\eta(S) - \widehat{W}_\eta(S')\| \le \frac{\Delta}{\lambda}$. Multiplying by the $L$-Lipschitz constant of the worst-group loss yields the final modified uniform stability constant $\beta_{\text{DRO}}$.
\end{proof}

\begin{theorem}[Group DRO Continuous Oracle Inequality]
\label{thm:dro_oracle}
Let $\min_g n_g^{\mathrm{tr}}$ and $\min_g n_g^{\mathrm{val}}$ denote the sizes of the smallest groups in the training and validation sets, respectively, where both sets consist of $G$ groups. For any reference predictor $W^*$, let $\Omega_\eta(W^*) = \frac{\lambda}{2}\|W^*\|^2 + \frac{\eta}{2}$ be the approximation bias penalty. Let $\hat{\eta} = \arg\min_{\eta \in H} f_{\text{val}}(\widehat{W}_\eta)$ be the hyperparameter chosen by minimizing the empirical validation worst-group risk $f_{\text{val}}$ over the continuous hypothesis class $\mathcal{H}_H$, where $H = [\eta_{\min}, \eta_{\max}]$. Let $\beta_{\text{DRO}} = \frac{2L^2}{\lambda \min_g n_g^{\mathrm{tr}}} \left( 1 + \frac{M\sqrt{G}}{\eta_{\min}} \right)$ be the modified uniform stability. With probability at least $1 - \delta$, the true worst-group risk $L_{\mathcal{D}}^{\text{worst}}(\widehat{W}_{\hat{\eta}}) = \max_g L_{\mathcal{D}, g}(\widehat{W}_{\hat{\eta}})$ satisfies:
\begin{align*}
    L_{\mathcal{D}}^{\text{worst}}(\widehat{W}_{\hat{\eta}}) \le &  L_{\mathcal{D}}^{\text{worst}}(W^*) + M \sqrt{\frac{2\log(4G/\delta)}{\min_g n_g^{\mathrm{tr}}}} + 2M \sqrt{\frac{2\log(4 G/\delta)}{\min_g n_g^{\mathrm{val}}}} \\
    &+ \min_{\eta \in H} \Bigg( \Omega_\eta(W^*) + \beta_{\text{DRO}} + \Sigma \sqrt{\frac{\log(4G/\delta)}{2}} \Bigg) \\
    &+ \frac{24M}{\sqrt{\min_g n_g^{\mathrm{val}}}} \left( \sqrt{\log \left( \frac{3\rho_{\mathcal{A}} R L}{M} \right)} + 2\sqrt{\log 2} \right) 
\end{align*}
where $\rho_{\mathcal{A}} \le \frac{G L M}{\lambda \eta_{\min}^2}$ is the Lipschitz constant of the DRO algorithmic mapping, and $\Sigma^2 = 4n^{\mathrm{tr}}\beta_{\text{DRO}}^2 + 8M\beta_{\text{DRO}} + \frac{4M^2}{\min_g n_g^{\mathrm{tr}}}$ bounds the stability variance.
\end{theorem}

\begin{proof}
The proof mirrors the three-step structure of Theorem~\ref{thm:erm_oracle}, isolating the exact points where worst-case group bounds modify the complexities.

\textbf{Step 1: Lower-Level Uniform Stability (via Union Bound).}
By Lemma~\ref{lem:dro_stability}, the lower-level Group DRO objective is uniformly stable with the modified constant $\beta_{\text{DRO}} = \frac{2L^2}{\lambda \min_g n_g^{\mathrm{tr}}} \left( 1 + \frac{M\sqrt{G}}{\eta_{\min}} \right)$. 
Because the empirical worst-group risk is a maximum over empirical averages, its expected generalization gap is not bounded directly by $\beta_{\text{DRO}}$ due to Jensen's inequality ($\max_g \mathbb{E}[\cdot] \le \mathbb{E}[\max_g \cdot]$). Instead, we decouple the maximum operator. By the subadditivity of the maximum, the worst-group generalization gap is bounded by the maximum of the individual group generalization gaps:
\begin{equation*}
    L_{\mathcal{D}}^{\text{worst}}(\widehat{W}_\eta) - L^{\mathrm{tr}}_{\text{worst}}(\widehat{W}_\eta) = \max_g L_{\mathcal{D}, g}(\widehat{W}_\eta) - \max_g L^{\mathrm{tr}}_g(\widehat{W}_\eta) \le \max_{g \in [G]} \underbrace{\left( L_{\mathcal{D}, g}(\widehat{W}_\eta) - L^{\mathrm{tr}}_g(\widehat{W}_\eta) \right)}_{:= Z_g(S)}
\end{equation*}
For any specific group $g$, the standard expected-loss stability result applies perfectly: $\mathbb{E}_S[Z_g(S)] \le \beta_{\text{DRO}}$. 
To bound the maximum over all groups with high probability, we rigorously evaluate the sensitivity of $Z_g(S)$ to a single point perturbation. Suppose we perturb exactly one training example $z_i \to z'$ to form $S^{(i)}$. Let $\widehat{W}$ and $\widehat{W}^{(i)}$ be the optimal predictors for $S$ and $S^{(i)}$, respectively. By the triangle inequality, the sensitivity of the group generalization gap is bounded by:
\begin{equation*}
    |Z_g(S) - Z_g(S^{(i)})| \le \underbrace{\left| L_{\mathcal{D}, g}(\widehat{W}) - L_{\mathcal{D}, g}(\widehat{W}^{(i)}) \right|}_{\text{True risk sensitivity}} + \underbrace{\left| L^{\mathrm{tr}}_g(\widehat{W}) - L^{\mathrm{tr}, (i)}_g(\widehat{W}^{(i)}) \right|}_{\text{Empirical risk sensitivity}}
\end{equation*}
For the true risk sensitivity, the uniform stability property guarantees that the absolute loss difference on any arbitrary point $z$ is deterministically bounded by $\beta_{\text{DRO}}$. Therefore, its expectation over the target group distribution $z \sim \mathcal{D}_g$ is identically bounded: $\left| L_{\mathcal{D}, g}(\widehat{W}) - L_{\mathcal{D}, g}(\widehat{W}^{(i)}) \right| \le \mathbb{E}_{z \sim \mathcal{D}_g} [|\ell(\widehat{W}, z) - \ell(\widehat{W}^{(i)}, z)|] \le \beta_{\text{DRO}}$.
For the empirical risk sensitivity, the bound depends on whether the perturbed index $i$ belongs to group $g$ ($i \in G_g$):
\begin{enumerate}
    \item \textbf{Case 1 ($i \notin G_g$):} The subset of points belonging to group $g$ is identical between $S$ and $S^{(i)}$. The empirical risk shifts solely due to the change in the algorithmic output $\widehat{W}$. Averaging the uniform stability bound over these $n_g^{\mathrm{tr}}$ unperturbed points yields an empirical shift of exactly $\beta_{\text{DRO}}$.
    \item \textbf{Case 2 ($i \in G_g$):} Group $g$ shares $n_g^{\mathrm{tr}}-1$ identical points between the two sets, but one point differs ($z_i \to z'$). For the identical points, the loss difference is bounded by $\beta_{\text{DRO}}$. For the single swapped point, the loss difference is naively bounded by $2M$ (Assumption~\ref{asm:loss}). Averaging these yields:
    \begin{align*}
        &\left| L^{\mathrm{tr}}_g(\widehat{W}) - L^{\mathrm{tr}, (i)}_g(\widehat{W}^{(i)}) \right| \\
        &\le \frac{1}{n_g^{\mathrm{tr}}} \sum_{j \in G_g, j \neq i} \underbrace{|\ell(\widehat{W}, z_j) - \ell(\widehat{W}^{(i)}, z_j)|}_{\le \beta_{\text{DRO}}} + \frac{1}{n_g^{\mathrm{tr}}} \underbrace{|\ell(\widehat{W}, z_i) - \ell(\widehat{W}^{(i)}, z')|}_{\le 2M} \\
        &\le \frac{n_g^{\mathrm{tr}}-1}{n_g^{\mathrm{tr}}} \beta_{\text{DRO}} + \frac{2M}{n_g^{\mathrm{tr}}} \le \beta_{\text{DRO}} + \frac{2M}{n_g^{\mathrm{tr}}}
    \end{align*}
\end{enumerate}
Summing the true and empirical sensitivities, the bounded difference constants $c_i$ for $Z_g(S)$ satisfy $c_i \le 2\beta_{\text{DRO}} + \frac{2M}{n_g^{\mathrm{tr}}}$ for $i \in G_g$, and $c_i \le 2\beta_{\text{DRO}}$ for $i \notin G_g$. 
The sum of squared differences over all $n^{\mathrm{tr}}$ independent examples is bounded by:
\begin{align*}
    \sum_{i=1}^{n^{\mathrm{tr}}} c_i^2 &= \sum_{i \notin G_g} (2\beta_{\text{DRO}})^2 + \sum_{i \in G_g} \left( 2\beta_{\text{DRO}} + \frac{2M}{n_g^{\mathrm{tr}}} \right)^2 \\
    &= n^{\mathrm{tr}}(2\beta_{\text{DRO}})^2 + n_g^{\mathrm{tr}} \cdot 2(2\beta_{\text{DRO}})\frac{2M}{n_g^{\mathrm{tr}}} + n_g^{\mathrm{tr}} \frac{4M^2}{(n_g^{\mathrm{tr}})^2} \le 4n^{\mathrm{tr}}\beta_{\text{DRO}}^2 + 8M\beta_{\text{DRO}} + \frac{4M^2}{\min_g n_g^{\mathrm{tr}}} \\
    &:= \Sigma^2
\end{align*}

Applying the one-sided McDiarmid's inequality (Lemma~\ref{lem:mcdiarmid}) to bound the deviation of $Z_g(S)$ above its expectation, and noting that the expected generalization gap is bounded by uniform stability $\mathbb{E}_S[Z_g(S)] \le \beta_{\text{DRO}}$, we obtain that with probability at least $1 - \frac{\delta}{4G}$:
\begin{equation*}
    Z_g(S) \le \beta_{\text{DRO}} + \Sigma \sqrt{\frac{\log(4G/\delta)}{2}}
\end{equation*}
Applying a union bound over all $G$ groups ensures that this bound holds simultaneously for the maximum $\max_g Z_g(S)$ with total failure probability $\delta/4$:
\begin{equation}
    L_{\mathcal{D}}^{\text{worst}}(\widehat{W}_\eta) \le L^{\mathrm{tr}}_{\text{worst}}(\widehat{W}_\eta) + \beta_{\text{DRO}} + \Sigma \sqrt{\frac{\log(4G/\delta)}{2}}
\end{equation}
\textbf{Step 2: Upper-Level Uniform Convergence Decomposition.}
The upper-level empirical objective is the worst-group validation loss $f_{\text{val}}(W) = \max_{g \in [G]} L^{\mathrm{val}}_g(W)$, and the target population risk is $L_{\mathcal{D}}^{\text{worst}}(W) = \max_{g \in [G]} L_{\mathcal{D},g}(W)$. 
By the non-expansive property of the maximum operator (Lemma~\ref{lem:max_non_expansive}), the uniform deviation over the continuous path $\mathcal{H}_H$ cleanly decomposes. Furthermore, because the supremum and finite maximum operators commute, we can isolate the supremum to each individual group:
\begin{align*}
    \sup_{\eta \in H} \left| f_{\text{val}}(\widehat{W}_\eta) - L_{\mathcal{D}}^{\text{worst}}(\widehat{W}_\eta) \right| &= \sup_{\eta \in H} \left| \max_{g \in [G]} L^{\mathrm{val}}_g(\widehat{W}_\eta) - \max_{g \in [G]} L_{\mathcal{D},g}(\widehat{W}_\eta) \right| \\
    &\le \sup_{\eta \in H} \max_{g \in [G]} \left| L^{\mathrm{val}}_g(\widehat{W}_\eta) - L_{\mathcal{D},g}(\widehat{W}_\eta) \right| \\
    &= \max_{g \in [G]} \sup_{\eta \in H} \left| L^{\mathrm{val}}_g(\widehat{W}_\eta) - L_{\mathcal{D},g}(\widehat{W}_\eta) \right|
\end{align*}
This isolates the continuous uniform convergence problem entirely into $G$ independent group-wise continuous uniform convergence bounds.

\textbf{Step 3: Rademacher Complexity and Union Bound.}
For any individual group $g$, bounding the continuous 1-dimensional path $\mathcal{H}_H$ relies on the Rademacher complexity of the independent subset $S_{\text{val},g}$ of size $n_g^{\mathrm{val}}$ . Following exactly the uniform deviation derivation in Lemma~\ref{lem:general_rademacher} (Equation~\eqref{eq:general_uniform_deviation} with $k=1$), we evaluate the continuous uniform deviation bound for group $g$. To ensure this bound holds simultaneously across all $G$ groups with a total failure probability of $\delta/2$, we apply a discrete union bound allocating confidence $\delta / (2G)$ to each group. As a result, for all $g \in [G]$ with probability $1 - \delta/2$:
\begin{equation*}
    \sup_{\eta \in H} \left| L^{\mathrm{val}}_g(\widehat{W}_\eta) - L_{\mathcal{D},g}(\widehat{W}_\eta) \right| \le \underbrace{\frac{12M}{\sqrt{n_g^{\mathrm{val}}}} \left( \sqrt{\log \left( \frac{3\rho_{\mathcal{A}} R L}{M} \right)} + 2\sqrt{\log 2} \right) + M \sqrt{\frac{2\log(4G/\delta)}{n_g^{\mathrm{val}}}}}_{:= \epsilon_{\text{val},g}}
\end{equation*}
Taking the maximum over all groups conservatively bounds the overall uniform deviation by the worst-case group size $\min_g n_g^{\mathrm{val}}$:
\begin{equation*}
    \max_{g \in [G]} \sup_{\eta \in H} \left| L^{\mathrm{val}}_g(\widehat{W}_\eta) - L_{\mathcal{D},g}(\widehat{W}_\eta) \right| \le \max_{g \in [G]} \epsilon_{\text{val},g} := \epsilon_{\text{val}}^{\text{worst}}
\end{equation*}

\textbf{Step 4: Final Decomposition.}
Let $\epsilon_{\text{train}}(\eta) = \beta_{\text{DRO}} + \Sigma \sqrt{\frac{\log(4/\delta)}{2}}$ denote the lower-level stability gap bound from Step 1. We define $\eta^* = \arg\min_{\eta \in H} \left( \Omega_\eta(W^*) + \epsilon_{\text{train}}(\eta) \right)$ as the optimal hyperparameter for the reference predictor $W^*$. Because $\hat{\eta} = \arg\min_{\eta \in H} f_{\text{val}}(\widehat{W}_\eta)$ minimizes the empirical worst-group validation loss $f_{\text{val}}$, we have $f_{\text{val}}(\widehat{W}_{\hat{\eta}}) \le f_{\text{val}}(\widehat{W}_{\eta^*})$. Recall that $f_{\text{val}}(W) = \max_g L^{\mathrm{val}}_g(W)$. Applying the uniform deviation bound (which holds over all $\eta \in H$ with probability $1-\delta/2$) to both $\hat{\eta}$ and $\eta^*$ yields:
\begin{align*}
    L_{\mathcal{D}}^{\text{worst}}(\widehat{W}_{\hat{\eta}}) &\le f_{\text{val}}(\widehat{W}_{\hat{\eta}}) + \epsilon_{\text{val}}^{\text{worst}} \\
    &\le f_{\text{val}}(\widehat{W}_{\eta^*}) + \epsilon_{\text{val}}^{\text{worst}} \\
    &\le L_{\mathcal{D}}^{\text{worst}}(\widehat{W}_{\eta^*}) + 2\epsilon_{\text{val}}^{\text{worst}}
\end{align*}

To bound the target $L_{\mathcal{D}}^{\text{worst}}(\widehat{W}_{\eta^*})$, we decompose the true risk via the empirical training risks. Crucially, while the algorithm optimizes the 3-term Group DRO objective $J_\eta(W)$, defined as:
\begin{equation*}
    J_\eta(W) = \max_{q \in \Delta_G} \left[ \sum_{g=1}^G q_g L^{\mathrm{tr}}_g(W) - \frac{\eta}{2}\left\|q - \frac{1}{G}\mathbf{1}\right\|^2 \right] + \frac{\lambda}{2}\|W\|^2
\end{equation*}
we can rigorously relate its minimizer back to the pure unregularized worst-group loss by bounding the $-\frac{\eta}{2}\|q - \frac{1}{G}\mathbf{1}\|^2$ penalty term.
Let $\Omega_{\eta^*}(W^*) = \frac{\lambda}{2}\|W^*\|^2 + \frac{\eta^*}{2}$ encompass the deterministic penalties. 
By the optimality of $\widehat{W}_{\eta^*}$ on $J_{\eta^*}$, we have $J_{\eta^*}(\widehat{W}_{\eta^*}) \le J_{\eta^*}(W^*)$. 

For the left side, the maximum over $q \in \Delta_G$ is lower bounded by evaluating it at the specific one-hot vector $q = \mathbf{e}_k$ corresponding to the worst group $k = \arg\max_g L^{\mathrm{tr}}_g(\widehat{W}_{\eta^*})$. The $\eta$-penalty for this one-hot vector evaluates to exactly $\frac{\eta^*}{2} \|\mathbf{e}_k - \frac{1}{G}\mathbf{1}\|^2 = \frac{\eta^*}{2}(1 - \frac{1}{G}) \le \frac{\eta^*}{2}$. Thus, retaining the non-negative regularizer $\frac{\lambda}{2}\|\widehat{W}_{\eta^*}\|^2 \ge 0$, we have:
\begin{equation*}
    J_{\eta^*}(\widehat{W}_{\eta^*}) \ge L^{\mathrm{tr}}_{\text{worst}}(\widehat{W}_{\eta^*}) - \frac{\eta^*}{2} + \frac{\lambda}{2}\|\widehat{W}_{\eta^*}\|^2 \ge L^{\mathrm{tr}}_{\text{worst}}(\widehat{W}_{\eta^*}) - \frac{\eta^*}{2}
\end{equation*}

For the right side, because the $\eta$-penalty term $-\frac{\eta^*}{2}\|q - \frac{1}{G}\mathbf{1}\|^2$ is strictly non-positive for any $q$, we can trivially upper bound the inner maximum by dropping the penalty entirely:
\begin{equation*}
    J_{\eta^*}(W^*) \le \max_{q \in \Delta_G} \left[ \sum_g q_g L^{\mathrm{tr}}_g(W^*) \right] + \frac{\lambda}{2}\|W^*\|^2 = L^{\mathrm{tr}}_{\text{worst}}(W^*) + \frac{\lambda}{2}\|W^*\|^2
\end{equation*}

Chaining these two inequalities ($L^{\mathrm{tr}}_{\text{worst}}(\widehat{W}_{\eta^*}) - \frac{\eta^*}{2} \le J_{\eta^*}(\widehat{W}_{\eta^*}) \le J_{\eta^*}(W^*) \le L^{\mathrm{tr}}_{\text{worst}}(W^*) + \frac{\lambda}{2}\|W^*\|^2$) securely isolates the empirical risks:
\begin{equation*}
    L^{\mathrm{tr}}_{\text{worst}}(\widehat{W}_{\eta^*}) \le L^{\mathrm{tr}}_{\text{worst}}(W^*) + \Omega_{\eta^*}(W^*)
\end{equation*}
We can now algebraically inject this upper bound into the true risk:
\begin{align*}
    L_{\mathcal{D}}^{\text{worst}}(\widehat{W}_{\eta^*}) &= L^{\mathrm{tr}}_{\text{worst}}(\widehat{W}_{\eta^*}) + \left( L_{\mathcal{D}}^{\text{worst}}(\widehat{W}_{\eta^*}) - L^{\mathrm{tr}}_{\text{worst}}(\widehat{W}_{\eta^*}) \right) \\
    &\le L^{\mathrm{tr}}_{\text{worst}}(W^*) + \Omega_{\eta^*}(W^*) + \underbrace{\left( L_{\mathcal{D}}^{\text{worst}}(\widehat{W}_{\eta^*}) - L^{\mathrm{tr}}_{\text{worst}}(\widehat{W}_{\eta^*}) \right)}_{\text{Stability gap}} \\
    &= L_{\mathcal{D}}^{\text{worst}}(W^*) + \Omega_{\eta^*}(W^*) + \text{Stability gap} + \underbrace{\left( L^{\mathrm{tr}}_{\text{worst}}(W^*) - L_{\mathcal{D}}^{\text{worst}}(W^*) \right)}_{\text{Hoeffding gap}}
\end{align*}

Simultaneously, since $W^*$ is fixed independent of $S_{\text{train}}$, we bound the one-sided deviation of its pure unregularized worst-group training loss. For any group $g \in [G]$, the one-sided Hoeffding's inequality bounds the deviation $L^{\mathrm{tr}}_g(W^*) - L_{\mathcal{D}, g}(W^*)$. Applying a union bound over all $G$ training groups with total confidence $\delta/4$ yields the Hoeffding gap:
\begin{equation*}
    L^{\mathrm{tr}}_{\text{worst}}(W^*) - L_{\mathcal{D}}^{\text{worst}}(W^*) \le \max_{g \in [G]} \left( L^{\mathrm{tr}}_g(W^*) - L_{\mathcal{D}, g}(W^*) \right) \le M \sqrt{\frac{2\log(4G/\delta)}{\min_g n_g^{\mathrm{tr}}}}
\end{equation*}
Combining the uniform convergence over $\mathcal{H}_H$ (probability $1 - \delta/2$), the stability gap for $\eta^*$ evaluated in Step 1 (probability $1 - \delta/4$), and the Hoeffding bound for the fixed reference $W^*$ (probability $1 - \delta/4$), the total failure probability sums exactly to $\delta$, yielding the final continuous oracle inequality.
\end{proof}

\begin{corollary}[Joint Hyperparameter Tuning of $\lambda$ and $\eta$]
\label{cor:dro_joint_tuning}
Suppose the assumptions of Theorem~\ref{thm:dro_oracle} hold. Let the joint hyperparameter vector be $\boldsymbol{\psi} = (\lambda, \eta) \in \Psi = \Lambda \times H \subset \mathbb{R}^2$, where $\Lambda = [\lambda_{\min}, \lambda_{\max}]$, $H = [\eta_{\min}, \eta_{\max}]$, and $R = \text{diam}(\Psi)$. Let $\hat{\boldsymbol{\psi}} = \arg\min_{\boldsymbol{\psi} \in \Psi} f_{\text{val}}(\widehat{W}_{\boldsymbol{\psi}})$. The lower-level algorithmic mapping is jointly Lipschitz with respect to $\boldsymbol{\psi}$ with constant $\rho_{\mathcal{A}} \le \frac{L}{\lambda_{\min}^2} + \frac{GLM}{\lambda_{\min}\eta_{\min}^2}$. With probability at least $1-\delta$, the true worst-group risk of the jointly tuned predictor satisfies:
\begin{align*}
    &L_{\mathcal{D}}^{\text{worst}}(\widehat{W}_{\hat{\boldsymbol{\psi}}}) \\ &\le L_{\mathcal{D}}^{\text{worst}}(W^*) + M \sqrt{\frac{2\log(4G/\delta)}{\min_g n_g^{\mathrm{tr}}}} + 2M \sqrt{\frac{2\log(4 G/\delta)}{\min_g n_g^{\mathrm{val}}}} \\
    &+ \min_{\lambda \in \Lambda, \eta \in H} \Bigg( \Omega_{\lambda, \eta}(W^*) + \beta_{\text{DRO}}(\lambda, \eta) + \Sigma(\lambda, \eta) \sqrt{\frac{\log(4G/\delta)}{2}} \Bigg) \\
    &+ \frac{24M}{\sqrt{\min_g n_g^{\mathrm{val}}}} \sqrt{2} \left( \sqrt{\log \left( \frac{3\rho_{\mathcal{A}} R L}{M} \right)} + 2\sqrt{\log 2} \right)
\end{align*}
where $\Omega_{\lambda, \eta}(W^*) = \frac{\lambda}{2}\|W^*\|^2 + \frac{\eta}{2}$, and $\beta_{\text{DRO}}(\lambda, \eta) = \frac{2L^2}{\lambda \min_g n_g^{\mathrm{tr}}} \left( 1 + \frac{M\sqrt{G}}{\eta} \right)$.
\end{corollary}

\begin{proof}[Proof of Corollary~\ref{cor:dro_joint_tuning}]
To establish the joint continuous uniform convergence bound, we must prove the mapping $\boldsymbol{\psi} \mapsto \widehat{W}_{\boldsymbol{\psi}}$ is Lipschitz over $\Psi$. We directly apply the exact strong monotonicity argument used in Lemma~\ref{lem:uniform_stability}. Let $J(W; \lambda, \eta)$ denote the lower-level Group DRO objective. For two hyperparameter configurations $\boldsymbol{\psi} = (\lambda, \eta)$ and $\boldsymbol{\psi}' = (\lambda', \eta')$, let $\widehat{W}$ and $\widehat{W}'$ be their respective minimizers.

Because $J(W; \lambda, \eta)$ is $\lambda$-strongly convex with respect to $W$, its gradient is $\lambda$-strongly monotone. Evaluating at the two optima and exploiting the first-order optimality condition $\nabla J(\widehat{W}; \lambda, \eta) = \nabla J(\widehat{W}'; \lambda', \eta') = 0$, we bound the shift in the predictors by the shift in the gradients evaluated at the fixed point $\widehat{W}'$:
\begin{align*}
    \lambda \|\widehat{W} - \widehat{W}'\|^2 &\le \langle \nabla J(\widehat{W}'; \lambda, \eta) - \nabla J(\widehat{W}; \lambda, \eta), \widehat{W}' - \widehat{W} \rangle \\
    &\le \langle \nabla J(\widehat{W}'; \lambda, \eta) - \nabla J(\widehat{W}'; \lambda', \eta'), \widehat{W}' - \widehat{W} \rangle \\
    &\le \|\nabla J(\widehat{W}'; \lambda, \eta) - \nabla J(\widehat{W}'; \lambda', \eta')\| \|\widehat{W} - \widehat{W}'\|
\end{align*}
Dividing by $\lambda \|\widehat{W} - \widehat{W}'\|$ isolates the deviation. The exact gradient of the objective is $\nabla J(W; \lambda, \eta) = \sum q_g \nabla L^{\mathrm{tr}}_g(W) + \lambda W$. Thus, the gradient shift decomposes into a regularization shift and a robust weight shift:
\begin{equation*}
    \|\nabla J(\widehat{W}'; \lambda, \eta) - \nabla J(\widehat{W}'; \lambda', \eta')\| \le |\lambda - \lambda'| \|\widehat{W}'\| + \left\| \sum_{g=1}^G (q_g - q_g') \nabla L^{\mathrm{tr}}_g(\widehat{W}') \right\|
\end{equation*}
For the first term, the optimality condition for $\widehat{W}'$ implies $\lambda' \widehat{W}' = -\sum q_g' \nabla L^{\mathrm{tr}}_g(\widehat{W}')$. Since $\|\nabla L^{\mathrm{tr}}_g\| \le L$, we have $\|\widehat{W}'\| \le \frac{L}{\lambda'} \le \frac{L}{\lambda_{\min}}$.
For the second term, following the weight perturbation derivation in Lemma~\ref{lem:dro_stability}, the optimal simplex weights shift by at most $\|q - q'\|_2 \le \left\| \left(\frac{1}{\eta} - \frac{1}{\eta'}\right) L^{\mathrm{tr}}(\widehat{W}') \right\|_2$. Since $|L^{\mathrm{tr}}_g| \le M$, this is bounded by $\frac{|\eta - \eta'|}{\eta_{\min}^2} \sqrt{G}M$. Multiplying by the Jacobian norm ($\sqrt{G}L$) bounds the robust weight shift by $\frac{GML}{\eta_{\min}^2} |\eta - \eta'|$.

Combining these and dividing by $\lambda \ge \lambda_{\min}$ yields the final perturbation bound:
\begin{equation*}
    \|\widehat{W} - \widehat{W}'\| \le \frac{L}{\lambda_{\min}^2} |\lambda - \lambda'| + \frac{G L M}{\lambda_{\min} \eta_{\min}^2} |\eta - \eta'| \le \left( \frac{L}{\lambda_{\min}^2} + \frac{G L M}{\lambda_{\min}\eta_{\min}^2} \right) \|\boldsymbol{\psi} - \boldsymbol{\psi}'\|_2
\end{equation*}
This establishes the joint Lipschitz constant $\rho_{\mathcal{A}}$ over $\Psi$. Injecting this $\rho_{\mathcal{A}}$ into a 2-dimensional variant of Lemma~\ref{lem:general_rademacher} yields the $\sqrt{2}$ dimension scaling on the validation uniform deviation. The minimization trades off this against the local stability gap $\beta_{\text{DRO}}(\lambda, \eta)$, completing the proof.
\end{proof}

\begin{corollary}[Joint Tuning with Deep Encoder Parameters]
\label{cor:dro_joint_tuning_encoder}
Suppose the assumptions of Corollary~\ref{cor:dro_joint_tuning} hold. Further assume that the input features are produced by a deep encoder $z = f_\theta(x)$ parameterized by $\theta \in \Theta$, where $\Theta \subset \mathbb{R}^{d_\theta}$. Assume (1) the encoder output $f_\theta(x)$ is $L_f$-Lipschitz with respect to $\theta$, and (2) the classification loss gradient $\nabla_W \ell(W, z)$ is $L_{\text{grad}}$-Lipschitz with respect to $z$. Let the extended joint hyperparameter vector be $\boldsymbol{\psi} = (\lambda, \eta, \theta) \in \Psi \times \Theta$.

The lower-level algorithmic mapping is jointly Lipschitz with respect to $\boldsymbol{\psi}$ with constant:
\begin{equation}
    \rho_{\mathcal{A}} \le \frac{L}{\lambda_{\min}^2} + \frac{GLM}{\lambda_{\min}\eta_{\min}^2} + \frac{L_{\text{grad}} L_f}{\lambda_{\min}}
\end{equation}
Crucially, the lower-level uniform stability constant $\beta_{\text{DRO}}(\lambda, \eta)$ remains unchanged, as $\theta$ is fixed during the lower-level optimization. Consequently, the true worst-group risk bound takes the identical structural form as Corollary~\ref{cor:dro_joint_tuning}, with the dimension factor $k$ increasing to $2 + d_\theta$ and the covering radius scaling to $\text{diam}(\Psi \times \Theta)$.
\end{corollary}

\begin{proof}
To derive the joint algorithmic Lipschitz constant, we follow the exact strong monotonicity argument used in Corollary~\ref{cor:dro_joint_tuning}. Let $J(W; \lambda, \eta, \theta) = \sum q_g(\eta) L^{\mathrm{tr}}_g(W, \theta) + \frac{\lambda}{2}\|W\|^2$ denote the lower-level objective. For any two joint configurations $\boldsymbol{\psi} = (\lambda, \eta, \theta)$ and $\boldsymbol{\psi}' = (\lambda', \eta', \theta')$, the optimal predictors shift according to the total gradient shift at the fixed point $\widehat{W}'$:
\begin{align*}
    \lambda \|\widehat{W} - \widehat{W}'\| &\le \|\nabla J(\widehat{W}'; \lambda, \eta, \theta) - \nabla J(\widehat{W}'; \lambda', \eta', \theta')\|
\end{align*}
By the triangle inequality, this gradient shift decomposes into three distinct perturbations corresponding to the regularization, the group weights, and the encoder representations:
\begin{align*}
    \|\nabla J(\widehat{W}'; \lambda, \eta, \theta) - \nabla J(\widehat{W}'; \lambda', \eta', \theta')\| 
    &\le \underbrace{|\lambda - \lambda'| \|\widehat{W}'\|}_{\le \frac{L}{\lambda_{\min}} |\lambda - \lambda'|} 
    + \underbrace{\left\| \sum_{g=1}^G (q_g - q_g') \nabla_W L^{\mathrm{tr}}_g(\widehat{W}', \theta) \right\|}_{\le \frac{GLM}{\eta_{\min}^2} |\eta - \eta'|} \\
    &+ \underbrace{\left\| \sum_{g=1}^G q_g' \left( \nabla_W L^{\mathrm{tr}}_g(\widehat{W}', \theta) - \nabla_W L^{\mathrm{tr}}_g(\widehat{W}', \theta') \right) \right\|}_{\text{Encoder shift}}
\end{align*}
The first two bounds follow identically from Corollary~\ref{cor:dro_joint_tuning}. For the third term, because the simplex weights satisfy $\sum q_g' = 1$, the encoder shift is bounded by the maximum gradient deviation across groups. Applying the smoothness of the loss and the Lipschitz property of the encoder, this shift evaluates to:
\begin{equation*}
    \max_g \| \nabla_W L_g^{\mathrm{tr}}(\widehat{W}', \theta) - \nabla_W L_g^{\mathrm{tr}}(\widehat{W}', \theta') \| \le L_{\text{grad}} L_f \|\theta - \theta'\|
\end{equation*}
Combining these three bounds and dividing by $\lambda \ge \lambda_{\min}$ yields the final mapping deviation:
\begin{equation*}
    \|\widehat{W} - \widehat{W}'\| \le \left( \frac{L}{\lambda_{\min}^2} + \frac{G L M}{\lambda_{\min} \eta_{\min}^2} + \frac{L_{\text{grad}} L_f}{\lambda_{\min}} \right) \|\boldsymbol{\psi} - \boldsymbol{\psi}'\|_2
\end{equation*}

For the stability term, uniform stability measures the sensitivity of the learning algorithm to a single training point perturbation for a \emph{fixed} hyperparameter configuration. Because $\theta$ is an upper-level variable, it acts as a constant mapping $x \mapsto z$ during the lower-level optimization. Provided the loss $\ell(W, z)$ is $L$-Lipschitz over the bounded representation space, the stability constant $\beta_{\text{DRO}}$ relies solely on the loss properties and the fixed regularization $\lambda$, remaining identical to the linear case.
\end{proof}
\begin{remark}[Neural Tangent Kernel]
While this generic continuous treatment seamlessly maintains the logical flow of our bilevel framework, the sample complexity bounds could be further tightened by incorporating specialized neural network generalization theories, such as the Neural Tangent Kernel (NTK) \citep{jacot2018neural}.
\end{remark}

\subsection{Extension to Hierarchical DRO (Bi-HDRO)} \label{sec:hdro_extension}
The continuous generalization theory seamlessly extends to Hierarchical DRO (HDRO) where the continuous hyperparameters being tuned include the multi-dimensional inner perturbation radii $\boldsymbol{\epsilon} = (\epsilon_1, \dots, \epsilon_G)$. In this formulation, the joint hyperparameter is $\boldsymbol{\psi} = (\lambda, \eta, \boldsymbol{\epsilon})$. As long as the smoothed robust loss $\ell_{\text{rob}, g}(W, \epsilon_g)$ is used, the lower-level mapping remains Lipschitz continuous.

\begin{lemma}[Algorithmic Lipschitz Continuity of Bi-HDRO]
\label{lem:hdro_lipschitz}
Assume the smoothed robust loss $\ell_{\text{rob},g}(W, \epsilon_g) = \mathbb{E}[\sup_{\|\delta_g\| \le \epsilon_g} \ell(W, z+\delta_g, y)]$ has bounded gradient shifts with respect to $\epsilon_g$, satisfying $\|\nabla_{W} \ell_{\text{rob},g}(W, \epsilon_{g}^{(1)}) - \nabla_{W} \ell_{\text{rob},g}(W, \epsilon_{g}^{(2)})\|_2 \le C_{\epsilon} |\epsilon_{g}^{(1)} - \epsilon_{g}^{(2)}|$, and is $L_{\epsilon}$-Lipschitz with respect to $\epsilon_g$. Let the lower-level objective maintain $\lambda$-strong convexity via $L_2$ regularization. Then, the HDRO algorithmic mapping $\widehat{W}_{\boldsymbol{\psi}}$ is jointly Lipschitz with resp
ect to the combined hyperparameter $\boldsymbol{\psi} = (\lambda, \eta, \boldsymbol{\epsilon})$ with the joint Lipschitz constant bounded by:
\begin{equation}
    \rho_{\mathcal{A},\boldsymbol{\psi}} \le \underbrace{\frac{L}{\lambda_{\min}^2} + \frac{GLM}{\lambda_{\min}\eta_{\min}^2}}_{\text{Shift from } \lambda, \eta} + \underbrace{\frac{C_{\epsilon}}{\lambda_{\min}} + \frac{\sqrt{G} L L_{\epsilon}}{\lambda_{\min}\eta_{\min}}}_{\text{Shift from } \boldsymbol{\epsilon}}
\end{equation}
\end{lemma}

\begin{remark}[Validity of the Lipschitz Assumption for Classification]
    The assumption that the robust loss has bounded gradient shifts with respect to $\epsilon_g$ naturally holds for the exact closed-form perturbations used in binary classification (derived in Appendix \ref{apdx:perturb}). Analytically solving the inner adversarial maximization $\min_{\|\delta\| \le \epsilon} y(W^\top(z + \delta) + b)$ yields the robust margin $y(W^\top z + b) - \epsilon \|W\|_*$, where $\|W\|_*$ is the dual norm of the perturbation constraint. 
    
    Consider $L_2$ norm perturbations where the dual norm is simply $\|W\|_* = \|W\|_2$. For the globally smooth robust logistic/BCE loss $\ell_{\text{rob}}(W, \epsilon) = \log(1 + \exp(A))$, where $A = -y(W^\top z + b) + \epsilon \|W\|_2$, the gradient with respect to $W$ everywhere is $\nabla_W \ell_{\text{rob}}(W, \epsilon) = \sigma(A) \left( -yz + \epsilon \frac{W}{\|W\|_2} \right)$. Using the $1/4$-Lipschitz continuity of the sigmoid function $\sigma(\cdot)$, the gradient shift between two perturbation radii $\epsilon^{(1)}$ and $\epsilon^{(2)}$ is explicitly bounded by the triangle inequality:
    \begin{align*}
        &\| (\sigma(A^{(1)}) - \sigma(A^{(2)})) (-yz) + (\sigma(A^{(1)})\epsilon^{(1)} - \sigma(A^{(2)})\epsilon^{(2)}) \frac{W}{\|W\|_2} \|_2 \\
        &\le \frac{1}{4} |A^{(1)} - A^{(2)}| \|z\|_2 + 1 \cdot |\epsilon^{(1)} - \epsilon^{(2)}| + \epsilon_{\max} \frac{1}{4} |A^{(1)} - A^{(2)}| \\
        &= \left( 1 + \frac{1}{4} \|W\|_2 \|z\|_2 + \frac{\epsilon_{\max}}{4} \|W\|_2 \right) |\epsilon^{(1)} - \epsilon^{(2)}|.
    \end{align*}
    Because the lower-level objective enforces $\lambda$-strong convexity via $\frac{\lambda}{2}\|W\|^2_2$, the optimal weights $\|W\|_2$ are bounded by a constant $M_w$. Assuming bounded $\|z\|_2$, this ensures the gradient shift is strictly bounded by $C_{\epsilon} |\epsilon^{(1)} - \epsilon^{(2)}|$ where $C_{\epsilon}$ is a finite constant. Furthermore, this justifies the $L_{\epsilon}$-Lipschitzness of the loss value itself: the derivative $\frac{\partial \ell_{\text{rob}}}{\partial \epsilon} = \sigma(A)\|W\|_2$ is strictly bounded by $\|W\|_2$, meaning the robust loss is $L_{\epsilon}$-Lipschitz with $L_{\epsilon} \le M_w$. Thus, all Lipschitz continuity assumptions are rigorously satisfied globally in our implementation.
\end{remark}

\begin{proof}
Let $F(W, q, \boldsymbol{\epsilon}) = \sum_{g=1}^G q_g \ell_{\text{rob},g}(W, \epsilon_g) - \frac{\eta}{2}\|q - \frac{1}{G}\mathbf{1}\|^2 + \frac{\lambda}{2}\|W\|^2$ be the lower-level HDRO objective. By Danskin's theorem, the gradient of the max-marginalized objective $J_{\boldsymbol{\epsilon}}(W)$ with respect to $W$ is $\nabla_{W} J_{\boldsymbol{\epsilon}}(W) = \sum_{g=1}^G q_g^* \nabla_{W} \ell_{\text{rob},g}(W, \epsilon_g) + \lambda W$, where $q^* = \Pi_{\Delta_G}( \frac{1}{\eta} \boldsymbol{\ell}_{\text{rob}}(W, \boldsymbol{\epsilon}) + \frac{1}{G} \mathbf{1} )$, with $\boldsymbol{\ell}_{\text{rob}}(W, \boldsymbol{\epsilon}) \in \mathbb{R}^G$ denoting the vector of robust losses across all groups.

If the perturbation hyperparameter shifts from $\boldsymbol{\epsilon}^{(1)}$ to $\boldsymbol{\epsilon}^{(2)}$, the gradient shift is bounded by the triangle inequality:
\begin{align*}
    &\|\nabla_{W} J_{\boldsymbol{\epsilon}^{(1)}}(W) - \nabla_{W} J_{\boldsymbol{\epsilon}^{(2)}}(W)\|_2\\ 
    &\le \left\| \sum_{g=1}^G q_g^{(1)} \left( \nabla_{W} \ell_{\text{rob},g}^{(1)} - \nabla_{W} \ell_{\text{rob},g}^{(2)} \right) \right\|_2 + \left\| \sum_{g=1}^G \left( q_g^{(1)} - q_g^{(2)} \right) \nabla_{W} \ell_{\text{rob},g}^{(2)} \right\|_2 \\
    &\le \sum_{g=1}^G q_g^{(1)} C_{\epsilon} |\epsilon_g^{(1)} - \epsilon_g^{(2)}| + \sum_{g=1}^G |q_g^{(1)} - q_g^{(2)}| \underbrace{\| \nabla_{W} \ell_{\text{rob},g}^{(2)} \|_2}_{\le L} \\
    &\le C_{\epsilon} \|\boldsymbol{\epsilon}^{(1)} - \boldsymbol{\epsilon}^{(2)}\|_2 + L \|q^{(1)} - q^{(2)}\|_1 \\
    &\le C_{\epsilon} \|\boldsymbol{\epsilon}^{(1)} - \boldsymbol{\epsilon}^{(2)}\|_2 + L \sqrt{G} \|q^{(1)} - q^{(2)}\|_2
\end{align*}
Because the simplex projection $\Pi_{\Delta_G}$ is 1-Lipschitz, the shift in the adversarial weights is strictly bounded by the shift in the robust loss terms scaled by the fixed penalty $\eta \ge \eta_{\min}$:
\begin{equation*}
    \|q^{(1)} - q^{(2)}\|_2 \le \frac{1}{\eta_{\min}} \|\boldsymbol{\ell}_{\text{rob}}(W, \boldsymbol{\epsilon}^{(1)}) - \boldsymbol{\ell}_{\text{rob}}(W, \boldsymbol{\epsilon}^{(2)})\|_2 \le \frac{L_{\epsilon}}{\eta_{\min}} \|\boldsymbol{\epsilon}^{(1)} - \boldsymbol{\epsilon}^{(2)}\|_2
\end{equation*}
Substituting this bound into the gradient shift yields a total shift bounded by $\left( C_{\epsilon} + \frac{\sqrt{G} L L_{\epsilon}}{\eta_{\min}} \right) \|\boldsymbol{\epsilon}^{(1)} - \boldsymbol{\epsilon}^{(2)}\|_2$.  
Because the objective is $\lambda$-strongly convex, applying the exact same strong monotonicity argument from Lemma~\ref{lem:dro_lipschitz} divides this gradient shift by $\lambda$, proving the mapping is Lipschitz continuous with respect to $\boldsymbol{\epsilon}$. Summing this $\boldsymbol{\epsilon}$-specific constant with the Lipschitz bounds for $\lambda$ and $\eta$ derived in Corollary 3.8 establishes the combined joint Lipschitz constant $\rho_{\mathcal{A},\boldsymbol{\psi}}$ over the entire hyperparameter space.
\end{proof}

\begin{remark}[Uniform Stability of Bi-HDRO]
While tuning $\boldsymbol{\epsilon}$ expands the algorithmic Lipschitz constant, the uniform stability of the lower-level algorithm remains unchanged. For a fixed configuration, Bi-HDRO optimizes the robust loss $\ell_{\text{rob}, g}(W) = \sup_{\|\delta\| \le \epsilon} \ell(W, z+\delta)$. Because the global constants $L$ and $M$ bound the base loss across all possible inputs, they naturally bound any perturbed input $z+\delta$. Thus, Bi-HDRO inherits the exact same stability constant $\beta_{\text{DRO}}(\lambda, \eta) = \frac{2L^2}{\lambda \min_g n_g^{\mathrm{tr}}} \left( 1 + \frac{M\sqrt{G}}{\eta} \right)$ as standard group DRO.
\end{remark}

\begin{corollary}[Joint Tuning of Bi-HDRO with Deep Encoder Parameters]
\label{cor:hdro_joint_tuning_encoder}
Suppose the assumptions of Lemma~\ref{lem:hdro_lipschitz} hold. Further assume the input features are generated by a deep encoder $z = f_\theta(x)$ parameterized by $\theta \in \Theta$, such that the encoder output is $L_f$-Lipschitz with respect to $\theta$, and the gradient of the robust loss $\nabla_W \ell_{\text{rob},g}(W, \theta, \epsilon_g)$ is $L_{\text{grad}}^{\text{rob}}$-Lipschitz with respect to the representation $z$. Let the fully extended joint hyperparameter vector be $\boldsymbol{\psi} = (\lambda, \eta, \boldsymbol{\epsilon}, \theta) \in \Psi \times \Theta$.

The Bi-HDRO algorithmic mapping is jointly Lipschitz with respect to $\boldsymbol{\psi}$ with constant:
\begin{equation}
    \rho_{\mathcal{A},\boldsymbol{\psi}} \le \underbrace{\frac{L}{\lambda_{\min}^2} + \frac{GLM}{\lambda_{\min}\eta_{\min}^2}}_{\text{Shift from } \lambda, \eta} + \underbrace{\frac{C_{\epsilon}}{\lambda_{\min}} + \frac{\sqrt{G} L L_{\epsilon}}{\lambda_{\min}\eta_{\min}}}_{\text{Shift from } \boldsymbol{\epsilon}} + \underbrace{\frac{L_{\text{grad}}^{\text{rob}} L_f}{\lambda_{\min}}}_{\text{Shift from } \theta}
\end{equation}
Moreover, as established in the preceding remark, the uniform stability constant $\beta_{\text{HDRO}}(\lambda, \eta)$ remains identical. 
\end{corollary}

\begin{proof}
The proof follows immediately by combining the derivation of Lemma~\ref{lem:hdro_lipschitz} with the triangle inequality decomposition established in Corollary~\ref{cor:dro_joint_tuning_encoder}. The gradient shift now contains a fourth additive term arising from the variation in $\theta$, which evaluates to $\max_g \| \nabla_W \ell_{\text{rob},g}(\widehat{W}', \theta, \epsilon_g) - \nabla_W \ell_{\text{rob},g}(\widehat{W}', \theta', \epsilon_g) \| \le L_{\text{grad}}^{\text{rob}} L_f \|\theta - \theta'\|$. Dividing by the strong convexity constant $\lambda_{\min}$ yields the additive $\theta$-shift term.
\end{proof}

\subsection{Background: Uniform Stability} \label{sec:uniform_stability_background}
To establish generalization guarantees, we rely on the framework of uniform stability. We first explicitly recall the uniform stability of the lower-level algorithm, adapting the standard analysis for Tikhonov regularization \citep[Section 13.3]{shalev2014understanding} to our general $\lambda$-strongly convex regularizer $\Omega_\lambda$.

\begin{definition}[Uniform Stability \citep{bousquet2002stability}]
A learning algorithm $\mathcal{A}$ is $\beta$-uniformly stable with respect to a loss function $\ell$ if, for any two training sets $S, S^{(i)}$ of size $m$ that differ by exactly one example, and for any arbitrary test point $z$, the following holds:
\begin{equation}
    \sup_z |\ell(\mathcal{A}(S), z) - \ell(\mathcal{A}(S^{(i)}), z)| \le \beta
\end{equation}
\end{definition}

\begin{lemma}[Uniform Stability of $\lambda$-Strongly Convex RLM] \label{lem:uniform_stability}
Assume the loss function $\ell(W, z)$ is convex and $L$-Lipschitz with respect to $W$. Let $\Omega_\lambda(W)$ be a $\lambda$-strongly convex regularization function. Then the Regularized Loss Minimization rule $\mathcal{A}(S) = \arg\min_{W} \left( L_S(W) + \Omega_\lambda(W) \right)$ is $\beta$-uniformly stable with $\beta = \frac{2L^2}{\lambda m}$.
\end{lemma}

\begin{proof}
Let $S = (z_1, \dots, z_m)$ be a training set, $z'$ an additional example, and $S^{(i)} = (z_1, \dots, z_{i-1}, z', z_{i+1}, \dots, z_m)$. Denote $f_S(W) = L_S(W) + \Omega_\lambda(W)$. 
Because $f_S$ is $\lambda$-strongly convex, its gradient is $\lambda$-strongly monotone. Evaluating this for the optimal predictors $W = \mathcal{A}(S)$ and $W^{(i)} = \mathcal{A}(S^{(i)})$ yields:
\begin{equation*}
    \lambda \|W^{(i)} - W\|^2 \le \langle \nabla f_S(W^{(i)}) - \nabla f_S(W), W^{(i)} - W \rangle
\end{equation*}
In view of the first-order optimality conditions of $W$ and $W^{(i)}$, we have:
\begin{align*}
    \langle -\nabla f_S(W), W^{(i)} - W \rangle &\le 0 \\
    \langle \nabla f_{S^{(i)}}(W^{(i)}), W^{(i)} - W \rangle &\le 0
\end{align*}
Summing these three inequalities yields:
\begin{equation*}
    \lambda \|W^{(i)} - W\|^2 \le \langle \nabla f_S(W^{(i)}) - \nabla f_{S^{(i)}}(W^{(i)}), W^{(i)} - W \rangle
\end{equation*}
Applying the Cauchy-Schwarz inequality, we can bound the distance strictly by the shift in the gradients:
\begin{equation} \label{eq:gradient_shift_stability}
    \lambda \|W^{(i)} - W\| \le \|\nabla f_S(W^{(i)}) - \nabla f_{S^{(i)}}(W^{(i)})\|
\end{equation}
Expanding the empirical risk gradients, the shift is exactly:
\begin{equation*}
    \|\nabla f_S(W^{(i)}) - \nabla f_{S^{(i)}}(W^{(i)})\| = \left\| \frac{\nabla \ell(W^{(i)}, z_i) - \nabla \ell(W^{(i)}, z')}{m} \right\| \le \frac{2L}{m}
\end{equation*}
Dividing by $\lambda$ gives the optimal parameter displacement $\|W^{(i)} - W\| \le \frac{2L}{\lambda m}$. 
Finally, the $L$-Lipschitzness of $\ell$ implies that for any test point $z$, the difference in loss is bounded by: 
\begin{equation}
    |\ell(\mathcal{A}(S^{(i)}), z) - \ell(\mathcal{A}(S), z)| \le L \|\mathcal{A}(S^{(i)}) - \mathcal{A}(S)\| \le \frac{2L^2}{\lambda m}
\end{equation}
Thus, the learning rule is $\frac{2L^2}{\lambda m}$-uniformly stable.
\end{proof}

\begin{lemma}[High Probability Generalization via Stability] \label{lem:stability_high_prob}
Let the learning algorithm $\mathcal{A}$ be $\beta$-uniformly stable, and assume the loss function $\ell$ is bounded by $M$. Then, for any $\delta \in (0, 1)$, with probability at least $1 - \delta$ over the random draw of a training set $S_{\text{train}}$ of size $n^{\mathrm{tr}}$, the true risk of the output hypothesis is bounded by:
\begin{equation}
    L_{\mathcal{D}}(\mathcal{A}(S_{\text{train}})) \le L^{\mathrm{tr}}(\mathcal{A}(S_{\text{train}})) + \beta + \left( 2\beta + \frac{2M}{n^{\mathrm{tr}}} \right) \sqrt{\frac{n^{\mathrm{tr}} \log(1/\delta)}{2}}
\end{equation}
For the $\lambda$-strongly convex Regularized Loss Minimization rule defined in Lemma~\ref{lem:uniform_stability}, we substitute $\beta = \frac{2L^2}{\lambda n^{\mathrm{tr}}}$.
\end{lemma}
\begin{proof}
Let $f(S_{\text{train}}) = L_{\mathcal{D}}(\mathcal{A}(S_{\text{train}})) - L^{\mathrm{tr}}(\mathcal{A}(S_{\text{train}}))$ denote the generalization gap. A fundamental result in stability theory \citep[Section 13.2]{shalev2014understanding} guarantees that the expected generalization gap is bounded by the uniform stability: $\mathbb{E}_{S_{\text{train}}}[f(S_{\text{train}})] \le \beta$. 

To obtain a high-probability bound, we analyze the sensitivity of $f(S_{\text{train}})$ to the replacement of a single training example. Let $S_{\text{train}}$ and $S_{\text{train}}^{(i)}$ be two training sets differing by exactly one example $z_i \to z'$. By definition of $\beta$-uniform stability, the loss on any arbitrary point $z$ changes by at most $\beta$:
\begin{equation}
    \sup_{z} |\ell(\mathcal{A}(S_{\text{train}}), z) - \ell(\mathcal{A}(S_{\text{train}}^{(i)}), z)| \le \beta
\end{equation}
Taking the expectation over $z \sim \mathcal{D}$, the difference in true risk is bounded by this uniform difference:
\begin{equation}
    |L_{\mathcal{D}}(\mathcal{A}(S_{\text{train}})) - L_{\mathcal{D}}(\mathcal{A}(S_{\text{train}}^{(i)}))| \le \mathbb{E}_{z \sim \mathcal{D}} \left[ \sup_{z'} |\ell(\mathcal{A}(S_{\text{train}}), z') - \ell(\mathcal{A}(S_{\text{train}}^{(i)}), z')| \right] \le \beta
\end{equation}
Furthermore, we can bound the change in the empirical risk between the two sets. Noting that $S_{\text{train}}$ and $S_{\text{train}}^{(i)}$ share $n^{\mathrm{tr}}-1$ identical points and only differ at the $i$-th point ($z_i$ vs $z'$), we have:
\begin{align*}
    |L^{\mathrm{tr}}(\mathcal{A}(S_{\text{train}})) - L^{\mathrm{tr}, (i)}(\mathcal{A}(S_{\text{train}}^{(i)}))| 
    &= \left| \frac{1}{n^{\mathrm{tr}}} \sum_{j=1}^{n^{\mathrm{tr}}} \left( \ell(\mathcal{A}(S_{\text{train}}), z_j) - \ell(\mathcal{A}(S_{\text{train}}^{(i)}), z_j^{(i)}) \right) \right| \\
    &\le \frac{1}{n^{\mathrm{tr}}} \sum_{j \neq i} \underbrace{|\ell(\mathcal{A}(S_{\text{train}}), z_j) - \ell(\mathcal{A}(S_{\text{train}}^{(i)}), z_j)|}_{\le \beta \text{ (uniform stability)}} \\
    &\quad + \frac{1}{n^{\mathrm{tr}}} \underbrace{|\ell(\mathcal{A}(S_{\text{train}}), z_i) - \ell(\mathcal{A}(S_{\text{train}}^{(i)}), z')|}_{\le 2M \text{ (bounded loss)}} \\
    &\le \frac{n^{\mathrm{tr}}-1}{n^{\mathrm{tr}}} \beta + \frac{2M}{n^{\mathrm{tr}}} \le \beta + \frac{2M}{n^{\mathrm{tr}}}
\end{align*}
Consequently, the change in the function $f$ when one point is perturbed is bounded by:
\begin{align*}
    |f(S_{\text{train}}) - f(S_{\text{train}}^{(i)})| &\le |L_{\mathcal{D}}(\mathcal{A}(S_{\text{train}})) - L_{\mathcal{D}}(\mathcal{A}(S_{\text{train}}^{(i)}))| + |L^{\mathrm{tr}}(\mathcal{A}(S_{\text{train}})) - L^{\mathrm{tr}, (i)}(\mathcal{A}(S_{\text{train}}^{(i)}))| \\
    &\le \beta + \left( \beta + \frac{2M}{n^{\mathrm{tr}}} \right) = 2\beta + \frac{2M}{n^{\mathrm{tr}}}
\end{align*}
Thus, $f(S_{\text{train}})$ satisfies the bounded differences property with constant $c = 2\beta + \frac{2M}{n^{\mathrm{tr}}}$. Applying the one-sided McDiarmid's inequality (Lemma~\ref{lem:mcdiarmid}), we have that with probability at least $1-\delta$:
\begin{equation}
    f(S_{\text{train}}) \le \mathbb{E}[f(S_{\text{train}})] + c \sqrt{\frac{n^{\mathrm{tr}} \log(1/\delta)}{2}} \le \beta + \left( 2\beta + \frac{2M}{n^{\mathrm{tr}}} \right) \sqrt{\frac{n^{\mathrm{tr}} \log(1/\delta)}{2}}
\end{equation}
which yields the final result.
\end{proof}

\subsection{Helpful Lemmas}
For completeness, we include the explicit derivations for properties utilized in the main theorems.

\begin{lemma}[McDiarmid's Inequality] \label{lem:mcdiarmid}
Let $X_1, \dots, X_m$ be independent random variables, and let $f(X_1, \dots, X_m)$ be a function that satisfies the bounded differences property with constants $c_1, \dots, c_m$:
\begin{equation}
    |f(x_1, \dots, x_i, \dots, x_m) - f(x_1, \dots, x_i', \dots, x_m)| \le c_i
\end{equation}
Then for any $\epsilon > 0$, the one-sided deviation is bounded by:
\begin{equation}
    P(f(X_1, \dots, X_m) - \mathbb{E}[f] \ge \epsilon) \le \exp\left( -\frac{2\epsilon^2}{\sum_{i=1}^m c_i^2} \right)
\end{equation}
By symmetry, the two-sided absolute deviation is bounded by:
\begin{equation}
    P(|f(X_1, \dots, X_m) - \mathbb{E}[f]| \ge \epsilon) \le 2\exp\left( -\frac{2\epsilon^2}{\sum_{i=1}^m c_i^2} \right)
\end{equation}
\end{lemma}

\begin{lemma}[Sub-Gaussian Variance Proxy] \label{lem:mcdiarmid_subgaussian}
Let $X_1, \dots, X_m$ be independent random variables, and let $f(X_1, \dots, X_m)$ be a function that satisfies the bounded differences property with constants $c_1, \dots, c_m$:
\begin{equation}
    |f(x_1, \dots, x_i, \dots, x_m) - f(x_1, \dots, x_i', \dots, x_m)| \le c_i
\end{equation}
Then the random variable $Z = f(X_1, \dots, X_m)$ is a sub-Gaussian random variable with variance proxy $\sigma^2 = \frac{1}{4} \sum_{i=1}^m c_i^2$.
\end{lemma}
\begin{proof}
By the one-sided McDiarmid's inequality (Lemma~\ref{lem:mcdiarmid}), for any $t \ge 0$, the probability of deviation from the expected value is bounded by:
\begin{equation}
    P(Z - \mathbb{E}[Z] \ge t) \le \exp\left( -\frac{2t^2}{\sum_{i=1}^m c_i^2} \right)
\end{equation}
A random variable $Z$ is formally defined as sub-Gaussian with variance proxy $\sigma^2$ if its tail distribution satisfies $P(Z - \mathbb{E}[Z] \ge t) \le \exp\left( -\frac{t^2}{2\sigma^2} \right)$. By equating the exponents of the bounds, we have:
\begin{equation*}
    \frac{t^2}{2\sigma^2} = \frac{2t^2}{\sum_{i=1}^m c_i^2} \implies 2\sigma^2 = \frac{1}{2} \sum_{i=1}^m c_i^2 \implies \sigma^2 = \frac{1}{4} \sum_{i=1}^m c_i^2
\end{equation*}
\end{proof}

\begin{lemma}[Maximal Inequality for Sub-Gaussian Random Variables] \label{lem:subgaussian_maximal}
Let $Z_1, \dots, Z_n$ be a finite collection of sub-Gaussian random variables, where each $Z_i$ has variance proxy $\sigma^2$ and expectation $\mu_i = \mathbb{E}[Z_i]$. The expected maximum of these random variables is bounded by:
\begin{equation}
    \mathbb{E}\left[\max_{1 \le i \le n} Z_i\right] \le \max_{1 \le i \le n} \mu_i + \sigma \sqrt{2 \log n}
\end{equation}
\end{lemma}
\begin{proof}
Let $Y_i = Z_i - \mu_i$. By definition, each $Y_i$ is a zero-mean sub-Gaussian random variable with variance proxy $\sigma^2$, satisfying the moment generating function bound $\mathbb{E}[\exp(s Y_i)] \le \exp\left(\frac{s^2 \sigma^2}{2}\right)$ for any $s > 0$. We wish to bound $\mathbb{E}[\max_i Y_i]$. 

By Jensen's inequality, since the exponential function is strictly convex for $s > 0$:
\begin{align*}
    \exp\left(s \mathbb{E}\left[\max_{1 \le i \le n} Y_i\right]\right) &\le \mathbb{E}\left[\exp\left(s \max_{1 \le i \le n} Y_i\right)\right] \\
    &= \mathbb{E}\left[\max_{1 \le i \le n} \exp(s Y_i)\right] \\
    &\le \mathbb{E}\left[\sum_{i=1}^n \exp(s Y_i)\right] \\
    &= \sum_{i=1}^n \mathbb{E}[\exp(s Y_i)] \le n \exp\left(\frac{s^2 \sigma^2}{2}\right)
\end{align*}
Taking the natural logarithm of both sides and dividing by $s$ yields:
\begin{equation}
    \mathbb{E}\left[\max_{1 \le i \le n} Y_i\right] \le \frac{\log n}{s} + \frac{s \sigma^2}{2}
\end{equation}
To minimize this upper bound, we select the optimal parameter $s = \sqrt{2 \log n / \sigma^2}$. Substituting this into the inequality gives:
\begin{equation}
    \mathbb{E}\left[\max_{1 \le i \le n} Y_i\right] \le \frac{\log n}{\sqrt{2 \log n} / \sigma} + \frac{\sigma^2 \sqrt{2 \log n / \sigma^2}}{2} = \sigma \sqrt{\frac{\log n}{2}} + \sigma \sqrt{\frac{\log n}{2}} = \sigma \sqrt{2 \log n}
\end{equation}
Finally, because $\max_i Z_i \le \max_i \mu_i + \max_i Y_i$, applying the expectation yields $\mathbb{E}[\max_i Z_i] \le \max_i \mu_i + \mathbb{E}[\max_i Y_i] \le \max_i \mu_i + \sigma \sqrt{2 \log n}$.
\end{proof}

\begin{lemma}[Non-Expansive Property of the Maximum Operator] \label{lem:max_non_expansive}
For any two finite sets of real numbers $A = \{A_1, \dots, A_G\}$ and $B = \{B_1, \dots, B_G\}$, the absolute difference between their maximums is bounded by the maximum of their element-wise absolute differences:
\begin{equation}
    \left| \max_{g \in [G]} A_g - \max_{g \in [G]} B_g \right| \le \max_{g \in [G]} |A_g - B_g|
\end{equation}
\end{lemma}
\begin{proof}
Without loss of generality, assume that $\max_{g} A_g \ge \max_{g} B_g$. Let $k = \arg\max_{g} A_g$ be the index that achieves the maximum for $A$. We can then write the difference as:
\begin{equation*}
    \max_g A_g - \max_g B_g = A_k - \max_g B_g
\end{equation*}
Since the maximum of the set $B$ must be at least as large as any specific element in $B$, we know that $\max_g B_g \ge B_k$. Substituting this lower bound can only increase the difference:
\begin{equation*}
    A_k - \max_g B_g \le A_k - B_k
\end{equation*}
Since a quantity is always bounded by its absolute value, and the $k$-th element's difference is bounded by the maximum absolute difference across all elements, we have:
\begin{equation*}
    A_k - B_k \le |A_k - B_k| \le \max_g |A_g - B_g|
\end{equation*}
This establishes the bound, completing the proof.
\end{proof}

\end{document}